\documentclass[11pt,letterpaper]{article}

\usepackage[utf8]{inputenc} 
\usepackage[T1]{fontenc}    
\usepackage{hyperref}       
\usepackage{url}            
\usepackage{booktabs}       
\usepackage{xcolor}
\usepackage{amsfonts}       
\usepackage{nicefrac}       
\usepackage{microtype}      
\usepackage{appendix}
\usepackage{amsthm}
\usepackage{amsmath}
\usepackage{amssymb}
\usepackage{amsfonts}
\usepackage{enumerate}
\usepackage{algorithm}
\usepackage{algorithmic}
\usepackage{geometry}
\usepackage{setspace}
\usepackage{xspace}
\usepackage{mathabx}
\usepackage{graphicx}
\usepackage{times}
\usepackage{tabularx, makecell, booktabs, array, longtable}
\usepackage{natbib}
\usepackage{verbatim}

\usepackage{amsthm, amsmath, amssymb, mathabx}
\allowdisplaybreaks 
\usepackage{graphicx}
\usepackage{enumerate}
\usepackage{pifont}
\usepackage{booktabs}
\usepackage{multirow}
\usepackage{multicol}
\usepackage{stfloats}
\usepackage{xspace}
\usepackage{thmtools}
\usepackage{thm-restate}
\usepackage{hyperref}
\usepackage{cleveref}
\crefname{appendix}{appendix}{appendices}
\Crefname{appendix}{Appendix}{Appendices}
\crefname{subappendix}{appendix}{appendices}
\Crefname{subappendix}{Appendix}{Appendices}
\crefname{subsubappendix}{appendix}{appendices}
\Crefname{subsubappendix}{Appendix}{Appendices}
\makeatletter
\let\hf@oldappendix\appendix
\renewcommand{\appendix}{%
  \hf@oldappendix
  \crefalias{section}{appendix}%
  \crefalias{subsection}{subappendix}%
  \crefalias{subsubsection}{subsubappendix}%
}
\makeatother
\crefname{ALC@unique}{line}{lines}
\Crefname{ALC@unique}{Line}{Lines}
\makeatletter
\renewcommand{\theALC@unique}{\arabic{ALC@line}}
\renewcommand{\theHALC@unique}{\arabic{ALC@unique}}
\makeatother
\usepackage{anyfontsize} 
\usepackage{makecell}
\usepackage{hhline}
\usepackage{colortbl}
\usepackage{xpatch}
\usepackage{scalerel,stackengine}
\usepackage{sidecap}

\makeatletter
\def\@fnsymbol#1{\ensuremath{\ifcase#1\or *\or \dagger\or \ddagger\or
  \mathsection\or \mathparagraph\or \|\or \diamond \or **\or \dagger\dagger
  \or \ddagger\ddagger \else\@ctrerr\fi}}
\makeatother

\makeatletter
\newcommand{\printfnsymbol}[1]{%
  \textsuperscript{\@fnsymbol{#1}}%
}
\makeatother

\newtheorem{theorem}{Theorem}
\newtheorem{lemma}{Lemma}
\newtheorem{corollary}[theorem]{Corollary}

\newtheorem{proposition}[theorem]{Proposition}
\newtheorem{definition}{Definition}

\crefname{condition}{Condition}{Conditions}

\crefname{assumption}{Assumption}{Assumptions}

\theoremstyle{definition}
\newtheorem{remark}{Remark}

\newcommand{\floor}[1]{\left\lfloor #1 \right\rfloor}

\newcommand{\abs}[1]{\left| #1 \right|}
\newcommand{\norm}[1]{\left\| #1 \right\|}

\DeclareMathOperator*{\argmax}{arg\,max}

\newcommand{\Proj}{\mathsf{Proj}}
\newcommand{\Cut}{\mathsf{Cut}}
\newcommand{\CutProj}{\mathsf{CutProj}}

\newcommand{\res}{\mathsf{res}}
\newcommand{\sta}{\mathsf{sta}}
\newcommand{\cut}{\mathsf{cut}}
\newcommand{\dev}{\mathsf{dev}}

\newcommand{\smax}{\mathsf{max}}
\newcommand{\mvptag}{\mathsf{MVP}}
\newcommand{\stoplabel}{\mathsf{stop}}
\newcommand{\known}{\mathsf{known}}
\newcommand{\markflag}{\mathsf{mark}}

\newcommand{\wt}[1]{\widetilde{#1}}
\newcommand{\wh}[1]{\widehat{#1}}

\newcommand{\ov}[1]{\overline{#1}}
\newcommand{\ud}[1]{\underline{#1}}

\newcommand{\cA}{\mathcal{A}}

\newcommand{\cC}{\mathcal{C}}
\newcommand{\cD}{\mathcal{D}}
\newcommand{\cE}{\mathcal{E}}
\newcommand{\cF}{\mathcal{F}}
\newcommand{\cG}{\mathcal{G}}

\newcommand{\cI}{\mathcal{I}}

\newcommand{\cK}{\mathcal{K}}

\newcommand{\cM}{\mathcal{M}}
\newcommand{\cN}{\mathcal{N}}
\newcommand{\cO}{\mathcal{O}}

\newcommand{\cR}{\mathcal{R}}
\newcommand{\cS}{\mathcal{S}}
\newcommand{\cT}{\mathcal{T}}
\newcommand{\cU}{\mathcal{U}}
\newcommand{\cV}{\mathcal{V}}
\newcommand{\cW}{\mathcal{W}}
\newcommand{\cX}{\mathcal{X}}

\newcommand{\boldc}{{\boldsymbol c}}

\newcommand{\bbE}{\mathbb{E}}

\newcommand{\bbI}{\mathbb{I}}

\newcommand{\bbN}{\mathbb{N}}

\newcommand{\bbP}{\mathbb{P}}

\newcommand{\bbR}{\mathbb{R}}

\newcommand{\bbV}{\mathbb{V}}

\newcommand{\II}{\mathbf{1}} 
\newcommand{\ee}{\textup{e}}

\newcommand{\N}{\mathsf{Null}}

\newcommand{\poly}{\mathsf{poly}}

\newcommand{\Regret}{\mathsf{Regret}}

\newcommand{\Roma}[1]{\uppercase\expandafter{\romannumeral#1}}

\newcommand{\SA}{\cS \times \cA}

\newcommand{\tref}{{\mathsf{ref}}}

\newcommand{\Ntotal}{N_{\mathsf{total}}}
\newcommand{\Nhtotal}{N_{h,\mathsf{total}}}

\newcommand{\trigger}{\ensuremath{\mathsf{trigger}}}
\newcommand{\true}{\ensuremath{\mathrm{True}}}
\newcommand{\false}{\ensuremath{\mathrm{False}}}
\newcommand{\aux}{\ensuremath{\mathsf{aux}}}

\newcommand{\red}[1]{\textcolor{red!80}{#1}}

\usepackage[colorinlistoftodos, textwidth=18mm]{todonotes}

\def\shownotes{1}
\ifnum\shownotes=0
\newcommand{\todorz}[1]{}
\newcommand{\todorzout}[1]{}
\newcommand{\todossdout}[1]{}
\newcommand{\todossd}[1]{}
\else
\newcommand{\todorz}[1]{\todo[color=blue!10, inline]{\small RZ: #1}}
\newcommand{\todorzout}[1]{\todo[color=blue!10]{\scriptsize RZ: #1}}
\newcommand{\todossdout}[1]{\todo[color=red!10]{\scriptsize SSD: #1}}
\newcommand{\todossd}[1]{\todo[color=red!10, inline]{\small SSD: #1}}
\fi

\makeatletter
\newif\ifhf@appendixtoc
\newcommand{\hf@appendixtocline}[2]{%
  \addtocontents{atoc}{\protect\contentsline{#1}{#2}{\thepage}{\@currentHref}\protected@file@percent}%
}
\newcommand{\hf@maybeappendixtocline}[2]{%
  \def\hf@entrytype{#1}%
  \def\hf@sectiontype{section}%
  \def\hf@subsectiontype{subsection}%
  \def\hf@subsubsectiontype{subsubsection}%
  \ifx\hf@entrytype\hf@sectiontype
    \hf@appendixtocline{#1}{#2}%
  \else\ifx\hf@entrytype\hf@subsectiontype
    \hf@appendixtocline{#1}{#2}%
  \else\ifx\hf@entrytype\hf@subsubsectiontype
    \hf@appendixtocline{#1}{#2}%
  \fi\fi\fi
}
\let\hf@oldaddcontentsline\addcontentsline
\renewcommand{\addcontentsline}[3]{%
  \hf@oldaddcontentsline{#1}{#2}{#3}%
  \ifhf@appendixtoc
    \def\hf@contentsfile{#1}%
    \def\hf@tocfile{toc}%
    \ifx\hf@contentsfile\hf@tocfile
      \hf@maybeappendixtocline{#2}{#3}%
    \fi
  \fi
}
\newcommand{\startappendixcontents}{\hf@appendixtoctrue}
\newcommand{\appendixtableofcontents}{%
  \section*{\contentsname}%
  \@starttoc{atoc}%
}
\makeatother

\title{Asymptotically Optimal Regret for Reinforcement Learning\\without Horizon Dependence}
\author{
    Runlong Zhou\thanks{Equal contribution.}~~\thanks{University of Washington. Email: \texttt{zhourunlongvector@gmail.com}}
    \and
    Zihan Zhang\printfnsymbol{1}\thanks{Hong Kong University of Science and Technology. Email: \texttt{zihanz@cse.ust.hk}}
    \and
    Maryam Fazel\thanks{University of Washington. Email: \texttt{mfazel@uw.edu}}
    \and
    Simon S. Du\thanks{University of Washington. Email: \texttt{ssdu@cs.washington.edu}}
}
\date{}

\begin{document}
\maketitle
\thispagestyle{empty}

\begin{abstract}
We study horizon-free regret minimization for finite-horizon time-homogeneous tabular Markov decision processes with $S$ states, $A$ actions, horizon $H$, and per-trajectory total reward bounded by $1$.
We propose a new algorithm and prove a regret upper bound
\begin{align*}
    \wt{O}\!(\sqrt{SAK}+S^8A^3)
\end{align*}
with failure probability $\delta$, where $K$ is the number of episodes and $\wt{O}(\cdot)$ hides $\poly\log\left(S,A,K,\frac{1}{\delta}\right)$.
Thus, the regret is $H$-free and asymptotically optimal in the sense of matching the contextual-bandit lower bound $\Omega(\sqrt{SAK})$ up to $\poly\log\left(S,A,K,\frac{1}{\delta}\right)$.
This totally removes the $\log H$ dependence from the previous $\wt{O}(\sqrt{SAK\log H} + S^2 A \log H)$ guarantee \cite{zhang2020reinforcement} and drastically improves the prior best horizon-free regret $\wt{O}(\sqrt{S^9A^3K})$ \cite{zhang2022horizon} asymptotically.

The main technical difficulty is that the optimal value functions $\{V_h^*\}_{h=1}^H$ are time-inhomogeneous even though the transition kernel is time-homogeneous.
A direct union bound over all value functions typically incurs an additional $\min\{\log H, S\}$ factor.
We avoid this factor by \ding{172} exploiting the monotonicity of $V_h^*$ in $h$ and \ding{173} non-trivially projecting the value functions onto a $S$-dimensional grid.

Our analysis relies on three additional ingredients.
First, we introduce a horizon-truncation argument that enables reward-based exploration and removes the cost of a separate reward-free exploration phase.
Second, we design a cutting bonus that preserves both optimism and the monotonicity needed for planning.
Third, we prove a new bound of the total deviation for time-homogeneous MDPs, which allows us to control the clipped variance terms in the cutting bonus with adjustable polynomial dependence on state space size and without paying any dependence on $H$.
Together, these tools yield an asymptotically optimal horizon-free regret guarantee.
\end{abstract}

\clearpage
\setcounter{page}{1}

\section{Introduction}\label{sec:intro}

Tabular Markov decision processes (MDPs) are one of the most fundamental models in reinforcement learning (RL).
In the finite-horizon episodic setting, an agent interacts with an unknown transition kernel for $K$ episodes, each of length $H$, and aims to compete with the optimal policy in hindsight.
When the MDP has $S$ states and $A$ actions, and the total reward in each episode is bounded by one, the minimax regret lower bound for the time-homogeneous setting\footnote{In time-homogeneous setting, the transition kernel is the same across layers $1,2,\ldots,H$.} is $\Omega(\sqrt{SAK})$, which is the same order as the lower bound for contextual bandits.
This raises a basic and conceptually important question:
is tabular RL, in the minimax sense, statistically harder than contextual bandits once the transition dynamics are time-homogeneous and the total reward is normalized?

A central difficulty is the dependence on the planning horizon $H$.
Classical finite-horizon analyses often pay polynomial or logarithmic factors in $H$.
Such factors are undesirable in the \emph{\textbf{horizon-free}} regime, where the episode length may be much larger than the number of episodes, or even exponentially large in the problem parameters.
A line of work has \emph{\textbf{totally removed the dependence on $H$}} while improving the dependence on other parameters progressively.
In particular, \cite{li2021settling} first showed that sample complexity can be made completely independent of $H$, though with \emph{an exponential dependence on the number of states}.
Later, \cite{zhang2022horizon} gave the first \emph{polynomial-time algorithm} for tabular MDPs whose regret is completely independent of $H$.
Their result established that horizon-free learning is possible in polynomial time, but the leading term in the bound is $\wt{O}(\sqrt{S^9A^3K})$, where $\wt{O}(\cdot)$ hides $\poly \log \left(S,A,K,\frac{1}{\delta}\right)$.
This result leaves a substantial polynomial gap from the contextual-bandit regret of $\wt{O}(\sqrt{SAK})$.
It is therefore natural to ask:
\begin{center}
\emph{\textbf{Is an $H$-free $\wt{O}(\sqrt{SAK})$ regret bound possible?}}
\end{center}

\paragraph{Our contribution.}

In this work, we close this gap asymptotically.
We design a horizon-free algorithm whose regret is bounded by $\wt{O}\!(\sqrt{SAK}+S^8A^3)$ as follows.

\begin{theorem}[Informal version of \Cref{thm:reg}]\label{thm:main}
Consider a time-homogeneous tabular MDP with $S$ states, $A$ actions, horizon $H$, and nonnegative, totally-bounded rewards.
There exists an algorithm such that, for any $K\ge1$ and $\delta\in(0,1)$, with probability at least $1-\delta$, its regret after $K$ episodes satisfies
\begin{align*}
   \Regret_{H}(K) = O\left(\sqrt{SAK} \log \frac{SAK}{\delta} + S^8A^3\log^2 \frac{SAK}{\delta} \right) .
\end{align*}
\end{theorem}

Thus, after a polynomial burn-in term independent of $H$, the regret matches the contextual-bandit lower bound up to $\poly \log \left(S,A,K,\frac{1}{\delta}\right)$.
To our knowledge, this is \emph{\textbf{the first horizon-free regret guarantee for time-homogeneous tabular MDPs whose leading term is asymptotically optimal in $K$, $S$, and $A$.}}

The main technical obstacle is that, although the transition kernel is time-homogeneous, the optimal value functions $V_1^*,V_2^*,\ldots,V_H^*$ are time-inhomogeneous.
A direct concentration argument over all possible $V_h^*$'s incurs an additional $\min\{\log H, S\}$ factor.
This issue is particularly severe in variance-aware regret analysis, where the bonus must be sharp enough to recover the $\sqrt{SAK}$ leading term.
The key structural observation is that the optimal value functions are monotone in the remaining horizon.
This monotonicity allows us to \emph{discretize} the value functions by projecting $\{V_h^*\}_{h=1}^H$ onto a carefully chosen $S$-dimensional grid.
The resulting function class has complexity independent of $H$, while the discretization error can be controlled at the scale required for near-optimal regret.

Building on this idea, we introduce three key technical innovations.

$\bullet$ \textbf{Horizon truncation and reward-based exploration.}
    We introduce a horizon-truncation framework that reduces the original $H$-step problem to a slightly shorter horizon.
    The remaining suffix is used for exploration only when the optimistic planning policy reaches an unlearned state-action pair.
    This enables reward-based exploration and avoids the cost of a separate reward-free exploration phase.

$\bullet$ \textbf{A cutting bonus for horizon-free optimism.}
    We develop a cutting bonus based on an $S$-dimensional discretization of the monotone optimal value sequence.
    The bonus preserves optimism and the monotonicity needed for dynamic programming, while avoiding the $\min\{ \log H, S\} $ factor caused by naive union bounds.

$\bullet$ \textbf{A new total-deviation bound for clipped variance.}
    We prove a new total-deviation bound for time-homogeneous MDPs, which controls the cumulative clipped variance of $\{V_h^*\}_{h=1}^H$.
    This bound is the key ingredient that closes the variance-aware regret analysis at the $\sqrt{SAK}$ scale.

We refer readers to \Cref{sec:tec} for more details.

\subsection{Related Works}\label{sec:rel}

\paragraph{Tabular RL.}
There is a long line of work on sample-complexity and regret guarantees for tabular RL, dating back to classical PAC-style algorithms such as $E^3$, R-MAX, and related model-based methods \citep{kearns1998near,brafman2002r,strehl2006pac,strehl2008analysis}.
For episodic tabular MDPs, a large body of work has studied near-optimal minimax regret bounds and sample-complexity guarantees \citep{kakade2003sample,kolter2009near,bartlett2009regal,jaksch2010near,szita2010model,lattimore2012pac,osband2013more,dann2015sample,azar2017minimax,dann2017unifying,osband2017posterior,agrawal2017optimistic,jin2018q,fruit2018near,talebi2018variance,dann2019policy,dong2019q,simchowitz2019non,russo2019worst,cai2019provably,zhang2020almost,yang2020q,pacchiano2020optimism,neu2020unifying,zhang2020reinforcement,li2021settling,menard2021ucb,xiong2021randomized,li2021breaking,pacchiano2020optimism}.
Many of these algorithms are based on optimism, posterior sampling, or model-free upper-confidence updates.

A key distinction is between time-inhomogeneous and time-homogeneous MDPs.
In the time-inhomogeneous setting, the transition kernel or reward function may vary with the layer $h$, so the number of degrees of freedom grows linearly with the horizon $H$.
As a result, a $\sqrt H$-type dependence is unavoidable in general.
In contrast, this paper studies time-homogeneous MDPs, where the transition kernel is shared across all time steps.
The main challenge is to exploit this time-homogeneous structure sharply enough to remove the dependence on the planning horizon.

We work under the totally-bounded reward assumption:
the total reward in each episode is at most one almost surely.
This assumption is more general than the commonly used uniformly bounded assumption $r_h\le \frac{1}{H}$ as it allows \emph{spiky rewards}.
Under this normalization, the contextual-bandit lower bound $\Omega(\sqrt{SAK})$ becomes the natural benchmark for time-homogeneous tabular MDPs.

\paragraph{Dependence on the horizon.}
Understanding the dependence on the planning horizon $H$ has been a central question in RL.
\cite{jiang2018open} asked whether polynomial dependence on $H$ is unavoidable.
Subsequent work showed that this is not true.
In particular, \cite{wang2020long} gave a computationally inefficient algorithm whose sample complexity depends only logarithmically on $H$, and \cite{zhang2020reinforcement} developed a computationally efficient algorithm with regret $\wt{O}\!(\sqrt{SAK \log H}+S^2A \log H)$.
Related ideas have also been used to reduce  horizon dependence in offline RL \cite{ren2021nearly}, stochastic shortest path \cite{tarbouriech2021stochastic,chen2021implicit}, and linear mixture MDPs \citep{zhang2021variance}.

Although these works substantially reduce the dependence on $H$, \emph{all their bounds still contain logarithmic dependence on $H$.}
The horizon-free regime is more demanding:
the regret and sample complexity should be completely independent of $H$, even when $H$ is exponential in $S,A,K$.

\paragraph{Horizon-free RL.}
The first result showing that \emph{sample complexity} can be \emph{\textbf{completely independent of $H$}} is given by \cite{li2021settling}.
Their result settles the information-theoretic horizon dependence, but the sample complexity has \emph{exponential} dependence on $S$.
A key idea in their analysis is to reason about entire trajectories rather than relying only on dynamic-programming-style local arguments.
They also use stationary policies to approximate the behavior of non-stationary policies.

The breakthrough work of \cite{zhang2022horizon} gave the first \emph{polynomial-time} horizon-free algorithm for tabular MDPs.
Their \emph{regret} bound is of order $\wt{O}(\sqrt{S^9A^3K})$, which is completely independent of $H$ and polynomial in $S$ and $A$.
Their algorithm relies on approximation properties of stationary policies, clipped MDPs, and an explicit exploration algorithm.
This established that polynomial-time horizon-free learning is possible, but left open whether the leading term can match the contextual-bandit regret of $\wt{O}(\sqrt{SAK})$.
Our work answers this question affirmatively up to a horizon-free polynomial burn-in term.

We summarize the most directly comparable horizon-dependence results in \Cref{tab:related-work-comparison}.
The distinction between regret and PAC/sample-complexity guarantees is important.
A high-probability regret bound implies a PAC guarantee by converting small average regret into the existence of an approximately optimal policy, for example by outputting a uniformly chosen episode policy;
see \citep[Section~3.1]{jin2018q} for a related discussion.
The reverse implication is not true:
a sample-complexity guarantee only controls suboptimality of the output policy, instead of cumulative suboptimality across all episodes.

\paragraph{Roadmap.} The remainder of the paper is organized as follows.
In \Cref{sec:p_setting}, we introduce the problem setting and notation.
In \Cref{sec:tec}, we discuss the main technical obstacles and summarize the key ideas used to overcome them.
We present the main algorithm in \Cref{sec:alg}, followed by additional discussion in \Cref{sec:discussion}.
The proof of \Cref{thm:main} is deferred to the appendix.

\begin{table}[t]
\centering
\small
\caption{Comparison with prior results for time-homogeneous tabular MDPs under the bounded-total-reward normalization.
Here $\wt{O}(\cdot)$ hides $\poly\log\left(S,A,K,\frac{1}{\varepsilon},\frac{1}{\delta}\right)$.
The final column marks whether the leading regret term matches the $\wt{O}(\sqrt{SAK})$ contextual-bandit regret up to $\poly\log\left(S,A,K,\frac{1}{\varepsilon},\frac{1}{\delta}\right)$.}
\label{tab:related-work-comparison}
\renewcommand{\arraystretch}{1.18}
\setlength{\tabcolsep}{3pt}
\begin{tabularx}{\linewidth}{@{}>{\raggedright\arraybackslash}p{0.15\linewidth}>{\raggedright\arraybackslash}p{0.30\linewidth}>{\raggedright\arraybackslash}p{0.25\linewidth}>{\centering\arraybackslash}p{0.08\linewidth}>{\centering\arraybackslash}p{0.15\linewidth}@{}}
\hline
\textbf{Paper} & \textbf{Regret bound} & \textbf{PAC/sample complexity} & \makecell{\textbf{$H$-}\\\textbf{free?}} & \makecell{\textbf{Optimal regret}\\\textbf{up to $\log$?}} \\
\hline
\cite{zhang2020reinforcement} & $\wt{O}\!(\sqrt{SAK\red{\log H}}+S^2A\red{\log H})$ & $\wt{O}\!\left(\frac{SA\red{\log H}}{\varepsilon^2}+\frac{S^2A\red{\log H}}{\varepsilon}\right)$ & No & \textbf{Yes} \\
\hline
\cite{li2021settling} & -- & $\frac{(SA)^{O(S)}}{\varepsilon^5}$ & \textbf{Yes} & -- \\
\hline
\cite{zhang2022horizon} & $\wt{O}\!(\sqrt{S^{\red{9}}A^{\red{3}}K})$ & $\wt{O}\!\left(\frac{S^{\red{9}}A^{\red{3}}}{\varepsilon^2}\right)$ & \textbf{Yes} & No \\
\hline
\cite{pmlr-v202-li23ak} & -- & $\wt{O} \left( \frac{S^2 A}{\varepsilon^2} + \frac{S^9 A^3}{\varepsilon} \right)$ & \textbf{Yes} & -- \\
\hline
\textbf{This paper} & $\wt{O}\!(\sqrt{SAK}+S^8A^3)$ & $\wt{O}\!\left(\frac{SA}{\varepsilon^2}+\frac{S^8A^3}{\varepsilon}\right)$ & \textbf{Yes} & \textbf{Yes} \\
\hline
\textbf{Lower bound} & $\Omega(\sqrt{SAK})$ & $\Omega\!\left(\frac{SA}{\varepsilon^2}\right)$ & -- & -- \\
\hline
\end{tabularx}
\end{table}

\section{Problem Setting}\label{sec:p_setting}

\paragraph{Basic notation.}
For any positive integer $N$, let $[N]:=\{1,2,\ldots,N\}$.
For an event $\cE$, let $\bbI [\cE]$ denote its indicator.
We use $\Delta(S)$ to denote the $S$-dimensional simplex.
For a probability vector $p\in\Delta(S)$ and a vector $v\in\bbR^{\cS}$, we write $ pv := \sum_{s\in\cS} p(s)v(s)$ as a shorthand of $p^{\top}v$, and define the variance $ \bbV(p,v) := p v^2-(pv)^2 = \sum_{s\in\cS}p(s)(v(s)-pv)^2$.
For two vectors $u,v\in\bbR^{\cS}$, we write $u\le v$ if $u(s)\le v(s)$ for every $s\in\cS$.
We use $\II_s$ to denote the one-hot vector representing state $s$.

\paragraph{Episodic time-homogeneous MDPs.}
We consider a finite-horizon time-homogeneous tabular Markov decision process $\cM=(\cS,\cA,P,r,H,\mu_1)$, where $\cS$ is a finite state space with cardinality $S$, $\cA$ is a finite action space with cardinality $A$, $P:\SA\to\Delta(S)$ is the unknown transition kernel, $r:\SA\to[0,1]$ is the \emph{\textbf{known}}\footnote{Learning reward is not a main obstacle in RL.} reward function, $H\in\bbN$ is the planning horizon, and $\mu_1\in\Delta(S)$ is the initial-state distribution.
We write $P_{s,a}:=P(\cdot\mid s,a)$ and $P_{s,a,s'}:=P(s'\mid s,a)$.

We assume that rewards are nonnegative and \emph{totally-bounded} in the following sense:
$\sum_{h=1}^H r(s_h,a_h)\le 1$ for any realizable trajectory $(s_1, a_1, \ldots, s_H, a_H)$.

\paragraph{Policies.}
A history-independent deterministic Markov policy is denoted by $\pi=\{\pi_h\}_{h=1}^H$, with $\pi_h:\cS\to\cA$.
Given policy $\pi$, a trajectory is generated as follows:
$s_1\sim\mu_1$, $a_h=\pi_h(s_h)$, and $s_{h+1}\sim P(\cdot\mid s_h,a_h)$ for $h\in[H]$.
Let $\Pi$ denote the class of all such policies;
a policy is stationary if $\pi_1=\pi_2=\cdots=\pi_H$, and we denote the class of stationary policies by $\Pi_{\sta}$.

\paragraph{Value functions and $Q$-functions.}
For $\pi \in \Pi$, define the value function $V_h^\pi$ and the $Q$-function $Q_h^\pi$ by
\begin{align*}
    V_h^\pi(s):= \bbE_{\pi}\left[\left. \sum_{t=h}^H r(s_t,a_t) \right| s_h=s\right],
    \qquad
    Q_h^\pi(s,a):= \bbE_{\pi}\left[ \left. \sum_{t=h}^H r(s_t,a_t) \right| s_h=s,\ a_h=a\right].
\end{align*}
We set $V_{H+1}^\pi(s)=0$ and $Q_{H+1}^\pi(s,a)=0$, then the Bellman equations are $Q_h^\pi(s,a) = r(s,a)+P_{s,a}V_{h+1}^\pi$ and $V_h^\pi(s) = Q_h^\pi(s,\pi_h(s))$.
The optimal value and $Q$-functions are defined by $V_h^*(s):=\sup_{\pi\in\Pi}V_h^\pi(s)$, $Q_h^*(s,a):=\sup_{\pi\in\Pi}Q_h^\pi(s,a)$, which satisfy $Q_h^*(s,a) = r(s,a)+P_{s,a}V_{h+1}^*$ and $V_h^*(s) = \max_{a\in\cA}Q_h^*(s,a)$.

Since the MDP is time-homogeneous and rewards are nonnegative, the optimal value sequence is monotone in the remaining horizon:
$V_h^*(s)\ge V_{h+1}^*(s)$ for all $h\in[H]$ and $s\in\cS$.
This monotonicity is one of the key structural properties used in our horizon-free analysis.

\paragraph{Regret.}
The agent interacts with the MDP for $K$ episodes.
At the beginning of episode $k$, it receives the initial state $s_1^k$, and chooses a policy $\pi^k$ based on the data collected in previous episodes.
The regret is
\begin{align*}
    \Regret_H(K) := \sum_{k=1}^K (V_1^*(s_1^k)-V_1^{\pi^k}(s_1^k)).
\end{align*}

\section{Technique Summary}\label{sec:tec}

In this section, we first explain the main obstacle to achieving an asymptotically optimal horizon-free regret bound, and then describe the technical innovations that allow us to overcome this obstacle.

\subsection{Hardness}
We first review the mechanism used in prior horizon-free algorithms \cite{li2021settling,zhang2022horizon}. 
A standard step in nearly all UCB-based analyses is to control the inverse-count term $ \sum_{k=1}^K\sum_{h=1}^H \frac{1}{\max\{N^k(s_h^k,a_h^k),1\}}$, where $N^k(s,a)$ is roughly the number of visits to $(s,a)$ before episode $k$, and $(s_h^k,a_h^k)$ is the state-action pair visited at step $h$ of episode $k$.
The usual doubling-count argument reduces this term to $\Theta \left(\sum_{s,a}\log \frac{N^{K+1}(s,a)}{N^1(s,a)}\right)$, which could be $\Omega (\log H)$ because some $(s, a)$ can be visited $\Omega (H)$ times.

The main insight of \cite{zhang2022horizon} is to collect enough initial samples during the first $K_0 \le O(\sqrt K)$ episodes and control $O\left(\sum_{s,a}\log \frac{N^{K+1}(s,a)}{N^{K_0}(s,a)}\right)$.
They show that, with high probability, the ratio $\frac{N^{K+1}(s,a)}{N^{K_0}(s,a)}$ is bounded by $\poly\left(S,A,K,\frac{1}{\delta}\right)$ uniformly over all $(s,a)$.
This removes the horizon dependence from the inverse-count term and is sufficient for obtaining a polynomial horizon-free regret bound.

However, this counting argument alone is not enough to obtain the optimal leading term $\widetilde O(\sqrt{SAK})$.
To see the remaining difficulty, consider an easier oracle problem in which the learner already has enough initial samples so that $\log \frac{N^{K+1}(s,a)}{N^{K_0}(s,a)}$ is independent of $H$, and moreover knows the entire optimal value sequence $V_1^*,V_2^*,\ldots,V_H^*$.
Even in this oracle setting, a simple bandit-style algorithm cannot be applied directly.
For a fixed state-action pair $(s,a)$, the learner must control the empirical errors $\left(P_{s,a} - \wh P_{s,a}\right)V_{h+1}^*$ simultaneously over all layers $h$.

A typical concentration bound for this error has the form $\sqrt{\frac{\bbV\cdot  B}{N}}$, where $\bbV$ is the relevant variance term, $N$ is the number of samples used to estimate $\wh P$, and $B$ is the extra factor coming from the failure probability and the size of the function class.
After summing over steps and episodes,  and applying Cauchy's inequality, the corresponding regret bound is of order $O\left(\sqrt{\sum \bbV}\cdot \sqrt{\sum \frac{1}{N}} \cdot \sqrt B\right)$.
The first term $\sum\bbV$ can be handled by standard total-variance arguments, and the second term $\sum \frac{1}{N}$ is controlled by the initial-sample and doubling-count arguments described above.
The real difficulty is the factor $B$.

There are two naive ways to bound $B$, and both are too crude.
If one treats $\{V_h^*\}_{h=1}^H$ as an arbitrary collection of $H$ test functions, then a direct Bernstein bound or an application of the Hoeffding--Azuma inequality with a union bound gives $B=\log \frac{H}{\delta}$, which violates the goal of a horizon-free regret bound.
On the other hand, if one avoids the union bound over $h$ by using a uniform coordinate-wise confidence set for the whole transition vector $P_{s,a}$, then $B$ becomes of order $S+\log \frac{1}{\delta}$.
This removes the $\log H$ term, but introduces an extra $S$-dependent factor that is too large for recovering the $\sqrt{SAK}$ leading term.

Thus, the core obstacle is that a single transition row must be accurate against a long, time-inhomogeneous sequence of value functions.
A naive bandit learning approach ignores the sequential structure and thus inevitably incurs a $\min\{ \log H, S\}$ factor.

\subsection{Solutions}
\label{subsec:solutions}

We now describe the main technical ideas behind our algorithm and analysis.

\vspace{-2ex}

\paragraph{Discretization.}
The key to a horizon-free analysis is to exploit the structure of the optimal value sequence.
Although the transition kernel is time-homogeneous, the optimal value functions $V_1^*,\ldots,V_H^*$ are time-inhomogeneous.
Nevertheless, because rewards are nonnegative, the optimal values are monotone:
$V_h^*(s)\ge V_{h+1}^*(s)$ for all $h\in[H-1],\ s\in\cS$.
Thus, $\{V_h^*\}_{h=1}^H$ has much smaller complexity than an arbitrary set of $H$ vectors in $[0,1]^S$.
We first recall a simple projection operator.

\begin{definition}[Projection operator]
\label{def:proj}
Fix $\epsilon\in(0,1)$.
For $x\in[0,1]$, define $    \Proj_{\epsilon}(x) := \floor{\frac{x}{\epsilon}}\epsilon .$
For $v\in[0,1]^S$, define
$
    \Proj_{\epsilon}(v) := [
        \Proj_{\epsilon}(v(1)),
        \ldots,
        \Proj_{\epsilon}(v(S))
    ]^\top .
$
\end{definition}

By monotonicity, the sequence $
    \Proj_{\epsilon}(V_1^*),
    \Proj_{\epsilon}(V_2^*),
    \ldots,
    \Proj_{\epsilon}(V_H^*)$
is coordinate-wise non-increasing.
Hence the number of distinct projected vectors is at most $1+\frac{S}{\epsilon}$.
Indeed, the integer-valued potential $
    \sum_{s\in\cS}
    \frac{\Proj_{\epsilon}(V_h^*(s))}{\epsilon}$
can decrease at most $\floor{\frac{S}{\epsilon}}$ times as $h$ increases.
Therefore, one may hope to decompose $
    V_h^* = \Proj_{\epsilon}(V_h^*) +
    (
        V_h^*-\Proj_{\epsilon}(V_h^*)
    ),$
take a union bound only over the projected value functions, and control the residual term by coordinate-wise concentration.

This direct projection argument, however, is \emph{not} sufficient for a sharp analysis.
It has two problems.
First, the projection can artificially increase variance.
For example, choose $v(i)=\epsilon+(-1)^i \epsilon_1$ and $p(i)=\frac{1}{S}$, where $0<\epsilon_1\ll\epsilon$.
Then $\bbV(p,v) \le O(\epsilon_1^2)$, while $\bbV(p, \Proj_{\epsilon} (v)) \ge \Omega (\epsilon^2)$.
Thus a naive projection may create variance that was not present in the original value function.
Second, even if the projection is accurate in $\ell_\infty$-distance, it is not clear that the residual variance $
    \bbV(P,V_h^*-\Proj_{\epsilon}(V_h^*)) $
is small enough after summed over all time steps to make orders of $S$ still tight.

The counterexample above hints the right modification:
coordinates close to the mean should not be discretized in the same way as coordinates far from the mean.
The variance contribution of near-mean coordinates is already small and should be controlled by a clipped variance term, while the far coordinates can be safely discretized.

\begin{definition}[Cutting operator]
\label{def:cut}
Fix $\epsilon\in(0,1)$.
Given $w\in[0,1]^S$ and $\ell\in[0,1]$, define $
    \Cut_{\epsilon}(w,\ell)=w'$ as
\begin{align*}
    w'(s)&=0,
    &&\text{if } |w(s)-\ell|\le 2\epsilon, \notag\\
    w'(s)&= (
        \lfloor(w(s)-\ell)/\epsilon\rfloor-2
    )\epsilon,
    &&\text{if } w(s)-\ell>2\epsilon, \notag\\
    w'(s)&= (
        \lfloor(w(s)-\ell)/\epsilon\rfloor+3
    )\epsilon,
    &&\text{if } w(s)-\ell<-2\epsilon .
\end{align*}
For $p\in\Delta(\cS)$ and $v\in[0,1]^S$, define $
    \ell_{\epsilon}(p,v) := \Proj_{\epsilon}
    (
        p^\top \Proj_{\epsilon}(v)
    )$,
and define the cut projection $
    \CutProj_{\epsilon}(p,v) := \Cut_{\epsilon} (
        \Proj_{\epsilon}(v),
        \ell_{\epsilon}(p,v)
    ).$
\end{definition}

The key difference between the ordinary projection and the cut projection is that the latter depends on the probability vector $p$.
It first estimates the mean $p^\top v$ at resolution $\epsilon$, and then separates the coordinates into two groups.
Coordinates far from the mean are kept in the cut component $\CutProj_{\epsilon}(p,v)$, while coordinates close to the mean are removed from that component and remain in the residual $ v-\CutProj_{\epsilon}(p,v)$.

Intuitively, the cut component contains the large deviations that dominate the ordinary variance, whereas the residual has only local fluctuations around the mean and can be controlled by a clipped variance.
We can show that the set $\{ \CutProj_{\epsilon}(p, V_{h}^*) \}_{p\in \Delta(S), h\in [H]}$ has size at most $\frac{4S}{\epsilon^2}$.
As a result, we use union bound to get uniform concentration bound for the error term $\left(P_{s, a} - \wh{P}_{s, a}\right) \CutProj_{\epsilon}\left(\wh{P}_{s, a}, V_{h}^*\right)$.
Then this part of error can be bounded by the classical total variance arguments.



\paragraph{Clipped variance and total deviation.}
We next show that the residual part has a variance that scales linearly with $\epsilon$.
For $p\in\Delta(\cS)$, $v\in\bbR^S$, and $x>0$, define the \emph{\textbf{$x$-clipped variance}}
\begin{align*}
    \bbV_x(p,v) := \sum_{s\in\cS}
    p(s)\min\{(v(s)-pv)^2,x^2\}.
\end{align*}
This quantity keeps the contribution of coordinates near the mean and clips the far tails.
The cutting construction ensures that the residual term produced by the bonus is controlled by the clipped variance. That is, $\bbV(p,v-\CutProj_{\epsilon}(p,v))\le \bbV_{5\epsilon}(p,v)$ for any $p\in\Delta(S)$ and $v\in[0,1]^S$ (see \Cref{lemma:lbb}).

Then we show the following crucial bound for expected single-episode \emph{\textbf{total clipped variance}}:

\begin{restatable}{corollary}{restatableExpectedEpsVar}\label{cor:expected_eps_var_from_total_deviation}
For every deterministic Markov policy $\pi\in\Pi$, initial state $s\in\cS$, and cutoff $x>0$,
\begin{align*}
    \bbE_\pi\!\left[
        \sum_{h=1}^{H}\bbV_x(P_{s_h,a_h},V_{h+1}^*)
        \,\middle|\, s_1=s
    \right]\le 2xS .
\end{align*}
\end{restatable}

Taking $x=5\epsilon$ gives an expected contribution of order $\epsilon S$ per episode, independent of $H$.
This linear dependence on $\epsilon$ is crucial:
we can choose $\epsilon$ inverse to some polynomial of $S$, $A$, and $K$ to make polynomial orders tight, while paying only logarithmic factors in the concentration event. This property is related to \emph{first-order deviation bound} of a time-homogeneous MDP, which may be of independent interest.

We now sketch the proof of \Cref{cor:expected_eps_var_from_total_deviation}.
First, the elementary inequality $\min\{a^2,x^2\}\le x|a|$ implies
total $x$-clipped variance is bounded by $x$ times \emph{total deviation}: $\bbE_{\pi} \left[\sum_{h=1}^H \left|   V_{h+1}^*(s_{h+1}) - P_{s_h,a_h} V_{h+1}^*\right|  \right]$.
For a policy $\pi$, define $D_h^\pi (s)$ (\Cref{eq:def_dev}) as the expected total deviation in the future starting from state $s$ at step $h$ while executing $\pi$.
The upper bound is then $x \cdot D_1^\pi(s)$.
The main novelty is the potential function
\begin{align*}
    \Phi_h(s) := 2\sum_{u\in\cS}\min\{V_h^*(s),V_h^*(u)\}.
\end{align*}
Using the monotonicity $V_h^*\ge V_{h+1}^*$, one proves that one-step deviation $|V_{h+1}^* (s') - P_{s,a} V_{h+1}^*|$ is no larger than the potential drop $\Phi_h (s) - \Phi_{h+1} (s')$ \emph{in expectation}.
Therefore, by induction, $D_h^\pi(s)\le \Phi_h(s)\le 2S$.
This gives a horizon-free bound on the expected total deviation per episode, and hence \Cref{cor:expected_eps_var_from_total_deviation}.


\paragraph{Horizon truncation and reward-aware exploration.}
The previous horizon-free algorithm~\cite{li2021settling,zhang2022horizon} relies on a sample-collection phase that is essentially reward-free.
This incurs a large exploration cost.
Our algorithm instead uses reward-aware exploration:
it first follows an optimistic policy for the original reward function guided by properly designed bonus function, and only invokes reward-free exploration if this policy reaches an unlearned state-action pair.
A technical obstacle is that, without additional structure, the optimistic policy might reach an unlearned pair only very late in the episode, leaving too few steps for  sample collection.
To address this, we truncate the planning horizon by a fraction of $\frac{1}{m}$.
The algorithm plans only for $
    H_1:=H-\frac{H}{m}$ steps,
and reserves the last $\frac{H}{m}$ steps as a suffix for reward-free exploration when needed.
The following lemma shows that this truncation costs only a small amount of value.

\begin{restatable}{lemma}{restatableCutNew}
\label{lemma:cutnew}
Fix a state $\tilde{s}\in\cS$.
Let $
    f(t) := \max_{\pi}
    \bbE_\pi \left[
        \sum_{h=1}^{t} r(s_h,a_h) | s_1 = \tilde{s}
    \right]$
be the optimal value with horizon $t$ and initial state $\tilde{s}$.
If $t>S$, then 
\begin{align*}
    f(t) \le \left(
        1+\frac{S}{t-S}
    \right)f(t-1).
\end{align*}
Consequently, by viewing $\ell$ consecutive steps as one big step, for any integer $\ell>0$ and any $t>S$, $
    f(t\ell) \le \left(
        1+\frac{S}{t-S}
    \right)f((t-1)\ell).$
\end{restatable}

This lemma provides a sharper comparison between optimal values at nearby
horizons than the horizon-comparison results used in prior work, such as
Lemma~2 of~\cite{zhang2022horizon} and Lemma~4.6 of~\cite{li2021settling}.
It allows us to remove a short suffix from each episode while incurring only a
controlled loss in value.
We defer the proof to \Cref{app:trunc_lemma_proof}.
We apply the lemma with $H=m\ell$.
Pick $0<\upsilon\le \sqrt{\frac{SA}{K}}$ and $m\ge \frac{4S}{\upsilon}$ such that $m$ divides $H$.
Using the fact that $m\geq 2S$ and  $f(\cdot)\le 1$, we have
$
    f(H) \le f\left(H-\frac{H}{m}\right) +
    \frac{2S}{m}
    f\left(H-\frac{H}{m}\right) \le f\left(H-\frac{H}{m}\right) +
    \frac{\upsilon}{2}$,.
Thus, if we learn the truncated $H_1$-step MDP successfully, then we also compete with the original $H$-step benchmark up to an additional loss $K\upsilon \le O(\sqrt{SAK})$.



\paragraph{Monotone bonus design.}
The discussion above treated the optimal value functions as if they were known.
In the actual learning process, the algorithm only has optimistic estimates $V_h^k$.
Thus the bonus must serve two purposes simultaneously:
it must give a concentration guarantee for $V_h^*$, and it must be monotone enough to propagate optimism from $V_{h+1}^k$ to $V_h^k$.

A standard MVP-style \cite{zhang2020reinforcement} approach designs a bonus $b_{\mvptag}$ such that, for a true transition row $p$ and an empirical row $\wh{p}$ obtained from $n$ samples, with high probability, $\wh{p}V_h^*+b_{\mvptag}\left(\wh{p},V_h^*,n\right)\ge pV_h^*$, and $\wh{p}v+b_{\mvptag}\left(\wh{p},v,n\right) \ge \wh{p}v'+b_{\mvptag}\left(\wh{p},v',n\right)$ for all $v'\le v$.
The second property makes the Bellman update monotone and allows optimism to be proved by induction.

In our setting, the concentration bound is based on the cutting decomposition.
For $V_h^*$, the concentration event has the schematic form
\begin{align*}
    \wh{p}V_h^* +
    b_{\cut}\left(\wh{p},V_h^*,n\right) +
    b_{\res}\left(\wh{p},V_h^*, n\right) \ge pV_h^* ,
\end{align*}
where $ b_{\cut}\left(\wh{p},V_h^*,n\right)$ denotes the confidence width of $\left(p - \wh{p}\right) \CutProj_{\epsilon}\left(\wh{p},V_h^*\right)$ and $b_{\res}\left(\wh{p},V_h^*,n\right)$ denotes the confidence width of $\left(p - \wh{p}\right) \left(V_h^* - \CutProj_{\epsilon}\left(\wh{p},V_h^*\right)\right)$.

It is natural to seek a single bonus $\ov{b}\left(\wh{p},v,n\right)$ that upper-bounds the two terms above, has a summable variance structure, and is monotone in $v$.
However, strict monotonicity is difficult because the cut projection $\CutProj_{\epsilon}\left(\wh{p},v\right)$ is discontinuous in $v$.
A small perturbation of one coordinate can move that coordinate from the ``near-mean'' region to the ``far-tail'' region by where $|v(s) -\ell_{\epsilon}(p,v) |$ is larger than $2\epsilon$ or not, changing the cut component abruptly.
Therefore, we do not require $\ov{b}$ to be monotone in the usual sense.

We overcome this by relaxing the monotonicity requirement by a constant factor.
We design a bonus $b$ (\Cref{eq:bonus}) such that, with probability at least $1-\poly(S,A,K)\delta$, the following two properties hold:
\begin{align*}
    \wh{p}V_h^*+b\left(\wh{p},V_h^*,n\right)\ge pV_h^*, \,\, \textup{and $\forall v',v\in[0,1]^S$ with $v'\le v$, } \wh{p}v+100\,b\left(\wh{p},v,n\right) \ge \wh{p}v'+b\left(\wh{p},v',n\right).
\end{align*}
The factor $100$ is only a universal constant.
It gives enough slack to handle the discontinuities caused by the cutting operation, while preserving the summability of the bonus.

With this relaxed monotonicity, optimism follows by the usual backward induction:
\begin{align*}
    Q_h(s,a) = r(s,a)+\wh{P}_{s,a}V_{h+1}+100\,b\left(\wh{P}_{s,a},V_{h+1},N(s,a)\right).
\end{align*}
Therefore, optimism is propagated: if $V_{h+1}\ge V_{h+1}^*$, 
\begin{align*}
    Q_h(s,a) \ge r(s,a)+\wh{P}_{s,a}V_{h+1}^*+b\left(\wh{P}_{s,a},V_{h+1}^*,N(s,a) \right)
    \ge r(s,a)+P_{s,a}V_{h+1}^* = Q_h^*(s,a).
\end{align*}  We defer the detailed proof of the monotonicity property to
\Cref{app:mono_pf}.

\section{Algorithm Description}
\label{sec:alg}

In this section, we present the main learning algorithm, \Cref{alg:new}.
It maintains a set $\cO\subseteq\cS\times\cA$ of \emph{unlearned} state-action pairs and a set of \emph{known} state-action-state triples $\cK$.
At the beginning of episode $k$, the learner computes an optimistic policy using \Cref{alg:plann} on the truncated horizon $ H_1=\frac{m-1}{m}H.$
The remaining suffix of length $\frac{H}{m}$ is reserved for reward-free exploration.
During the first $H_1$ steps, the learner follows the optimistic policy.
If the trajectory does not visit $\cO$, the learner plays a uniformly random policy for the remaining $\frac{H}{m}$ steps.
If instead the trajectory first reaches some unlearned pair $    \left(\widetilde s,\widetilde a\right)\in\cO,$ then the planning phase is stopped and \Cref{alg:explore} is invoked for the remaining steps of the episode.

 \vspace{-1ex}
\paragraph{The subroutines.} \Cref{alg:new} relies on two key subroutines.
\Cref{alg:plann} performs optimistic planning using our newly designed bonus function, whereas \Cref{alg:explore} conducts reward-free exploration, following the prior work \cite{zhang2022horizon} with minor modifications to adapt to our main algorithm.
Due to space constraints, we defer the detailed descriptions of these two subroutines to \Cref{app:missing_alg}.

 \vspace{-1ex}
\paragraph{Total count.} The algorithm maintains the cumulative count $ \Ntotal(s,a,s')$, which records the total number of observed transitions $(s,a,s')$ before episode $k$.
This count is updated after every step, including transitions generated during optimistic planning, random suffix play, and reward-free exploration.

 \vspace{-1ex}
\paragraph{The \emph{unlearned} state-action set $\cO$.}
The set $\cO$ consists of state-action pairs which are not sufficiently visited yet. It has two important usages in the algorithm: \ding{172}
For $(s,a)\in\cO$, the planning algorithm directly sets $Q_h(s,a)=1$ to keep optimism;
\ding{173}
Whenever an unlearned pair is reached, the exploration subroutine (\Cref{alg:explore}) is called.

A state-action pair is removed from $\cO$ before episode $k$ once it has been effectively explored enough so that the term $\log \frac{N^{K+1}(s,a)}{N^k(s,a)}$ no longer depends on $H$.
Here an effective exploration is certified by the event that \Cref{alg:explore} returns $\trigger=\true$.

 \vspace{-1ex}
\paragraph{The \emph{known} triple set $\cK$.} At each episode, $\cK$ is defined as $\{  (s,a,s') \mid \Ntotal(s,a,s') \geq N_{\tref}\}$ with $N_{\tref} = 1025S^2\log \frac{1}{\delta}$.
This set records the state-action-state triples for which the transition probability $P_{s,a,s'}$ can be estimated accurately.

 \vspace{-1ex}

\paragraph{Doubling update rule.}
We also introduce the conception of frozen count $ N(s,a,s')$, which is used only by the planning subroutine to form the empirical transition model $\widehat P$.
It is defined at the beginning of each episode as follows.
If $(s,a)\in\cO$, then the empirical model for $(s,a)$ is not used in planning.
In this case, we set $ N(s,a,s'):=0$ for all $ s'\in\cS$ and use the dummy denominator convention $ N(s,a):=1.$
If $(s,a)\notin\cO$, then $N(s,a,s')$ is defined by a doubling rule:
\begin{align*}
    \Ntotal(s,a) := \sum_{s' \in\cS}\Ntotal(s,a,s'),
    \quad m(s,a) := \lfloor\log_2 \Ntotal(s,a)\rfloor,
    \quad N(s,a):=2^{m(s,a)}. \notag
\end{align*}
Let $Y_1,Y_2,\ldots,Y_{\Ntotal(s,a)}$ be the ordered list of successor states observed in the first $\Ntotal(s,a)$ visits to $(s,a)$.
Then the frozen count defined as 
\begin{align}
N(s,a,s') := \sum_{i=1}^{N(s,a)} \bbI\{Y_i(s,a)=s'\}\label{eq:doubling_rule}
\end{align}
for all $s'\in \cS$.
Accordingly, $ \widehat P_{s,a,s'} := \frac{N(s,a,s')}{N(s,a)}.$
Thus, although the raw count $\Ntotal$ is updated after every episode, the frozen count used in planning changes only when $\Ntotal(s,a)$ crosses a power of two.

 \vspace{-1ex}

\paragraph{Reference model $P^{\tref}$.}
We define $P^{\tref}$ over the extended state space $\cS\cup \{z,z'\}$, where $z$ and $z'$ are two dummy states with $P^{\tref}_{z,a,z'} = 1, P^{\tref}_{z',a,z'} = 1$ for any action $a$.
For all $(s,a)$, $P^{\tref}_{s,a}$ is initialized as $\II_{z}$.
At episode $k$, assume the newly \emph{known} triples $\cR^k\neq \emptyset$.
Define $\cT^k= \{ (s,a) \mid \exists s', (s,a,s')\in \cR^k\}$.
At the end of the $k$-th episode (after the update $\cK\gets \cK\cup \cR^k$), we update the reference model $P^{\tref}$ as follows.
For $(s,a)\notin \cT^k$, keep $P^{\tref}$ invariant.
For $(s,a) \in \cT^k$, define $\cR^k(s,a) = \{ s'\mid (s,a,s')\in \cR^k\}$.
For each $s'\in \cR^k(s,a)$, there exists some $h_{s,a,s'}$ such that the total count $\Ntotal(s,a,s')$ reaches $N_{\tref}$ at the $h_{s,a,s'}$-th step.
Find the largest $h_{\smax} = \max_{s'\in \cR^k(s,a)} h_{s,a,s'}$ satisfying this condition for some $s'\in \cR^k(s,a)$.
Let $\widetilde{N}(s,a,y)$ be the total count $\Ntotal(s,a,y)$ after the $h_{\smax}$-th step  for all $y\in \cS$.
Then we update $P^{\tref}_{s,a}$ by $\widetilde{N}$:
\begin{align}
    P^{\tref}_{s,a,s'}
    :=
    \begin{cases}
    \displaystyle
    \frac{\widetilde{N}(s,a,s')}
    {\sum_{y\in\cK(s,a)}\widetilde{N}(s,a,y)},
    & (s,a,s')\in \cK,\\[1em]
    0, & (s,a,s')\notin \cK,
    \end{cases} \label{eq:ref_model}
\end{align}
and $P^{\tref}_{s,a,z} = 0$.
As a result, for each $(s,a)$, $P^{\tref}_{s,a}$ is only updated when there exists $s'$ such that $\Ntotal(s,a,s')$ reaches $N_{\tref}$. 
Intuitively, $P^{\tref}$ is the empirical transition kernel restricted to $\cK$.
It is used inside the exploration routine (\Cref{alg:explore}) to determine whether the remaining suffix can be used for effective exploration.

\begin{remark}
The use of doubling frozen counts for $\widehat P$ for $(s,a)\notin \cO$, together with threshold-based updates of $P^{\tref}$, is crucial for removing the dependence on $H$.
 These models are updated only at prescribed count levels, such as doubling times for $\widehat P$ and threshold-crossing times for $P^{\tref}$.
Thus, the number of possible empirical models is independent of the horizon $H$.
\end{remark}

\begin{algorithm}[t]
\caption{\texttt{Horizon-free Learning}}
\label{alg:new}
\begin{algorithmic}[1]
\STATE{\textbf{Input:} state space $\cS$, action space $\cA$, reward $r$, horizon $H$, number of episodes $K$, confidence level $\delta$.}
\STATE{\textbf{Initialize:} $N(s,a,s')\gets0$, $\Ntotal(s,a,s')\gets0$, $M(s,a)\gets0$, $\cO\gets\SA$, and $\cK\gets\emptyset$.}
\STATE{Set $\upsilon = \min\left\{\sqrt{\frac{SA}{1000K}}, \frac{1}{20S\log S} \right\}$, $m=\frac{4S}{\upsilon}$, $H_1=\frac{(m-1)H}{m}$,  $\epsilon  = \frac{1}{S^2}$, $N_{\tref}=1025S^2\log \frac{1}{\delta}$, $N_{\known}=10000\log \frac{1}{\delta}$, and set $P^{\tref}_{s,a} = \II_{z}$ for all $(s,a)$.}

\FOR{$k=1,2,\ldots,K$}
    \STATE{Receive $s_1^k$, compute $\pi^k\gets\texttt{Planning}(\cO,\{N(s,a,s')\}_{s,a,s'},\epsilon)$ (\Cref{alg:plann}).}
    \STATE{Set $\trigger\gets\false$, $\markflag\gets\false$, and initialize the trajectory record $\cD_k\gets\emptyset$.}

    \FOR{$h=1,2,\ldots,H_1$}
        \STATE{Set $a_h^k\gets\pi_h^k(s_h^k)$.}
        \IF{$(s_h^k,a_h^k)\in\cO$}
            \STATE{Set $\left(\wt{s},\wt{a}\right)\gets(s_h^k,a_h^k)$ and $\markflag\gets\true$.}
            \STATE{Run $\texttt{Exploration}\left(\left(\wt{s},\wt{a}\right),P^{\tref},\{\Ntotal(s,a,s')\}_{s,a,s'},\cK\right)$ (\Cref{alg:explore}) for the remaining steps, append the generated transitions to $\cD_k$, and set $\trigger$ to be its return value.}
            \STATE{Update $M\left(\wt{s},\wt{a}\right)\gets M\left(\wt{s},\wt{a}\right)+\bbI[\trigger=\true]$.} \label{line:M_upd}
            \STATE{\textbf{break}}
        \ELSE
            \STATE{Execute $a_h^k$, observe $s_{h+1}^k$, and append $(s_h^k,a_h^k,s_{h+1}^k)$ to $\cD_k$.}
        \ENDIF
    \ENDFOR

    \IF{$\markflag=\false$}
        \STATE{Run a random policy for steps $H_1+1,\ldots,H$, and append all generated transitions to $\cD_k$.}
    \ENDIF

    \STATE{For every $(s,a,s')\in\cD_k$, update $\Ntotal(s,a,s')\gets \Ntotal(s,a,s')+1$.}
    \STATE{Let $\cR^k\gets\{(s,a,s')\notin\cK:\Ntotal(s,a,s')\ge N_{\tref}\}$.}
    \STATE{Update $\cK\gets\cK\cup\cR^k$ and $\cO\gets\cO\setminus\{(s,a)\in\cO:M(s,a)\ge N_{\known}\}$.} \label{line:known_removal}
    \STATE{For all $(s,a)\notin\cO$, update $N(s,a,s')$ by the \emph{doubling update rule} (\Cref{eq:doubling_rule}).}
    \STATE{If $\cK$ is updated (i.e., $\cR^k\neq \emptyset)$, reconstruct $P^{\tref}$ following the \emph{reference model rule} (\Cref{eq:ref_model}).}
\ENDFOR
\end{algorithmic}
\end{algorithm}

\section{Discussion}\label{sec:discussion}

In this paper, we improve the leading term of horizon-free regret for tabular time-homogeneous MDPs to the asymptotically optimal rate $\widetilde O(\sqrt{SAK})$.
This matches the contextual-bandit lower bound up to $\poly \log \left(S, A, K, \frac{1}{\delta}\right)$, after a horizon-free polynomial burn-in term.
Thus, our result shows that the long-horizon and state-transition difficulties of MDPs do not affect the leading dependence on the number of episodes $K$.
The main open question is whether the burn-in term $\widetilde O(S^8A^3)$ can be reduced or even fully eliminated.
A positive answer would give a regret bound of order $\widetilde O(\sqrt{SAK})$ and would suggest that, tabular RL for MDPs is statistically no harder than for contextual bandits in the minimax sense.
Another important direction is to refute this possibility.
For example, one may try to prove a nontrivial lower bound on the burn-in term, such as $\Omega(S^{1+\alpha}A)$ for some constant $\alpha>0$.
Such a result would establish a  finite-sample separation between MDPs and contextual bandits.

\bibliography{ref}
\bibliographystyle{plain}

\newpage
\appendix
\startappendixcontents

\section*{Appendix}

In this appendix, we provide the proof of the formal version of \Cref{thm:main} as follows.

\begin{theorem}[Formal statement of \Cref{thm:main}]
\label{thm:reg}
Fix $K\geq1$.
Consider a time-homogeneous tabular MDP with $S\geq200$ states, $A\geq8$ actions, horizon $H$, and nonnegative totally-bounded rewards.
Let $\upsilon$ and $m=\frac{4S}{\upsilon}$ be the parameters in \Cref{alg:new}, assume that $m$ is a positive integer dividing $H$, and put $d:=\frac{H}{m}$.
Assume $\frac{d}{20S\log S}\geq22$.
For any $\delta_0\in(0,1)$, run \Cref{alg:new} with
\begin{align*}
    \delta:=\frac{\delta_0^2}{210S^7A^2K}.
\end{align*}
Then, with probability at least $1-\delta_0$,
\begin{align*}
   \Regret_H(K)
   =O\left(
       \sqrt{SAK}\log\frac{SAK}{\delta_0}
       +S^8A^3\log^2\frac{SAK}{\delta_0}
   \right).
\end{align*}
\end{theorem}

\paragraph{Organization.} 
In \Cref{app:tec_lemma}, we collect several technical tools used throughout the analysis, including concentration inequalities and the horizon truncation lemma.
\Cref{app:missing_alg} contains the algorithmic details omitted from the main text.
\Cref{app:notation} summarizes the notation used in the regret analysis, and In \Cref{sec:reg_analysis}, we provide the full regret analysis of \Cref{alg:new}, which proves \Cref{thm:reg}.
We leave two key lemmas about the optimism of the $Q$-function and the novel bound of clipped total variance to \Cref{sec:p_opt} and \Cref{sec:p_var}, respectively.
The analysis specific to the explicit exploration subroutine (\Cref{alg:explore}) is deferred to \Cref{app:expexp}.
This part closely follows the original analysis of~\cite{zhang2022horizon};
we include it here for completeness and to make the notation consistent with our algorithm.

\newpage

\appendixtableofcontents
\newpage




\newpage

\section{Technical Lemmas}\label{app:tec_lemma}

In this section, we collect several technical lemmas used throughout the
analysis.  These include standard counting tools and concentration inequalities,
as well as a few variants tailored to the stopping-time and frozen-count
arguments used in this paper.

\subsection{Variance and Concentration Bounds}

\begin{lemma}[Lemma\,30 in \cite{chen2021implicit}] \label{lem:var_xy}
For any two random variables $X, Y$ with finite second moments, with $Y$ essentially bounded, we have
\begin{align*}
    \bbV (X Y) \le 2 \bbV (X) (\sup |Y|)^2 + 2 (\bbE [X])^2 \bbV (Y).
\end{align*}
Consequently, $\sup |X| \le C<\infty$ implies $\bbV (X^2) \le 4 C^2 \bbV (X)$.
\end{lemma}

\begin{lemma}[Bennett's Inequality, Theorem\,3 in \cite{maurer2009empirical}]\label{lemma:bennet}
Let $Z,Z_1,...,Z_n$ be i.i.d. random variables with values in $[0,1]$ and let $\delta\in(0,1)$.
Define $\bbV Z = \bbE[(Z-\bbE Z)^2 ]$.
Then we have
\begin{align*}
\bbP\left[ \bbE[Z]-\frac{1}{n}\sum_{i=1}^n Z_i > \sqrt{\frac{  2\bbV Z \log \frac{1}{\delta}}{n}} +\frac{\log \frac{1}{\delta}}{3n} \right]\leq \delta,
\end{align*}
and
\begin{align*}
\bbP\left[ \frac{1}{n}\sum_{i=1}^n Z_i - \bbE[Z] > \sqrt{\frac{  2\bbV Z \log \frac{1}{\delta}}{n}} +\frac{\log \frac{1}{\delta}}{3n} \right]\leq \delta.
\end{align*}
\end{lemma}

\begin{lemma} [Theorem\,4 in \cite{maurer2009empirical}]\label{lem:bennett_empirical}
Let $Z,Z_1,\ldots,Z_n\ (n\ge 2)$ be i.i.d. random variables with values in $[0,b]$ and let $\delta\in(0,1)$.
Define $\bar{Z}=\frac{1}{n}\sum_{i=1}^n Z_i$ and $\hat V_n=\frac{1}{n}\sum_{i=1}^n \left(Z_i-\bar Z\right)^2$.
Then we have
\begin{align*}
    \bbP\left[\abs{\bbE[Z]-\frac{1}{n}\sum_{i=1}^n Z_i}>\sqrt{\frac{2\hat V_n \log \frac{2}{\delta}}{n-1}}+\frac{7 b \log \frac{2}{\delta}}{3(n-1)}\right]\le \delta.
\end{align*}
\end{lemma}

\begin{lemma} [Hoeffding-Azuma inequality]
\label{lemma:hoeffding}
Let $(\cF_i)_{i=0}^n$ be a filtration.  Let
$Z_1,\ldots,Z_n$ be random variables such that $Z_i$ is
$\cF_i$-measurable and $Z_i\in[0,1]$ almost surely for every
$i\in[n]$.
For any $\delta\in(0,1)$,
\begin{align*}
\bbP\left[
    \left|
        \sum_{i=1}^n
        \left(
            Z_i-\bbE[Z_i\mid\cF_{i-1}]
        \right)
    \right|
    >
    \sqrt{
        \frac{n\log\frac{2}{\delta}}{2}
    }
\right]
\le \delta .
\end{align*}
Equivalently, with probability at least $1-\delta$,
\begin{align*}
    \left|
        \frac1n\sum_{i=1}^n Z_i
        -
        \frac1n\sum_{i=1}^n
        \bbE[Z_i\mid\cF_{i-1}]
    \right|
    \le
    \sqrt{
        \frac{\log\frac{2}{\delta}}{2n}
    } .
\end{align*}
\end{lemma}

\begin{lemma} \label{lem:martingale_conc_mean}
Let $(\cF_k)_{k\geq0}$ be a filtration and let $X_1,X_2,\ldots$ be an adapted sequence taking values in $[0,l]$ for some $l>0$.
Set $Y_k:=\bbE[X_k\mid\cF_{k-1}]$.
For any $\delta\in(0,1)$, we have
\begin{align*}
    &\bbP \left[ \exists n, \sum_{k = 1}^n X_k \ge 3 \sum_{k = 1}^n Y_k + l \log \frac{1}{\delta} \right] \le \delta, \\
    &\bbP \left[ \exists n, \sum_{k = 1}^n Y_k \ge 3 \sum_{k = 1}^n X_k + 2l \log \frac{1}{\delta} \right] \le \delta.
\end{align*}
For any predictable mask $M_k\in[0,1]$, the same conclusions hold with $X_k$ and $Y_k$ replaced by $M_kX_k$ and $M_kY_k$, respectively.
\end{lemma}

\begin{proof}
[Proof of \Cref{lem:martingale_conc_mean}]
Write $\bar X_k=\frac{X_k}{l}$ and $\bar Y_k=\frac{Y_k}{l}$.
For the first inequality, use $\ee^x\le 1+(\ee-1)x$ for $x\in[0,1]$ to get
\begin{align*}
    \bbE\left[\exp(\bar X_k-(\ee-1)\bar Y_k)\mid\cF_{k-1}\right]
    &\le \ee^{-(\ee-1)\bar Y_k}\left(1+(\ee-1)\bar Y_k\right)\le 1 .
\end{align*}
Thus $
    S_n^+:=\exp\left(\sum_{k=1}^n \bar X_k-(\ee-1)\sum_{k=1}^n \bar Y_k\right)$ 
is a non-negative supermartingale with respect to $(\cF_n)_{n\ge0}$ and $S_0^+=1$.
On the event
$\sum_{k=1}^n X_k\ge 3\sum_{k=1}^n Y_k+l\log\frac{1}{\delta}$,
we have
\begin{align*}
    \sum_{k=1}^n \bar X_k-(\ee-1)\sum_{k=1}^n \bar Y_k
    \ge \log\frac{1}{\delta}+(4-\ee)\sum_{k=1}^n\bar Y_k
    \ge \log\frac{1}{\delta} .
\end{align*}
Ville's inequality gives the first bound.

For the second inequality, $1-\ee^{-3/2}>\frac{1}{2}$ and
$\ee^{-3x/2}\le 1-(1-\ee^{-3/2})x$ for $x\in[0,1]$, so
\begin{align*}
    \bbE\left[\exp\left(\frac12\bar Y_k-\frac32\bar X_k\right)\mid\cF_{k-1}\right]
    \le \exp\left(\frac12\bar Y_k\right)
    \left(1-(1-\ee^{-3/2})\bar Y_k\right)
    \le 1 .
\end{align*}
Consequently $
    S_n^-:=\exp\left(\frac12\sum_{k=1}^n \bar Y_k-\frac32\sum_{k=1}^n \bar X_k\right)$ 
is a non-negative supermartingale with $S_0^-=1$.  If
$\sum_{k=1}^n Y_k\ge 3\sum_{k=1}^n X_k
+2l\log\frac{1}{\delta}$, then $S_n^-\ge \frac{1}{\delta}$.
Another application of Ville's inequality proves the second bound.
For a predictable mask, $\bbE[M_kX_k\mid\cF_{k-1}]=M_kY_k$, so the final assertion follows by applying the same argument to the masked sequence.
\end{proof}


\begin{lemma}\label{lemma:conn}
Fix $\delta\in(0,1)$ and an integer $N_0 \geq 10000\log \frac{1}{\delta}$.
Consider a group of i.i.d. random variables $X_1, X_2,...$ with Bernoulli distribution of parameter $p\in (0,1)$.
Let $T$ be the stopping time that $\sum_{t=1}^T X_t = N_0$, and define $\hat{p} = \frac{N_0}{T}$.
With probability at least $1-2\delta$, it holds that
\begin{align*}
\left| \hat{p} - p\right| \leq \min\left\{ 2 \sqrt{\frac{p(1-p)\log \frac{1}{\delta}}{T}} + \frac{\log \frac{1}{\delta}}{T}, 2\sqrt{  \frac{\hat{p}(1-p)\log \frac{1}{\delta}}{T}}  +\frac{\log \frac{1}{\delta}}{T}    \right\}.\notag
\end{align*}
\end{lemma}
\begin{proof} [Proof of \Cref{lemma:conn}]
Let $G_1,\ldots,G_{N_0}$ be the waiting times between consecutive successes, where the successful trial is included in each waiting time.
Then $G_1,\ldots,G_{N_0}$ are i.i.d. geometric random variables with parameter $p$, and $T=\sum_{i=1}^{N_0}G_i.$
For a geometric random variable $G$, a direct calculation gives, for $0\le\lambda<1$,
\begin{align*}
    \bbE[\ee^{\lambda(pG-1)}]
    =
    \frac{p \ee^{-(1-p)\lambda}}
         {1-(1-p)\ee^{p\lambda}}.
\end{align*}
Differentiating the logarithm of this moment-generating function twice and using $\ee^{p\lambda}\le (1-p\lambda)^{-1}$, we obtain $\log\bbE[\ee^{\lambda(pG-1)}] \le \frac{(1-p)\lambda^2}{2(1-\lambda)}.$
Similarly, for every $\lambda\ge0$, $\log\bbE[\ee^{-\lambda(pG-1)}] \le \frac{(1-p)\lambda^2}{2}.$
Indeed, the logarithms of the two moment-generating functions, together with their first derivatives, vanish at zero, while their second derivatives are bounded by $\frac{1-p}{(1-\lambda)^2}$ and $1-p$, respectively.

By independence,
\begin{align*}
\begin{aligned}
    \log\bbE
    \exp(
        \lambda(pT-N_0)
    )
    &\le
    \frac{N_0(1-p)\lambda^2}{2(1-\lambda)},
    \qquad 0\le\lambda<1,                                      \\
    \log\bbE
    \exp(
        -\lambda(pT-N_0)
    )
    &\le
    \frac{N_0(1-p)\lambda^2}{2},
    \qquad \lambda\ge0.
\end{aligned}
\end{align*}
Applying Chernoff's inequality to the first bound yields
\begin{align*}
\begin{aligned}
    \bbP\left[
        pT-N_0
        \ge
        \sqrt{
            2N_0(1-p)\log \frac{1}{\delta}
        }
        +
        \log \frac{1}{\delta}
    \right]
    \le \delta.
\end{aligned}
\end{align*}
Similarly, the second bound gives
\begin{align*}
\begin{aligned}
    \bbP\left[
        N_0-pT
        \ge
        \sqrt{
            2N_0(1-p)\log \frac{1}{\delta}
        }
    \right]
    \le \delta.
\end{aligned}
\end{align*}
Consequently, with probability at least $1-2\delta$,
\begin{align}
    |pT-N_0|
    \le
    \sqrt{
        2N_0(1-p)\log \frac{1}{\delta}
    }
    +
    \log \frac{1}{\delta}.\label{eq:localx}
\end{align}

Dividing \Cref{eq:localx} by $T$ and using $\hat p=\frac{N_0}{T}$, we obtain the sharper empirical-form bound
\begin{align*}
\begin{aligned}
    |\hat p-p|
    &\le
    \sqrt{
        \frac{
            2\hat p(1-p)\log \frac{1}{\delta}
        }{
            T
        }
    }
    +
    \frac{\log \frac{1}{\delta}}{T}.
\end{aligned}
\end{align*}
This implies the displayed empirical-form bound since $\sqrt 2\le 2$.

It remains to derive the bound involving $p$.
By  \Cref{eq:localx} and the assumption $N_0\ge10000\log \frac{1}{\delta}$,
\begin{align*}
\begin{aligned}
    \frac{|pT-N_0|}{N_0}\le
    \sqrt{
        \frac{
            2(1-p)\log \frac{1}{\delta}
        }{
            N_0
        }
    }
    +
    \frac{\log \frac{1}{\delta}}{N_0}                    \le
    \sqrt{\frac{2}{10000}}+\frac{1}{10000}
    <
    \frac{1}{50}.
\end{aligned}
\end{align*}
It follows that $pT\ge \frac{49N_0}{50}$, and hence $\frac{\hat p}{p} = \frac{N_0}{pT} \le \frac{50}{49}.$
Substituting this relation into the sharper bound preceding \Cref{eq:localx} gives
\begin{align*}
\begin{aligned}
    |\hat p-p|
    \le
    \sqrt{
        \frac{
            2\hat p(1-p)\log \frac{1}{\delta}
        }{
            T
        }
    }
    +
    \frac{\log \frac{1}{\delta}}{T}                       \le
    \frac{10}{7}
    \sqrt{
        \frac{
            p(1-p)\log \frac{1}{\delta}
        }{
            T
        }
    }
    +
    \frac{\log \frac{1}{\delta}}{T}                       \le
    2
    \sqrt{
        \frac{
            p(1-p)\log \frac{1}{\delta}
        }{
            T
        }
    }
    +
    \frac{\log \frac{1}{\delta}}{T}.
\end{aligned}
\end{align*}
The proof is finished.
\end{proof}

\begin{lemma}\label{lemma:prefcon2}
Fix $\delta\in(0,1)$ and an integer $N_0\geq10000\log\frac1\delta$.
Let $Y_1,Y_2,\ldots$ be i.i.d. $\cS$-valued categorical random variables with mass function $p\in\Delta(S)$, where $|\cS|=S$.
For $t\geq1$ and $s\in\cS$, define $
    N_t(s):=\sum_{i=1}^t\bbI\{Y_i=s\}$, $
    \hat p_t(s):=\frac{N_t(s)}{t}$ and $
    \tau(s):=\inf\{t\geq1:N_t(s)=N_0\}$
with the convention $\inf\emptyset=\infty$.
With probability at least $1-4S^2\delta$, simultaneously for every ordered pair $s,s'\in\cS$ with $\tau(s)<\infty$,
\begin{align*}
\left|\hat p_{\tau(s)}(s')-p(s')\right|
\leq
\min\left\{
4\left(\sqrt{\frac{p(s')\log\frac1\delta}{\tau(s)}}+\frac{\log\frac1\delta}{\tau(s)}\right), \ 
8\left(\sqrt{\frac{\hat p_{\tau(s)}(s')\log\frac1\delta}{\tau(s)}}+\frac{\log\frac1\delta}{\tau(s)}\right)
\right\}.
\notag
\end{align*}
\end{lemma}

\begin{proof} [Proof of \Cref{lemma:prefcon2}]
We work on the probability-one event on which every state of zero $p$-mass is never observed and every state of positive $p$-mass is observed infinitely often.
On this event, $\tau(s)<\infty$ if and only if $p(s)>0$.

First consider the diagonal pair $s'=s$.
If $p(s)=1$, then $\tau(s)=N_0$ and $\hat p_{\tau(s)}(s)=1$, so the claim is exact.
If $p(s)\in(0,1)$, apply \Cref{lemma:conn} to $X_t=\bbI\{Y_t=s\}$.
With failure probability at most $2\delta$, both branches of \Cref{lemma:conn} hold at $T=\tau(s)$, and each is bounded by the corresponding displayed branch above.

Now fix \(s'\neq s\).
Set $r_0 = 1-p(s)$.
If \(r_0=0\), then \(p(s')=0\) and \(N_T(s')=0\), so the claim is trivial.
Assume \(r_0>0\), and define \(q(s'):=\frac{p(s')}{r_0}\).
Conditional on \(T\), among the \(T-N_0\) samples that are not equal to \(s\), the values are i.i.d. on \(\mathcal S\setminus\{s\}\) with probabilities \(q(\cdot)\).
Thus \(N_T(s')\), conditional on \(T\), is binomial with \(T-N_0\) trials and success probability \(q(s')\).

By Bennett's inequality (\Cref{lemma:bennet}), with conditional probability at least \(1-2\delta\),
\[
    |N_T(s')-q(s')(T-N_0)|
    \le
    \sqrt{2q(s')(T-N_0)\log \frac{1}{\delta}}
    +
    \frac{\log \frac{1}{\delta}}{3}.
\]
Since this bound holds conditionally for every value of \(T\), it also holds at the random time \(T\).
Combining this with the stopping-time estimate,
\[
    \left|
        \frac{T-N_0}{T}-r_0
    \right|
    =
    \left|
        p(s)-\hat{p}_{T}(s)
    \right|
    \le
   2 \sqrt{\frac{r_0\cdot p(s)\log \frac{1}{\delta}}{T}  }+\frac{\log \frac{1}{\delta}}{T}
\]

Therefore, we reach that

\begin{align}
\left| p(s') - \hat{p}_{T}(s') \right| & = \left| p(s') - \frac{N_T(s')}{T}\right| \nonumber
\\ & \leq \left| p(s')  - q(s') \frac{T-N_0}{T}  \right| + \left| \frac{N_T(s')}{T} - q(s') \frac{T-N_0}{T} \right| \nonumber
\\ & \leq  2q(s') \sqrt{\frac{r_0 \cdot p(s)\log \frac{1}{\delta}}{T}  }   +  \sqrt{\frac{2q(s') (T-N_0)\log \frac{1}{\delta} }{T^2}} + \frac{\log \frac{1}{\delta}}{T}\nonumber
\\ & \leq 2q(s') \sqrt{\frac{r_0 \cdot p(s)\log \frac{1}{\delta}}{T}  }   +  \sqrt{\frac{2p(s') \log \frac{1}{\delta}}{T}} + \sqrt{\frac{2q(s')\log \frac{1}{\delta}}{T} \cdot 2\sqrt{\frac{p(s)\log \frac{1}{\delta}}{T}}}  +\frac{\log \frac{1}{\delta}}{T} \nonumber
\\ & \leq \frac{5}{2} q(s') \sqrt{\frac{ r_0 \cdot p(s)\log \frac{1}{\delta}}{T}  }   +  \sqrt{\frac{2p(s') \log \frac{1}{\delta}}{T}} + \frac{4\log \frac{1}{\delta}}{T}
\\ & \leq   4\sqrt{\frac{p(s') \log \frac{1}{\delta}}{T}} + \frac{4\log \frac{1}{\delta}}{T}.\nonumber
\end{align}

In a similar way, we replace Bennet's inequality with empirical Bennett inequality (see \Cref{lem:bennett_empirical}) to get that:
with probability $1-4\delta$,
\begin{align}
\left| p(s') -\hat{p}_T(s')\right|  &\leq  2q(s')\sqrt{   \frac{ r_0 \cdot p(s) \log\frac{1}{\delta}   }{T}      } + \sqrt{\frac{2\hat{p}_T(s')\log\frac{1}{\delta}  }{T}          } +  \frac{3\log\frac{1}{\delta}}{T} \nonumber
\\ & \leq 2\sqrt{   \frac{ p(s') \log\frac{1}{\delta}   }{T}      }  + \sqrt{   \frac{2\hat{p}_T(s')\log\frac{1}{\delta} }{T}      } +             \frac{3\log\frac{1}{\delta}}{T}.\nonumber
\end{align}

Solving the above inequaity, and noting that $T\geq N_0\geq 10000\log\frac{1}{\delta}$, we have that 
\begin{align}
\left| p(s') - \hat{p}_{T}(s') \right|\leq 8 \left( \sqrt{\frac{\hat{p}_T(s')\log \frac{1}{\delta} }{T}} + \frac{\log \frac{1}{\delta}}{T}\right).\nonumber
\end{align}

We finish the proof by a union bound over all $s, s'$ pair.
\end{proof}

\begin{lemma} [Supermartingale Total Drift with Resets]\label{lem:hp_drift_bound}
Let $(\cF_t)_{t=1}^{T+1}$ be a filtration and let $\cU\subseteq[T]$ be a deterministic set with $|\cU|\le R$.
Suppose $X_t\in[0,B]$ is $\cF_t$-measurable for $t\in[T+1]$.
For every $t\in[T]\setminus\cU$, let $d_t\ge0$ be $\cF_t$-measurable and $
    \bbE[X_{t+1}\mid \cF_t] \le X_t-d_t $.
Then, for every $\delta\in(0,1)$, with probability at least $1-\delta$, $
    \sum_{t\in[T]\setminus\cU} d_t \le 2B\log \frac{2^{R+1}}{\delta}.$
\end{lemma}

\begin{proof}
[Proof of \Cref{lem:hp_drift_bound}]
If $B=0$, then $X_t=0$ for all $t$, and the drift condition implies $d_t=0$ for every $t\notin\cU$; the claim is deterministic.  Hence assume $B>0$.
Let $\wt X_t:=X_t+B$, so that $B\le\wt X_t\le2B$.
For $t=1,\ldots,T+1$, define
\begin{align*}
    M_t := \exp\left(\frac{1}{2B}\sum_{i \in [t-1] \setminus \cU} d_i\right)\wt X_t,
    \qquad
    N_t:=|[t-1] \cap \cU| .
\end{align*}
We show that $(2^{-N_t}M_t)_{t=1}^{T+1}$ is a supermartingale.

First consider $t\notin\cU$.
Since $\bbE[X_{t+1}\mid\cF_t]\le X_t-d_t$ and $X_{t+1}\ge0$, we have $d_t\le X_t\le\wt X_t$.
With $z_t=\frac{d_t}{\wt X_t}\in[0,1]$,
\begin{align*}
    \exp\left(\frac{d_t}{2B}\right)\left(\wt X_t-d_t\right) &\le \exp(z_t)\wt X_t(1-z_t) \le \wt X_t,
\end{align*}
where the last inequality uses $(1-z)\ee^z\le1$ for $z\in[0,1]$.
Thus, because $N_{t+1}=N_t$,
\begin{align*}
    \bbE[2^{-N_{t+1}}M_{t+1}\mid\cF_t]
    &\le 2^{-N_t}\exp\left(\frac{1}{2B}\sum_{i \in [t-1] \setminus \cU} d_i\right)
    \exp\left(\frac{d_t}{2B}\right)\left(\wt X_t-d_t\right)
    \le 2^{-N_t}M_t .
\end{align*}

Now consider $t\in\cU$.
The exponential factor does not change, $N_{t+1}=N_t+1$, and $\wt X_{t+1}\le2B\le2\wt X_t$.
Consequently,
\begin{align*}
    \bbE[2^{-N_{t+1}}M_{t+1}\mid\cF_t]
    &\le 2^{-N_t-1}\exp\left(\frac{1}{2B}\sum_{i \in [t-1] \setminus \cU} d_i\right)2\wt X_t
    =2^{-N_t}M_t .
\end{align*}
Therefore $\bbE[2^{-N_{T+1}}M_{T+1}]\le \bbE[M_1]\le2B$, and since $N_{T+1}\le R$, we have $\bbE[M_{T+1}]
\le2^{R+1}B$.
Finally, $\wt X_{T+1}\ge B$.
By Markov's inequality,
\begin{align*}
    \bbP\left[\sum_{t\in[T]\setminus\cU}d_t>2B\log \frac{2^{R+1}}{\delta}\right]
    &\le \bbP\left[M_{T+1}>\frac{2^{R+1}B}{\delta}\right]
    \le \delta .
\end{align*}
\end{proof}

\begin{lemma}[Linear Freedman Inequality]\label{lem:linear_freedman}
Let $(M_n)_{n\ge0}$ be a martingale with respect to $(\cF_n)_{n\ge0}$, with $M_0=0$ and $|M_n-M_{n-1}|\le c$ for all $n\ge1$, for some $c>0$.
Let $
    \cV_n := \sum_{i=1}^{n}
    \bbE [(M_i-M_{i-1})^2\mid\cF_{i-1}]$. 
Then, for any $\eta\in(0,\frac{1}{c}]$ and $\delta\in(0,1)$,
\begin{align*}
    \bbP\left[
        \exists n\ge0:
        M_n \ge \eta \cV_n+\frac{\log \frac{1}{\delta}}{\eta}
    \right] \le \delta .
\end{align*}
\end{lemma}

\begin{proof} [Proof of \Cref{lem:linear_freedman}]
Let $\xi_i:=M_i-M_{i-1}$ and $
    v_i:=\bbE[\xi_i^2\mid\cF_{i-1}]$. 
Since $\{M_i\}_{i\ge0}$ is a martingale, we have
\begin{align*}
    \bbE[\xi_i\mid\cF_{i-1}] = \bbE[M_i-M_{i-1}\mid\cF_{i-1}] = 0.
\end{align*}
Moreover, $|\xi_i|\le c$ and $\eta\in(0,\frac{1}{c}]$, so $\eta\xi_i\in[-1,1]$.
Using the elementary inequality $\ee^x\le 1+x+x^2$ for $x\in[-1,1]$, we get $
    \exp(\eta\xi_i) \le 1+\eta\xi_i+\eta^2\xi_i^2.$ 
Taking conditional expectation with respect to $\cF_{i-1}$ gives
\begin{align*}
    \bbE[\exp(\eta\xi_i)\mid\cF_{i-1}] &\le 1+\eta\bbE[\xi_i\mid\cF_{i-1}] +\eta^2\bbE[\xi_i^2\mid\cF_{i-1}]= 1+\eta^2 v_i\le \exp(\eta^2 v_i).
\end{align*}

Now define $
    S_n := \exp(\eta M_n-\eta^2\cV_n)$.
Since $
    M_n=M_{n-1}+\xi_n$ and $\cV_n=\cV_{n-1}+v_n$,
we have $
    S_n = S_{n-1}\exp(\eta\xi_n-\eta^2v_n)$.
    Because $S_{n-1}$ and $v_n$ are $\cF_{n-1}$-measurable,
\begin{align*}
    \bbE[S_n\mid\cF_{n-1}] &= S_{n-1}\exp(-\eta^2v_n)
    \bbE[\exp(\eta\xi_n)\mid\cF_{n-1}]\\
    &\le S_{n-1}\exp(-\eta^2v_n)\exp(\eta^2v_n)\\
    &= S_{n-1}.
\end{align*}
Therefore $(S_n)_{n\ge0}$ is a non-negative supermartingale with $S_0=1$.
By Ville's inequality, we then obtain $
    \bbP\left[
        \exists n\ge0:
        S_n\ge \frac1\delta
    \right] \le \delta.$
Equivalently, it holds that $
    \bbP\left[
        \exists n\ge0:
        \eta M_n-\eta^2\cV_n \ge \log \frac{1}{\delta}
    \right] \le \delta.$
Dividing by $\eta>0$ finishes the proof.
\end{proof}

\begin{lemma}[Lemma 20 in \cite{zhang2022horizon}]\label{lemma:li4}
	Let $H>0$, and let $X_1, X_2,\ldots$ be i.i.d. positive random variables.
	Define 
    \begin{align*}
    \tau_H:=\min\left\{ i\ge1:\sum_{j=1}^i X_j \geq H\right\}.
    \end{align*}
	Then we have that
	\begin{align*}
	\bbP\left[ \tau_{H}\geq  \frac{1}{2}\bbE[\tau_H] -1\right]\geq \frac{1}{2}.
	\end{align*}
\end{lemma}

\subsection{Proof of \Cref{lemma:cutnew}} \label{app:trunc_lemma_proof}

We re-state the lemma as follows.

\restatableCutNew*

With a slight abuse of notation, we rewrite the finite-horizon optimal value function as
\begin{align*}
    v(s,h):=\max_{\pi}
    \bbE_\pi\left[\sum_{t=1}^{h} r(s_t,a_t)\,\big|\, s_1 = s\right],
    \qquad v(s,0):=0\quad \text{for all }s\in\cS .
\end{align*}
Under the finite state and finite action assumptions, an optimal deterministic Markov policy exists.
Therefore $v(s,h)$ satisfies the Bellman recursion
\begin{align*}
    v(s,h)=\max_{a\in\cA} \left\{r(s,a)+\sum_{s'\in\cS}P_{s,a,s'}v(s',h-1)\right\},
    \qquad h\geq 1.
\end{align*}

Define the marginal increment $$
    \Delta_h(s):=v(s,h)-v(s,h-1)$$
    for $h\geq 1$ (where we set $v(s,0) =0$ for all $s$).
Because the one-step rewards are nonnegative, $ v(s,h)\geq v(s,h-1)$, hence $
    \Delta_h(s)\geq 0$. 
Fix an initial state $s_0\in\cS$ and a horizon $H\geq 1$.
For each $h=1,2,\ldots,H$ and $s\in\cS$, let $a_h(s)$ denote a Bellman-optimal action when the current state is $s$ and $h$ steps remain; these choices define a deterministic Markov policy that is optimal from every state for the corresponding remaining-time problem.
Define the corresponding transition matrix $P^h$ as $
    P^h_{s,s'}:=P_{s,a_h(s),s'}.$ 
Let $P^h_{s}$ denote the vector $[P^{h}_{s,s'}]^{\top}_{s'\in \mathcal{S}}$.
Then, for each $h\geq 1$,
\begin{align*}
    v(s,h)=r(s,a_h(s))+P_{s}^h v(\cdot, h-1).
\end{align*}

Let  $(s_t)_{t=1}^H$ be the trajectory started from $s_1=\tilde{s}$ and run under the above $H$-step optimal policy $\pi^*$. That is, $\pi^*_h(s) = \arg\max_{a} r(s,a) + P_{s,a} v(\cdot, h-1)$ for all proper $(s,h)$, where the tie is broken arbitrarily.

\begin{lemma} \label{lem:increment_prop}
For every state $s\in\cS$ and every $2\leq h\leq H$,
\begin{align*}
    \Delta_h(s)\leq (P^h\Delta_{h-1})(s) = \sum_{s'\in\cS}P^h_{s,s'}\Delta_{h-1}(s').
\end{align*}
As a result, 
   $ \Delta_H(\tilde{s})\leq \bbE[\Delta_{H+1-t}(s_t)]$ for $t=1,2, \ldots, H$.
\end{lemma}

\begin{proof} [Proof of \Cref{lem:increment_prop}]
By definition of $a_h(s)$, we have
\begin{align*}
   v(s,h)=r(s,a_h(s))+P^h v(\cdot, h-1).
\end{align*}
On the other hand,
by the optimality of $v(s,h-1)$, it holds
\begin{align*}
    v(s,h-1) \geq r(s,a_h(s))+P^h v(\cdot, h-2).
\end{align*}
Subtracting the two displays gives
\begin{align*}
    \Delta_h(s)= v(s,h)-v(s,h-1) \leq  P^h ( v(\cdot, h-1) - v(\cdot, h-2)) =(P^h\Delta_{h-1})(s).
\end{align*}

Along the $H$-step optimal policy, the number of remaining steps at time $t$ is $H-t$.
Thus, whenever $H -t\geq 1$,
\begin{align*}
    \Delta_{H+1-t}(s_t) \leq \bbE[\Delta_{H-t}(s_{t+1})\mid s_t].
\end{align*}
Iterating this inequality and taking expectations yields $
    \Delta_H(s_0) \leq \bbE[\Delta_{H+1-t}(s_t)]$ for any $t=1,2 \ldots, H$.
\end{proof}

Summing the inequality $\Delta_{H}(\tilde{s}) \leq \mathbb{E}[\Delta_{H+1-t}(s_t)]$ over $t=1,2, \ldots, H$, we obtain
\begin{align}\label{eq:sum-delta}
    H\Delta_H(\tilde{s}) \leq \sum_{t=1}^{H}
    \bbE[\Delta_{H+1-t}(s_t)].
\end{align}
We now bound the right-hand side.
For each state $s\in\cS$, define $
    L_s:=\sum_{t=1}^{H}
    \bbE[
        \mathbb{I}[s_t=s]\Delta_{H+1-t}(s)
    ].$ 
Then we have $
    \sum_{t=1}^{H}
    \bbE[\Delta_{H+1-t}(s_t)] =\sum_{s\in\cS}L_s$.

\begin{lemma} \label{lem:L_s}
For every state $s\in\cS$, $L_s\leq v(\tilde{s}, H)$.
\end{lemma}

\begin{proof} [Proof of \Cref{lem:L_s}]
Let $
    \tau_s:=\inf\{1\leq t\leq H:s_t = s\}$, 
with the convention that $\tau_s=\infty$ if $s$ is not visited during times $1,2, \ldots, H$.
Terms multiplied by $\mathbb{I}[\tau_s\le H]$ are interpreted as zero on $\{\tau_s=\infty\}$.
Consider a fixed sample path.
Suppose its visits to $s$ during times $1,2, \ldots, H$ occur at
\begin{align*}
    t_1<t_2<\cdots<t_m.
\end{align*}
Then its pathwise contribution to $L_s$ is $ 
    \sum_{\ell=1}^m \Delta_{H+1-t_\ell}(s).$ 
If $m=0$, this contribution is zero.  If $m\geq 1$, then $t_1=\tau_s$, and the set of remaining horizons $
    \{H+1-t_1,H+1-t_2,\ldots,H+1-t_m\}$ 
is a subset of $\{1,2,\ldots,H+1-\tau_s\}$.
Since all increments are nonnegative, in this case
\begin{align*}
    \sum_{\ell=1}^m \Delta_{H+1-t_\ell}(s) \leq \sum_{h=1}^{H+1-\tau_s}\Delta_h(s) =v(s, H+1-\tau_s),
\end{align*}
where the last equality is telescoping: for every integer $m'\ge1$,
    $\sum_{h=1}^{m'} \Delta_h(s) =\sum_{h=1}^{m'} (v(s,h)-v(s,h-1))=v(s,m')$.
Thus, in all cases, the pathwise contribution is at most
$\mathbb{I}[\tau_s\le H] v(s,H+1-\tau_s)$.  Therefore
\begin{align}\label{eq:Lx-first-hit}
    L_s\leq \bbE[
        \mathbb{I}[\tau_s\leq H]v(s,H+1-\tau_s)
    ].
\end{align}

It remains to bound the right-hand side by $v(\tilde{s},H)$.
Along the fixed $H$-step optimal policy, define
\begin{align*}
    M_t:=\sum_{u=1}^{t-1}r(s_u,a_{H+1-u}(s_u)) +v(s_t, H+1-t),
    \qquad t=1,2,\ldots,H+1.
\end{align*}
For $t=1,2, \ldots, H+1$, the Bellman equation implies
\begin{align*}
   v(s_t, H+1-t) =r(s_t,a_{H+1-t}(s_t)) +\bbE[v(s_{t+1} , H-t)\mid s_t],
\end{align*}
so $(M_t)_{t=1}^{H+1}$ is a martingale and $M_1=v(\tilde{s},H)$.
Let $\sigma:=\tau_s\wedge (H+1)$, where $\infty\wedge (H+1)=H+1$.
Since $\sigma$ is bounded, optional stopping gives $
    \bbE[M_\sigma]=M_1=v(\tilde{s}, H).$ 
Because rewards are nonnegative, 
\begin{align*}
    M_\sigma\geq \mathbb{I}[\tau_s\leq H] v(s,H+1-\tau_s).
\end{align*}
Combining this inequality with \Cref{eq:Lx-first-hit} gives $
    L_s\leq v(\tilde{s}, H).$
\end{proof}

Summing the inequality $L_{s}\leq v(\tilde{s}, H)$ over $s\in\cS$, we get
\begin{align*}
    \sum_{t=1}^{H}
    \bbE[\Delta_{H+1-t}(s_t)] =\sum_{s\in\cS}L_s \leq S  v(\tilde{s}, H).
\end{align*}
Together with \Cref{eq:sum-delta}, this yields the key estimate
\begin{align}\label{eq:key}
    H\Delta_H(\tilde{s}) \leq S  v(\tilde{s}, H),
\end{align}
As a result, we have $\Delta_H(\tilde{s})\leq \frac{S}{H} v(\tilde{s}, H)$.
Now assume $H>S$.
Since $\Delta_H(\tilde{s})= v(\tilde{s}, H)- v(\tilde{s}, H-1)$, \Cref{eq:key} implies
\begin{align*}
    v(\tilde{s}, H)-v(\tilde{s}, H-1) \leq \frac{S}{H}v(\tilde{s}, H).
\end{align*}
Rearranging the inequality, we have that
\begin{align*}
v(\tilde{s}, H) \leq \frac{H}{H-S} v(\tilde{s}, H-1).
\end{align*}

That is, for the fixed initial state $\tilde{s}$,
\begin{align*}
f(H)\leq \left(1 + \frac{S}{H-S}\right)f(H-1).
\end{align*}
The boundary case $H=S+1$ is included because $1-S/H=1/(S+1)>0$.

It remains to justify the big-step consequence.
Fix an integer $\ell\ge1$ and form the $\ell$-step macro-MDP whose state space is the same $\cS$.
A macro-action is a deterministic Markov controller for the next $\ell$ primitive steps; its macro-transition kernel is the law of the primitive state after those $\ell$ steps, and its macro-reward is the expected sum of the $\ell$ primitive rewards collected in the block.
The macro-action space is finite, and the macro-rewards are nonnegative.
Moreover, for every integer $k\ge0$, the optimal $k$-step value of this macro-MDP from $\tilde{s}$ is exactly the original primitive value $f(k\ell)$: any macro-policy is a finite-horizon history-dependent primitive policy, which cannot improve on the deterministic Markov optimum, and any primitive deterministic Markov policy induces a macro-policy by restricting it to consecutive blocks.
Applying the one-step comparison already proved to this macro-MDP at macro-horizon $t>S$ gives
\begin{align*}
f(t\ell)\leq \left(1 + \frac{S}{t-S}\right)f((t-1)\ell).
\end{align*}
This proves the lemma.

\newpage

\section{Missing Algorithms}\label{app:missing_alg}

In this section, we introduce the two main subroutines:
\Cref{alg:plann}, which performs optimistic planning, and \Cref{alg:explore}, which carries out explicit exploration.

\subsection{The \texttt{Planning} Algorithm}

\Cref{alg:plann} is an optimistic planning procedure on the truncated horizon $H_1$.
We set the clipping threshold $\epsilon = \frac{1}{S^2}$.
For $(s,a)\in\cO$ and $h\in [H_1]$, the algorithm sets $Q_h(s,a)=1$.
For $(s,a)\notin\cO$, it uses the frozen empirical transition row $\widehat P_{s,a}$ and the bonus $ 100\,b\left(\widehat P_{s,a},V_{h+1},N(s,a)\right)$ for the backward planning:
for each $(s,a)\notin \cO$ and $h = H_1,H_1-1,\ldots, 1$,
\begin{align*}
 &  Q_{H_1+1}(s,a)= 0;
 \\ & Q_{h}(s,a) = r(s,a) + \widehat{P}_{s,a}V_{h+1} + 100b\left(\widehat{P}_{s,a}, V_{h+1}, N(s,a)\right)
 \\ & V_{h}(s) = \max_{a}Q_{h}(s,a).
\end{align*}

Here the bonus function $b(p,v,n)$ is defined as 
\begin{align}
    b(p,v,n)
    =
    3\sqrt{\frac{\bbV(p,v)\log \frac{1}{\delta}}{n}}
    +
    5\sqrt{\frac{S\bbV_{5\epsilon}(p,v)\log \frac{1}{\delta}}{n}}
    +
    20\frac{S\log \frac{1}{\delta}}{n}.\label{eq:bonus}
\end{align}

\subsection{The \texttt{Exploration} Algorithm}
\label{subsec:exploration-algorithm}

\Cref{alg:explore} is the explicit exploration subroutine used after the planning policy first reaches an unlearned state-action pair $\left(\widetilde s,\widetilde a\right)\in\cO$.
This subroutine is adapted from the explicit-exploration procedure of~\cite{zhang2022horizon}, with minor modifications to fit our algorithmic framework.
We briefly summarize the intuition here and defer a more detailed discussion to \Cref{app:expexp}.
We also refer readers to \cite{zhang2022horizon} for  more details.

The goal of \Cref{alg:explore} is twofold.
First, if there is still some reachable state-action pair $(s,a)$ with an unknown outgoing triple not in $\cK$, then the subroutine spends the suffix of the episode collecting samples that may enlarge the known set $\cK$.
Second, if no such auxiliary exploration is needed, then the subroutine uses the reference model to collect samples of the original target pair $\left(\widetilde s,\widetilde a\right)$.
In the latter case, the call is counted as an effective exploration call for $\left(\widetilde s,\widetilde a\right)$.

The remaining suffix has length $\frac{H}{m}$.
We split it into two parts:
$ H_3:=\frac{H}{20mS\log S}$ and $H_2:=\frac{H}{m}-H_3$, where the long phase of length $H_2$ is used for reaching a candidate state and the short phase of length $H_3$ is used for collecting samples of a candidate state-action pair.
Accordingly, the reaching phase uses $\pi_2$ for up to $H_2$ steps, while the sampling phase uses $\pi_3$ for $H_3$ steps.
We use the following two quantities to measure discounted occupancy and discounted reachability under a stationary policy:
\begin{align}
 & X_\gamma^\pi(\cX,p,\mu)
    :=
    \sum_{h\ge1}
    \gamma^{h-1}
    \bbP_{p,\pi}
    [
        s_h\in\cX,\ 
        s_t\notin\cX,\ \forall t<h
        \mid s_1\sim\mu
    ],\label{eq:defX}
\\ &   W_\gamma^\pi(g,p,\mu)
    :=
    \bbE_{p,\pi}
    \left[
        \sum_{h\ge1}\gamma^{h-1}g(s_h,a_h)
        \,\middle|\,s_1\sim\mu
\right]. \label{eq:defW}
\end{align}

For each candidate pair $(s,a)$ with at least one unknown outgoing triple, (i.e., with some $s'\in\cS$ satisfying $(s,a,s')\notin \cK$), \Cref{alg:explore} computes two stationary policies.
The reaching policy $\pi_2$ tries to reach the candidate state $s$ from the current target $\left(\widetilde s,\widetilde a\right)$:
\begin{align*}
    \pi_2
    \in
    \argmax_{\pi\in\Pi_{\sta},\,\pi\left(\widetilde s\right)=\widetilde a}
    X_\gamma^\pi\left(\{s\},P^{\tref},\II_{\widetilde s}\right),
\end{align*}
where $\Pi_{\sta}$ denotes the class of stationary policies.
We write $ u(s) := X_\gamma^{\pi_2}\left(\{s\},P^{\tref},\II_{\widetilde s}\right)$ for the corresponding reachability estimate.
The sampling policy $\pi_3$ tries to collect visits to $(s,a)$, starting from $s$:
\begin{align*}
    \pi_3
    \in
    \argmax_{\pi\in\Pi_{\sta}}
    W_\gamma^\pi(\II_{s,a},P^{\tref},\II_s),
\end{align*}
and we write $ v(s,a) := W_\gamma^{\pi_3}(\II_{s,a},P^{\tref},\II_s).$
The test
\begin{align*}
    u(s)\ge \frac{1}{1200S},
    \qquad
    D(s,a)\le 1620S^2A N_{\tref}u(s)v(s,a)
\end{align*}
checks whether $(s,a)$ is both reachable and still under-sampled relative to its estimated collectability under the reference model.  Here $
D(s,a):=\max\left\{1,\sum_{s':(s,a,s')\in\cK}\Ntotal(s,a,s')\right\}$ 
is the exploration-local count of currently known successors; it is distinct from the global raw visit count.
If the test succeeds, the algorithm first runs $\pi_2$ for at most $H_2$ steps, stopping as soon as the candidate state $s$ is reached or an unknown triple in $\cK^\complement$ is observed.
If $s$ is reached before an unknown transition, the algorithm then runs $\pi_3$ for $H_3$ steps to collect samples of $(s,a)$;
otherwise it plays an arbitrary policy until the end of the episode.

If no candidate pair passes the test, then the subroutine regards the reference model as sufficiently accurate for the current target.
In this case, it sets $ \trigger=\true$, computes $ \pi \in \argmax_{\pi\in\Pi_{\sta}} W_\gamma^\pi\left( \II_{\widetilde s,\widetilde a}, P^{\tref}, \II_{\widetilde s} \right)$ and runs $\pi$ until the end of the episode.
This is counted as an effective sample-collection call for $\left(\widetilde s,\widetilde a\right)$, and the main algorithm increments $M\left(\widetilde s,\widetilde a\right)$.
Once $M\left(\widetilde s,\widetilde a\right)\ge N_{\known} = 10000\log\frac{1}{\delta}$, the pair is removed from $\cO$.

The key property needed in the regret analysis is that each pair $(s,a)\in\cO$ can trigger only polynomially many calls to \Cref{alg:explore}.
This follows from the same clipped-model and stationary-policy counting ideas as in~\cite{zhang2022horizon}, with the minor modification that our call to \Cref{alg:explore} occurs as a suffix after reward-aware optimistic planning reaches $\cO$.

\begin{remark}
Fix an arbitrary deterministic order on $\SA$ for the loop in \Cref{alg:explore}, and use deterministic tie-breaking for every displayed $\argmax$.
In \Cref{alg:explore}, the optimization problems $ \argmax_{\pi\in\Pi_{\sta}} X_{\gamma}^{\pi}(\cX,p,\mu)$ and $\argmax_{\pi\in\Pi_{\sta}} W_{\gamma}^{\pi}(g,p,\mu)$ can be solved efficiently by standard discounted Bellman iteration.

\end{remark}

\begin{algorithm}[t]
\caption{\texttt{Planning}}
\label{alg:plann}
\begin{algorithmic}[1]
\STATE{\textbf{Input:} $\cO$, frozen counts $\{N(s,a,s')\}_{s,a,s'}$, cutting resolution $\epsilon$.}
\STATE{\textbf{Initialize:} $V_{H_1+1}(s)\gets0$ for all $s\in\cS$, and $Q_h(s,a)\gets1$ for all $(s,a)\in\cO$ and $h\in[H_1]$.}
\STATE{Construct $\{N(s,a)=\max\left\{1,\sum_{s'}N(s,a,s')\right\}\}_{s,a}$ and $\left\{\widehat P_{s,a,s'} = \frac{N(s, a, s')}{N(s, a)}\right\}_{s,a, s'}$.}

\FOR{$h=H_1,H_1-1,\ldots,1$}
    \STATE{For all $(s,a)\notin\cO$, set $b_h(s,a)\gets100\,b\left(\widehat P_{s,a},V_{h+1},N(s,a)\right)$ (see \Cref{eq:bonus}).}
    \STATE{For all $(s,a)\notin\cO$, set
    \begin{align*}
        Q_h(s,a)\gets
        \min\left\{1,\ r(s,a)+\widehat P_{s,a}V_{h+1}+b_h(s,a)\right\}.
    \end{align*}}
    \STATE{Set $V_h(s)\gets\max_{a\in\cA}Q_h(s,a)$ and $\pi_h(s)\gets\argmax_{a\in\cA}Q_h(s,a)$ for all $s\in\cS$.}
\ENDFOR

\STATE{\textbf{return:} $\pi$.}
\end{algorithmic}
\end{algorithm}

\begin{algorithm}[t]
\caption{\texttt{Exploration}}
\label{alg:explore}
\begin{algorithmic}[1]
\STATE{\textbf{Input:} initial state-action pair $\left(\tilde{s},\tilde{a}\right)$, reference model $P^{\tref}$, raw counts $\{\Ntotal(s,a,s')\}_{s,a,s'}$, known triples $\cK$.}
\STATE{\textbf{Initialize:} $H_2\gets \frac{H}{m}-\frac{H}{20mS\log S}$, $H_3\gets \frac{H}{20mS\log S}$, $\gamma\gets1-\frac{1}{H_3}$, $\trigger\gets\false$, and $\aux\gets\false$.}
\STATE{Set $
    D(s,a)\gets
    \max\left\{
        1,\sum_{s':(s,a,s')\in\cK}\Ntotal(s,a,s')
    \right\}$
for all $(s,a)\in\SA$.}
\STATE{Define the functions $X_{\gamma}^{\pi}(\cdot, \cdot, \cdot)$ and $W_{\gamma}^{\pi}(\cdot, \cdot, \cdot)$ following \Cref{eq:defX,eq:defW}. }
\FOR{all $(s,a)\in\SA$}
    \IF{there exists $s'\in\cS$ such that $(s,a,s')\notin\cK$}
        \STATE{$\pi_2\gets\argmax_{\pi\in\Pi_{\sta},\,\pi\left(\tilde{s}\right)=\tilde{a}}X_\gamma^\pi\left(\{s\},P^{\tref},\II_{\tilde{s}}\right)$; set $u(s)\gets X_\gamma^{\pi_2}\left(\{s\},P^{\tref},\II_{\tilde{s}}\right)$.}
        \STATE{$\pi_3\gets\argmax_{\pi\in\Pi_{\sta}}W_\gamma^\pi(\II_{s,a},P^{\tref},\II_s)$; set $v(s,a)\gets W_\gamma^{\pi_3}(\II_{s,a},P^{\tref},\II_s)$.}
        \IF{$u(s)\ge\frac{1}{1200S}$ and $ D(s,a)\le 1620S^2A N_{\tref}u(s)v(s,a)$}\label{line:condition}
            \STATE{Set $\aux\gets\true$.}
            \STATE{Run $\pi_2$ for at most $H_2$ steps, stopping as soon as state $s$ is reached or an unknown triple in $\cK^\complement$ is observed.} \label{line:explore_reaching}
            \IF{state $s$ is reached before an unknown transition}
                \STATE{Run $\pi_3$ for $H_3$ steps.} \label{line:explore_sampling}
                \STATE{Play an arbitrary policy until the end of the episode.}
            \ELSE
                \STATE{Play an arbitrary policy until the end of the episode.}
            \ENDIF
            \STATE{\textbf{break}}
        \ENDIF
    \ENDIF
\ENDFOR

\IF{$\aux=\false$}
    \STATE{Set $\trigger\gets\true$.}
    \STATE{$\pi\gets\argmax_{\pi\in\Pi_{\sta}}W_\gamma^\pi\left(\II_{\tilde{s},\tilde{a}},P^{\tref},\II_{\tilde{s}}\right)$.} \label{line:explore_pi}
    \STATE{Run $\pi$ until the end of the episode.}
\ENDIF

\STATE{\textbf{return:} $\trigger$.}
\end{algorithmic}
\end{algorithm}

\newpage

\section{Notations and Naming Conventions}\label{app:notation}

Throughout the appendices, unless explicitly stated otherwise, all sums over layers are over the truncated horizon, namely $h\in[H_1]$.
The notation specific to the explicit exploration subroutine, \Cref{alg:explore}, is collected separately in \Cref{app:expexp};
this section records the notation used in the regret analysis of \Cref{alg:new} and \Cref{alg:plann}.

\paragraph{Basic notations.}
For a positive integer $N$, let $[N]:=\{1,2,\ldots,n\}$.
For an event $\cE$, let $\bbI[\cE]$ denote its indicator.
We use $\Delta(\cS)$ for the probability simplex over $\cS$.
For $p\in\Delta(\cS)$ and $v\in\bbR^{\cS}$, recall that $pv:=\sum_{s\in\cS}p(s)v(s)$ and $ \bbV(p,v):=p v^2-(pv)^2 =\sum_{s\in\cS}p(s)(v(s)-pv)^2.$
    
For $x>0$, recall that the $x$-clipped variance is
\begin{align*}
    \bbV_x(p,v)
    :=
    \sum_{s\in\cS}
    p(s)\min\{(v(s)-pv)^2,x^2\}.
\end{align*}
For two vectors $u,v\in\bbR^{\cS}$, we write $u\le v$ if the inequality holds coordinatewise.
The vectors $\II_s$, $\II_{s,a}$, and $\II_{s,a,s'}$ denote the corresponding one-hot vectors or indicator reward functions.
The meaning will be clear from context.
We use $\bbE_{p,\pi}[\cdot ]$ denote the expectation under transition model $p$ and policy $\pi$.
When $p$ is clear from the context, we simplify it as $\bbE_{\pi}[\cdot ]$.

\begin{definition}[Maximal expected visit count.]\label{def:umax}
Define the maximal expected visit count to $(s,a)$ in one episode of $H$ steps as:
\begin{align*}
U(s,a) =\max_{s',\pi\in \Pi}\sum_{h=1}^H \bbE_{\pi}\left[ \left. \sum_{h=1}^H \bbI[(s_h,a_h) = (s,a)] \right| s_1 = s'\right] .
\end{align*}
\end{definition}

\paragraph{Parameters settings.}
The main parameters are
\begin{align*}
    \upsilon:=\min\left\{\sqrt{\frac{SA}{1000K}},\frac{1}{20S\log S}\right\},
    \qquad
    \epsilon:=\frac{1}{S^2},
    \qquad
    m:=\frac{4S}{\upsilon},
    \qquad
    H_1:=H-\frac{H}{m}.
\end{align*}
We assume, for notational simplicity, that $m$ divides $H$ and that the phase lengths $H_2$ and $H_3$ are positive integers; rounding to adjacent integers only changes universal constants.
The reference and known-pair thresholds are $N_{\rm ref}:=1025S^2\log \frac{1}{\delta}$ and $N_{\known}:=7000\log \frac{1}{\delta}$.

\paragraph{Raw counts.}
Let $\Ntotal^k(s,a,s')$ denote the total number of observed transitions $(s,a,s')$ before episode $k$, and $\Ntotal^k(s,a):=\sum_{y\in\cS}\Ntotal^k(s,a,y)$.

\paragraph{Frozen count and empirical models.}
If $(s,a)\in\cO^k$, then $N^k(s,a,s'):=0$ for every $s'$, and we use the dummy denominator $N^k(s,a):=1$.
If $(s,a)\notin\cO^k$, we have $N^k(s,a):=2^{m^k(s,a)}$, where $m^k(s,a):=\max\{m\in\bbN_0:2^m\le \Ntotal^k(s,a)\}$, Let
\begin{align*}
    Y_1(s,a),Y_2(s,a),\ldots,Y_{\Ntotal^k(s,a)}(s,a)
\end{align*}
be the ordered list of successors observed on visits to $(s,a)$ before episode $k$.
Then
\begin{align*}
    N^k(s,a,s')
    :=
    \sum_{i=1}^{N^k(s,a)}\bbI\{Y_i(s,a)=s'\},
    \qquad
    \widehat P_{s,a,s'}^k:=\frac{N^k(s,a,s')}{N^k(s,a)}.
\end{align*}

\paragraph{Algorithm-state notation.}
Let $\cO^k$ denote the set of unlearned state-action pairs,
$\cK^k$ denote the set of known state-action-state triples, and
$M^k(s,a)$ denote the effective-exploration counter.

Let $Q_h^k$ and $V_h^k$ be the $Q$-value and value function computed by
\Cref{alg:plann} in episode $k$, and let $\pi^k = \{\pi^k_h\}_{h=1}^H$ be the policy
computed by \Cref{alg:plann} in episode $k$. We denote the rollout
trajectory in episode $k$ by
\begin{align*}
    (s_1^k,a_1^k,s_2^k,a_2^k,\ldots,s_H^k,a_H^k).
\end{align*}

We also use the shorthand notation $
    P_h^k := P_{s_h^k,a_h^k}, 
    P_{h,s}^k := P_{s_h^k,a_h^k,s}, 
    \widehat P_h^k := \widehat P_{s_h^k,a_h^k}^k, 
    \widehat P_{h,s}^k := \widehat P_{s_h^k,a_h^k,s}^k$,  $
    N_h^k := N^k(s_h^k,a_h^k)$ and $r_h^k := r(s_h^k,a_h^k)$.

\paragraph{Confidence parameters.}
Every confidence parameter lies in $(0,1)$.  No additional lower bound on
$\log\frac1\delta$ is imposed by the lower-level probabilistic results.
Whenever a proof uses a constant lower bound on this logarithm, it follows
from the nontrivial regime of that result's stated probability guarantee.

\newpage

\section{Regret Analysis}\label{sec:reg_analysis}
In this section, we present the full proof of \Cref{thm:reg}.
We work under the explicit conditions $S\geq200$, $A\geq 8$, $H\geq K$, with $m$ a positive integer dividing $H$, and use the integer phase convention $d=\frac{H}{m}$,  under the suffix condition $\frac{d}{20S\log S}\geq22$, as stated in \Cref{thm:reg}. 
We also assume
$K\geq1000S^2A\log\frac1\delta$ where we recall $\delta = \frac{\delta_0^2}{210S^7A^2K}\leq \frac{1}{210S^7A^2K}$. The regret bound in the regime  $K\leq1000S^2A\log\frac1\delta$ is simply $\Regret_H(K)\leq K$. 

The proof of \Cref{thm:reg} is organized around five main steps.
The first step reduces the original $H$-horizon regret to the truncated horizon $H_1=H-\frac{H}{m}$.
The next three steps control the truncated regret under the optimistic planning policy:
we first prove optimism of the planned value functions, then decompose the stopped regret into Bellman-error and martingale terms, and finally bound the accumulated Bellman errors through a variance closure argument.
The last step controls the cost of stopping and explicit exploration.
Combining these ingredients yields the claimed horizon-free regret bound.

More precisely, the full proof proceeds as follows.

\begin{enumerate}
    \item \textbf{Truncating the horizon; see \Cref{sec:trun}.}
    We first compare the original $H$-step problem with the truncated $H_1$-step problem.
    The horizon-truncation lemma implies that the optimal $H$-step value is within $O(\upsilon)$ of the optimal $H_1$-step value.
    Therefore,
    \begin{align*}
        \Regret_H(K)
        \le
        \Regret_{H_1}(K)+K\upsilon .
    \end{align*}
    With our choice of $\upsilon$, the truncation loss $K\upsilon$ is absorbed by the leading $\widetilde O(\sqrt{SAK})$ term.
    Hence it remains to analyze the regret on the truncated horizon.

    \item \textbf{Optimism of $Q_h^k$ and $V_h^k$; see \Cref{sec:opt} and \Cref{sec:p_opt}.}
    We prove that the planning procedure is optimistic:
    with high probability, $Q_h^k(s,a)\ge Q_h^\star(s,a)$ and $V_h^k(s)\ge V_h^\star(s)$ for all episodes $k$, layers $h\in[H_1]$, states $s$, and actions $a$.
    The proof is a backward induction on $h$.
    For $(s,a)\in\cO^k$, optimism holds because the algorithm sets $Q_h^k(s,a)=1$.
    For $(s,a)\notin\cO^k$, optimism follows from the concentration guarantee of the cutting bonus and the relaxed monotonicity property of the bonus.

    \item \textbf{Stopped regret decomposition; see \Cref{sec:reg_decomposition}.}
    We next decompose the truncated regret using the stopping indicators $I_h^k$.
    The process is stopped when the trajectory reaches the unlearned set $\cO^k$ or when the frozen count of a learned pair doubles.
    On the optimism and Bellman-error events, the stopped regret is bounded by the sum of the Bellman-error envelope, martingale terms, and the stopping cost.

    \item \textbf{Recursive variance and Bellman-error bounds; see \Cref{sec:p_totalvar}.}
    The main quantitative step is to bound the accumulated Bellman error $B$.
    This requires a variance closure argument for the optimistic values and the optimistic gaps.
    The ordinary variance terms are controlled by the stopped regret structure, while the clipped-variance terms are controlled by the total-deviation bound for the monotone optimal value sequence.
    This is where the cutting projection and the clipped variance are used to avoid both a $\log H$ factor and an additional leading polynomial dependence on $S$.
    The result is a self-bounding inequality for $B$, which is then solved by a standard absorption argument.

    \item \textbf{Bounding the exploration error; see \Cref{sec:b_exp}.}
    Finally, we bound the total stopping cost.
    The doubling events are controlled by the frozen-count doubling rule and contribute only logarithmic factors.
    The exploration events are controlled through the explicit exploration analysis of \Cref{alg:explore}:
    once a pair $(s,a)\in\cO$ has been effectively explored sufficiently many times, it is removed from $\cO$, and the number of ineffective calls is bounded using the clipped-reference-model argument and stationary policy analysis.
    This yields a polynomial burn-in term independent of $H$.

\end{enumerate}

Putting these five pieces together, we first obtain a high-probability bound for $\Regret_{H_1}(K)$.
Adding the truncation loss $K\upsilon$ and substituting the choices of $\upsilon$, $\epsilon$, and the logarithmic parameters gives the stated bound in \Cref{thm:reg}.

\subsection{Truncating the Horizon}\label{sec:trun}

Recall that $ H_1 := H-\frac{H}{m}$.
Recall that $\Regret_H(K)$ denote the regret in $K$ episode with horizon $H$ as
\begin{align*}
\Regret_H(K) & = \sum_{k=1}^K (  V_1^*(s_1^k) - V_1^{\pi^k}(s_1^k)) \notag\\
& = \sum_{k=1}^K\left(  \max_{\pi}\bbE_{\pi}\left[\sum_{h=1}^H r(s_h,a_h) \,\big| s_1 = s_1^k\right] -\bbE_{\pi^k}\left[\sum_{h=1}^H r(s_h,a_h)\,\big| s_1 = s_1^k\right] \right) .\notag
\end{align*}

Define $\Regret_{H_1}(K)$ as
\begin{align*}
\Regret_{H_1}(K) := \sum_{k=1}^K\left(  \max_{\pi}\bbE_{\pi}\left[\sum_{h=1}^{H_1} r(s_h,a_h) \,\big| s_1 = s_1^k\right] -\bbE_{\pi^k}\left[\sum_{h=1}^{H_1} r(s_h,a_h) \,\big| s_1 = s_1^k\right] \right).\notag
\end{align*}

Put $d:=\frac{H}{m}$, so that $H=md$ and $H_1=(m-1)d$ are integers under the conditions of \Cref{thm:reg}.  Since $m=\frac{4S}{\upsilon}>S$, \Cref{lemma:cutnew} applies with the integers $t=m$ and $\ell=d$.  The total-reward normalization and $m\geq2S$ give, for each $k$,
\begin{align*}
  \max_{\pi}\bbE_{\pi}\left[\sum_{h=1}^H r(s_h,a_h) \,\big| s_1 = s_1^k\right]
  &\leq
  \left(1+\frac{S}{m-S}\right)
  \max_{\pi}\bbE_{\pi}\left[\sum_{h=1}^{H_1} r(s_h,a_h)\,\big| s_1 = s_1^k\right]
  \notag\\
  &\leq
  \max_{\pi}\bbE_{\pi}\left[\sum_{h=1}^{H_1} r(s_h,a_h)\,\big| s_1 = s_1^k\right]
  +\frac{2S}{m}
  \notag\\
  &\leq
  \max_{\pi}\bbE_{\pi}\left[\sum_{h=1}^{H_1} r(s_h,a_h)\,\big| s_1 = s_1^k\right]
  +\upsilon,
  \notag\\
  \bbE_{\pi^k}\left[\sum_{h=1}^{H_1} r(s_h,a_h) \,\big| s_1 = s_1^k\right]
  &\leq
  \bbE_{\pi^k}\left[\sum_{h=1}^H r(s_h,a_h)\,\big| s_1 = s_1^k\right].
\end{align*}
The last inequality uses only nonnegativity of rewards.  Subtracting it from the optimal-value comparison and summing over episodes yields
\begin{align*}
    \Regret_H(K) \leq \Regret_{H_1}(K) + K\upsilon .
\end{align*}
It remains to bound $\Regret_{H_1}(K)$.
In the following, we assume the optimal function are defined with respect to this truncated MDP with horizon $H_1$.

\subsection{Optimistic Planning}\label{sec:opt}

We display the optimism lemma and its here, and leave most of the utility lemmas in \Cref{sec:p_opt}. Define \begin{align}
L_N := 2 \log \left(\frac{504S^2K\log S}{\upsilon\delta\log\frac1\delta}\right). \label{eq:def_L_N}
\end{align}

\begin{lemma}\label{lemma:optimism}
With probability at least
$1-\left(\frac{6S^2AL_N}{\epsilon^2}+2S^3A+2SA\right)\delta$,
it holds that $V_h^k(s) \geq V_h^*(s)$ for all $k\in [K]$,
$h\in [H_1]$ and $s\in \cS$.
\end{lemma}

\begin{proof}
Let $\cE=\cE_{\mathsf{1step}}$ be the event from
\Cref{lemma:one_step_conc} with failure probability  at most $\left(\frac{6S^2AL_N}{\epsilon^2}+2S^3 A + 2SA \right) \delta$.  Assume $\cE$ holds.
We prove by backward induction on $h$ that: $Q_h^k(s,a)\ge Q_h^*(s,a)$ for all $(s,a)$ and $h\in [H_1]$.
The terminal condition holds because $V_{H_1+1}^k=V_{H_1+1}^*\equiv0$ by convention.

Assume the value claim holds at layer $h+1$ and fix $k\in[K]$, $s\in\cS$, and $a\in\cA$.
If $(s,a)\in\cO^k$, then \Cref{alg:plann} sets $Q_h^k(s,a)=1$, while $Q_h^*(s,a)\le1$ because the total reward is bounded by one.
If $Q_h^k(s,a)\ge1$, the same bound on $Q_h^*$ gives optimism.
It remains to consider a learned state-action pair $(s,a)\notin\cO^k$ with $Q_h^k(s,a)<1$.
Then the clipping in \Cref{alg:plann} is inactive, $N^k(s,a)>0$, and
\begin{align*}
Q_h^k(s,a) &  = r(s,a) + \hat{P}^k_{s,a}V_{h+1}^k
+ 100 b\left(\hat{P}^k_{s,a}, V_{h+1}^k, N^k(s,a)\right).
\end{align*}
The planning recursion and the reward normalization give $V_{h+1}^k,V_{h+1}^*\in[0,1]^S$, and the induction hypothesis gives $V_{h+1}^k\ge V_{h+1}^*$.
Thus \Cref{lemma:mono} applies with $v=V_{h+1}^k$, $u=V_{h+1}^*$,
$p=\hat P_{s,a}^k$, and sample size $N^k(s,a)$, yielding
\begin{align*}
Q_h^k(s,a)
&\ge r(s,a) + \hat{P}^k_{s,a} V_{h+1}^*
+ \ud{b}\left(\hat{P}_{s,a}^k, V_{h+1}^* , N^k(s,a)\right).
\end{align*}
Since the truncated horizon is chosen so that $H_1+1\le H$, the index $h+1$ belongs to $[H]$ for every $h\in[H_1]$.
Applying \Cref{eq:underline_optimism_explicit} on $\cE$ at layer $h+1$ gives
\begin{align*}
 \hat{P}^k_{s,a} V_{h+1}^*
 + \ud{b}\left(\hat{P}_{s,a}^k, V_{h+1}^* , N^k(s,a)\right)
 \geq P_{s,a}V_{h+1}^*.
\end{align*}
Combining this with the preceding display,
\begin{align*}
 Q_h^k(s,a) \geq r(s,a) + P_{s,a}V_{h+1}^* = Q_h^*(s,a).
\end{align*}
Taking the maximum over the action set gives $V_h^k(s)=\max_a Q_h^k(s,a)\ge \max_a Q_h^*(s,a)=V_h^*(s)$, completing the induction.
\end{proof}

\subsection{Regret Decomposition}\label{sec:reg_decomposition}

For episode $k$, define the exploration event
\begin{align*}
    E^k := \{
        \exists h \le H_1 :
        (s_h^k,a_h^k) \in O^k
    \}.
\end{align*}
On this event, \Cref{alg:new} stops following the planning policy and calls \Cref{alg:explore}.
Define also the doubling event
\begin{align*}
    U^k := \left\{ \exists (s,a) \notin \cO^k,  \frac{N^{k+1}(s,a)}{N^k(s,a)}\geq 2
    \right\}.
\end{align*}
Let the stopping time
\begin{align*}
    \tau_k := \inf\left\{
        h \le H_1 :
        (s_h^k,a_h^k)\in O^k
       \text { or }  \exists (s,a)\notin \cO^k, \frac{\Nhtotal^k(s,a)}{N^k(s,a)}\geq 2
    \right\},
\end{align*}
with the convention $\tau_k=H_1+1$ if neither event occurs.
Here $\Nhtotal^k(s,a)$ denote the total visit count of $(s,a)$ before the $h$-th step in the $k$-th episode.
Define the stopped good-step indicator $
    I_h^k := \bbI[h< \tau_k].$
Thus the regret analysis only checks the trajectory before either an unlearned pair is reached or the doubling schedule is violated.

On the other hand, we define the stopping cost to measure the regret due to early stop:
\begin{align}
    L_{\stoplabel} := \sum_{k=1}^K  \bbI[E^k] +
    \sum_{k=1}^K  \bbI[U^k].\label{def:Lstop}
\end{align}

For each $(s,a)$, let $k_0(s,a)$ be the first episode $k$ such that $(s,a)\notin \cO^k$. If such $k$ does not exist, we set $k_0(s,a) = K+1$.
Let $\cE_{\mathsf{floor}}$ be the event in \Cref{lemma:exp1}.  Its
failure probability is at most $(2S^3A+SA)\delta$, and on it every removed
pair, with first learned episode $k_0(s,a)$, satisfies
\begin{align}
N^{k_0(s,a)}(s,a)\geq\frac{\upsilon\log\frac1\delta}{504S^2\log S}U(s,a).\label{eq:k_ini}
\end{align}

\begin{lemma}[Inverse-count sums]
\label{lem:inverse-count-section2}
There exists an event $\cE_{\mathsf{count}}$ satisfying $    \bbP(\cE_{\mathsf{count}}^\complement)
    \leq(2S^3A+2SA)\delta$
on which the following bounds hold simultaneously for every collection of
nonnegative predictable weights $w_h^k\in[0,1]$:
\begin{align*}
    \sum_{k=1}^K \sum_{h=1}^{H_1} I_h^k \frac{1}{N^k(s_h^k,a_h^k)} \le S A L_N ,
\end{align*}
and
\begin{align*}
    \sum_{k=1}^K \sum_{h=1}^{H_1}
    I_h^k
    \sqrt{\frac{w_h^k}{N^k(s_h^k,a_h^k)}} \le \sqrt{
        S A L_N
       \sum_{k=1}^K \sum_{h=1}^{H_1} I_h^k w_h^k
    } .
\end{align*}
\end{lemma}

\begin{proof} We assume the event $\cE_{\mathsf{floor}}$. 
For each episode $k$, define
\begin{align*}
    X_k(s,a) := \sum_{h=1}^{H_1} I_h^k \bbI[(s_h^k,a_h^k)=(s,a)],
    \quad
    Y_k(s,a) := \sum_{h=1}^{H} \bbI[(s_h^k,a_h^k)=(s,a)].
\end{align*}
The raw counts, and only the raw counts, satisfy $
\Ntotal^{k+1}(s,a)
=\Ntotal^k(s,a)+Y_k(s,a).$
If $I_h^k=1$, then
$N^k(s_h^k,a_h^k)$ is a positive power of two.  Thus the dummy value $1$
for unlearned state-action pairs never contributes to either sum in the statement.

Let $\cF_{k-1}$ denote the history before episode $k$.  Conditionally
on this history, the behavior in that episode is a policy, and hence
$\bbE[Y_k(s,a)\mid\cF_{k-1}]\leq U(s,a)$.  Consequently,
\begin{align*}
\bbE\left[\Ntotal^{K+1}(s,a)\right]\leq KU(s,a).
\end{align*}
For each fixed pair, Markov's inequality gives
$\Ntotal^{K+1}(s,a)\leq \frac{KU(s,a)}{\delta}$ except on an event of
probability at most $\delta$. A union bound over the exactly
$SA$ pairs therefore gives an event $\cE_{\mathsf{raw}}$ with $
\bbP(\cE_{\mathsf{raw}}^\complement)\leq SA\delta$ and $
\qquad
\Ntotal^{K+1}(s,a)\leq\frac K\delta U(s,a)$ for every $(s,a)$. 
Set $\cE_{\mathsf{count}}:=
\cE_{\mathsf{floor}}\cap\cE_{\mathsf{raw}}$.  Then $
\bbP(\cE_{\mathsf{count}}^\complement)
\leq(2S^3A+2SA)\delta.$

Work on $\cE_{\mathsf{count}}$ and fix $(s,a)$.  If the pair is never
removed from $\cO$, then
$X_k(s,a)=0$ for every $k$.  Hence
\begin{align}
\frac{N^{K+1}(s,a)}{N^{k_0(s,a)}(s,a)}
\leq\frac{504S^2K\log S}{\upsilon\delta\log\frac1\delta}.
\label{eq:inverse-count-ratio}
\end{align}

For a dyadic level $n$, consider all episodes
$k\in\{k_0(s,a),\ldots,K\}$ with
$N^k(s,a)=n$.  At the beginning of each such episode, $
n\leq \Ntotal^k(s,a)<2n.$ 
While $I_h^k=1$, the stopping rule gives
$\Nhtotal^k(s,a)<2n$.  Since the raw count is nondecreasing over
steps and episodes, the total number of stopped-good visits to $(s,a)$ over
all episodes having frozen level $n$ is therefore at most $n$.  

Let $\cD(s,a)$ be the set of frozen levels attained from episode
$k_0(s,a)$ through episode $K$.  The preceding block bound gives
\begin{align*}
\sum_{k=k_0(s,a)}^K\frac{X_k(s,a)}{N^k(s,a)}
&=\sum_{n\in\cD(s,a)}\frac1n
\sum_{\substack{k_0(s,a)\leq k\leq K:\\N^k(s,a)=n}}X_k(s,a)\leq |\cD(s,a)|.
\end{align*}
The levels are powers of two and may be skipped.  Let
$R:=\frac{504S^2K\log S}{\upsilon\delta\log\frac1\delta}$.  Since
$\delta\log\frac1\delta\leq\ee^{-1}$ and
$\frac1\upsilon\geq20S\log S$, $
R\geq10080\ee S^3K(\log S)^2\geq\ee^4.$ 
Thus \Cref{eq:inverse-count-ratio} implies
\begin{align*}
|\cD(s,a)|
&\leq1+\log_2\frac{N^{K+1}(s,a)}{N^{k_0(s,a)}(s,a)}
\leq1+\log_2R\leq1+\frac32\log R
\leq2\log R=L_N.
\end{align*}
Here the last inequality uses $\log R\geq4$.  Therefore, for every
pair,
\begin{align*}
\sum_{k=1}^K\sum_{h=1}^{H_1}
I_h^k\frac{\bbI[(s_h^k,a_h^k)=(s,a)]}{N^k(s_h^k,a_h^k)}
\leq L_N.
\end{align*}
Summing over all $(s,a)\in\SA$ gives $
   \sum_{k=1}^K \sum_{h=1}^{H_1}
    I_h^k
    \frac{1}{N^k(s_h^k,a_h^k)} \le S A L_N .$ 
For the second inequality, apply Cauchy's inequality:
\begin{align*}
   \sum_{k=1}^K \sum_{h=1}^{H_1}
    I_h^k
    \sqrt{
        \frac{w_h^k}{N^k(s_h^k,a_h^k)}
    } &= \sum_{k=1}^K \sum_{h=1}^{H_1}
    \sqrt{I_h^k w_h^k}\,
    \sqrt{
        \frac{I_h^k}{N^k(s_h^k,a_h^k)}
    }\\
    &\le \sqrt{
       \sum_{k=1}^K \sum_{h=1}^{H_1} I_h^k w_h^k
    }
    \sqrt{
       \sum_{k=1}^K \sum_{h=1}^{H_1}
        \frac{I_h^k}{N^k(s_h^k,a_h^k)}
    }\\
    &\le \sqrt{
        S A L_N
       \sum_{k=1}^K \sum_{h=1}^{H_1} I_h^k w_h^k
    } .
\end{align*}  This
proves the lemma.
\end{proof}

Recall that \Cref{alg:plann} computes
\begin{align*}
    Q_h^k(s,a) =
    \begin{cases}
    \min\left\{
        1,\,
        r(s,a)+\wh{P}^k_{s,a}V_{h+1}^k +100 b\left(\wh{P}^k_{s,a},V_{h+1}^k,N^k(s,a)\right)
    \right\},
    & (s,a)\notin O^k,\\[0.5em]
    1,
    & (s,a)\in O^k.
    \end{cases}
\end{align*}
Define
\begin{align*}
    \beta_h^k :={}&
    100\,b\left(\wh{P}_h^k,V_{h+1}^k,N_h^k\right) +
    b\left(\wh{P}_h^k,V_{h+1}^*,N_h^k\right) +
    2\sqrt{
        \frac{
        S\bbV\left(\wh{P}_h^k,V_{h+1}^k-V_{h+1}^*\right)\log \frac{1}{\delta}
        }{
        N_h^k
        }
    } +
    10\frac{S\log \frac{1}{\delta}}{N_h^k}.
\end{align*}

\begin{lemma}
\label{lem:local-bellman-section2} Assume the event $\cE_{\mathsf{1step}}$ from \Cref{lemma:one_step_conc}.  It then holds that: for every $k\in[K]$ and $h\in[H_1]$ with $I_h^k=1$,
\begin{align*}
    V_h^k(s_h^k) -
    r_h^k -
    P_h^k V_{h+1}^k \le \beta_h^k .
\end{align*}
In particular, this holds with probability at least
$1-\left(\frac{6S^2AL_N}{\epsilon^2}+4S^3A+2SA\right)\delta$.
\end{lemma}

\begin{proof}

Fix $k\in[K]$ and $h\in[H_1]$ with $I_h^k=1$.  Since $h<\tau_k$, the encountered state-action pair is learned, $(s_h^k,a_h^k)\notin\cO^k$, and $N_h^k>0$.  Moreover, \Cref{alg:new,alg:plann} give
\begin{align*}
    a_h^k=\pi_h^k(s_h^k)\in\argmax_a Q_h^k(s_h^k,a),
    \qquad
    V_h^k(s_h^k)=Q_h^k(s_h^k,a_h^k).
\end{align*}
Therefore, 
\begin{align*}
    V_h^k(s_h^k)
    \le r_h^k+
    \wh{P}_h^kV_{h+1}^k+
    100\,b\left(\wh{P}_h^k,V_{h+1}^k,N_h^k\right).
\end{align*}
Thus, writing $D_{h+1}^k:=V_{h+1}^k-V_{h+1}^*$,
\begin{align*}
    V_h^k(s_h^k)-r_h^k-P_h^kV_{h+1}^k
    \le {}&100\,b\left(\wh{P}_h^k,V_{h+1}^k,N_h^k\right) +\left(\wh{P}_h^k-P_h^k\right)V_{h+1}^*
    +\left(\wh{P}_h^k-P_h^k\right)D_{h+1}^k .
\end{align*}

The planning recursion and reward normalization give $V_{h+1}^k,V_{h+1}^*\in[0,1]^S$.  If $h<H_1$, \Cref{lemma:optimism} gives $V_{h+1}^k\ge V_{h+1}^*$ on $\cE_{\mathsf{1step}}$; if $h=H_1$, the same inequality follows from $V_{H_1+1}^k=V_{H_1+1}^*\equiv0$.  Hence $D_{h+1}^k\in[0,1]^S$.  Also $h+1\le H_1+1\le H$, so \Cref{prop:hfree_concentration_optimism} and Corollary~\ref{coro1} yield
\begin{align*}
    \left(\wh{P}_h^k-P_h^k\right)V_{h+1}^*
    &\le \left|\left(\wh{P}_h^k-P_h^k\right)V_{h+1}^*\right| \le \ud{b}\left(\wh{P}_h^k,V_{h+1}^*,N_h^k\right)
    \le b\left(\wh{P}_h^k,V_{h+1}^*,N_h^k\right).
\end{align*}
Finally, the uniform inequality \Cref{eq:any_u_coordinate_conc} may be instantiated with the data-dependent vector $u=D_{h+1}^k$ without a union bound, and gives
\begin{align*}
    \left(\wh{P}_h^k-P_h^k\right)D_{h+1}^k
    \le {}&\left|\left(\wh{P}_h^k-P_h^k\right)D_{h+1}^k\right| 
    \le 2\sqrt{
        \frac{
            S\bbV\left(\wh{P}_h^k,D_{h+1}^k\right)\log\frac1\delta
        }{
            N_h^k
        }
    }
    +10\frac{S\log\frac1\delta}{N_h^k}.
\end{align*}
Substituting the last two inequalities into the model-error decomposition, and recalling the  the definition of $\beta_h^k$, we finish the proof.
\end{proof}

\begin{lemma}
\label{lem:stopped-regret-section2}
 Define
\begin{align*}
    B:=\sum_{k=1}^K\sum_{h=1}^{H_1}\beta_h^kI_h^k,
    \qquad
    L_1:=\sum_{k=1}^K\sum_{h=1}^{H_1}
    (P_h^k-\II_{s_{h+1}^k})V_{h+1}^kI_h^k.
\end{align*}
On the same event $\cE_{\mathsf{1step}}$ as in \Cref{lemma:one_step_conc},
\begin{align*}
    \Regret_{H_1}(K)
    \le L_1+B+L_{\stoplabel}
    +\sum_{k=1}^K\left(\sum_{h=1}^{H_1}r_h^k-V_1^{\pi^k}(s_1^k)\right).
\end{align*}
Consequently, the inequality holds with probability at least
$1-\left(\frac{6S^2AL_N}{\epsilon^2}+4S^3A+2SA\right)\delta$.
\end{lemma}

\begin{proof}
Note that on $\cE_{\mathsf{1step}}$, both \Cref{lem:local-bellman-section2} and \Cref{lemma:optimism} hold.  For each episode set $I_{H_1+1}^k:=0$ and recall $V_{H_1+1}^k\equiv0$.  The indicators are nonincreasing in $h$.  Hence, for every $h\in[H_1]$, the local Bellman inequality when $I_h^k=1$ (and the trivial identity when $I_h^k=0$), together with
\begin{align*}
    P_h^kV_{h+1}^kI_h^k
    ={}&(P_h^k-\II_{s_{h+1}^k})V_{h+1}^kI_h^k
    +V_{h+1}^k(s_{h+1}^k)I_{h+1}^k +V_{h+1}^k(s_{h+1}^k)(I_h^k-I_{h+1}^k),
\end{align*}
gives
\begin{align*}
    V_h^k(s_h^k)I_h^k
    -V_{h+1}^k(s_{h+1}^k)I_{h+1}^k
    \le {}&r_h^kI_h^k
    +(P_h^k-\II_{s_{h+1}^k})V_{h+1}^kI_h^k
    +\beta_h^kI_h^k +V_{h+1}^k(s_{h+1}^k)(I_h^k-I_{h+1}^k).
\end{align*}
Summing over $h$, we have that for every $k$,
\begin{align*}
    V_1^k(s_1^k)I_1^k
    \le {}&\sum_{h=1}^{H_1}r_h^kI_h^k
    +\sum_{h=1}^{H_1}(P_h^k-\II_{s_{h+1}^k})V_{h+1}^kI_h^k
    +\sum_{h=1}^{H_1}\beta_h^kI_h^k \\
    &+\sum_{h=1}^{H_1}V_{h+1}^k(s_{h+1}^k)(I_h^k-I_{h+1}^k).
\end{align*}

By optimism, adding and subtracting the full realized reward sum in episode $k$, and then using the preceding telescoping inequality, gives
\begin{align*}
    V_1^*(s_1^k)-V_1^{\pi^k}(s_1^k)
    \le {}&\sum_{h=1}^{H_1}(P_h^k-\II_{s_{h+1}^k})V_{h+1}^kI_h^k +\sum_{h=1}^{H_1}\beta_h^kI_h^k+C_k
    +\left(\sum_{h=1}^{H_1}r_h^k-V_1^{\pi^k}(s_1^k)\right),
\end{align*}
where we use that 
   $ -\sum_{h=1}^{H_1}r_h^k(1-I_h^k)\le 0$
and define
\begin{align*}
    C_k:=V_1^k(s_1^k)(1-I_1^k)
    +\sum_{h=1}^{H_1}V_{h+1}^k(s_{h+1}^k)(I_h^k-I_{h+1}^k).
\end{align*}
These two endpoint terms require only one stopping charge.  Indeed, if $\tau_k=H_1+1$, then $C_k=0$ because the only terminal indicator drop is multiplied by $V_{H_1+1}^k=0$.  If $\tau_k=1$, then $C_k=V_1^k(s_1^k)\le1$.  If $2\le\tau_k\le H_1$, the unique nonzero drop is at $h=\tau_k-1$, and $C_k=V_{\tau_k}^k(s_{\tau_k}^k)\le1$.  Thus $
    C_k\le\bbI[\tau_k\le H_1].$ As a result,
\begin{align*}
    C_k\le\bbI[E^k]+\bbI[U^k].
\end{align*}
Summing the preceding regret inequality over $k$ gives the claim. 
\end{proof}

\subsection{Bound of Bonus Sum and Total Variance}\label{sec:p_totalvar}

\begin{lemma}[Total variance of optimistic values and gaps]
\label{lem:variance-closure-section2}
Assume $S\geq200$, that $H_3$ is a positive integer, and fix
$0<\epsilon<1$ and $\delta\in(0,1)$.  For $h\in[H_1+1]$, define the optimistic gap
$D_h^k:=V_h^k-V_h^*$, and let
\begin{align*}
    \Sigma_V&:=\sum_{k=1}^K\sum_{h=1}^{H_1}
    I_h^k\bbV(P_h^k,V_{h+1}^k),\\
    \Sigma_{V^*}&:=\sum_{k=1}^K\sum_{h=1}^{H_1}
    I_h^k\bbV(P_h^k,V_{h+1}^*),\\
    \Sigma_D&:=\sum_{k=1}^K\sum_{h=1}^{H_1}
    I_h^k\bbV(P_h^k,D_{h+1}^k).
\end{align*}
There exists an event $\cE_{\mathsf{var}}$ satisfying
$\bbP(\cE_{\mathsf{var}}^\complement)\le3\delta$ such that, on
$\cE_{\mathsf{1step}}\cap\cE_{\mathsf{var}}$, simultaneously,
\begin{align*}
    \Sigma_V&\le4K+4B+4L_{\stoplabel}+64\log\frac1\delta,\\
    \Sigma_{V^*}&\le8K+48\log\frac1\delta,\\
    \Sigma_D&\le4B+4L_{\stoplabel}+64\log\frac1\delta.
\end{align*}
Here $B$ and $L_{\stoplabel}$ are the previously defined pathwise
quantities.
\end{lemma}

\begin{proof}
Recall $r_h^k=r(s_h^k,a_h^k)$ and use the terminal conventions
$I_{H_1+1}^k=0$, $V_{H_1+1}^k=V_{H_1+1}^*=0$, and
$D_{H_1+1}^k=0$.

We first record the filtration used below.  Flatten the transitions in
chronological $(k,h)$ order, and let $\cF_{k,h}^-$ be the
sigma-field immediately before $s_{h+1}^k$ is drawn.  At that time the
state and action, the transition kernel $P_h^k$, and the episode-start functions
$V_{h+1}^k$, $V_{h+1}^*$, and $D_{h+1}^k$ are measurable.  Moreover,
$I_h^k=\bbI[h<\tau_k]$ is measurable because whether the stopping rule
has fired by step $h$ depends only on the history available before this
transition.  Thus $I_h^k$ is predictable.  On $\{I_h^k=1\}$,
conditionally on $\cF_{k,h}^-$ the successor has law $P_h^k$;
on $\{I_h^k=0\}$ every increment below is zero.  Inserting the
algorithmic updates between consecutive transitions into the flattened
filtration makes each centered sum below a martingale. Therefore, the time-uniform
form of the concentration inequality in \Cref{lem:linear_freedman} may  be applied at the
terminal flattened time, without using union bound over episodes, layers, or
stopping times.

\paragraph{Bound of $\Sigma_V$.}
Work on $\cE_{\mathsf{1step}}$.  Fix $(h,k)$ such that $I_h^k=1$.  By
\Cref{lem:local-bellman-section2},
$V_h^k(s_h^k)-P_h^kV_{h+1}^k\le r_h^k+\beta_h^k$.  Since
$V_h^k(s_h^k),P_h^kV_{h+1}^k\in[0,1]$ and $r_h^k,\beta_h^k\ge0$, this also covers
the case $V_h^k(s_h^k)-P_h^kV_{h+1}^k<0$ and gives
\begin{align*}
    (V_h^k(s_h^k))^2-(P_h^kV_{h+1}^k)^2\le2(r_h^k+\beta_h^k).
\end{align*}
Decompose
\begin{align*}
    \Sigma_V={}&F_V+
    \sum_{k=1}^K\sum_{h=1}^{H_1}I_h^k
    \left[(V_{h+1}^k(s_{h+1}^k))^2
    -(P_h^kV_{h+1}^k)^2\right],\\
    F_V:={}&\sum_{k=1}^K\sum_{h=1}^{H_1}I_h^k
    \left[P_h^k(V_{h+1}^k)^2
    -(V_{h+1}^k(s_{h+1}^k))^2\right].
\end{align*}
The increments of $F_V$ are centered at zero and have absolute value at most
one.  Their predictable quadratic variation is
\begin{align*}
    \sum_{k=1}^K \sum_{h=1}^{H_1}I_h^k\bbV(P_h^k,(V_{h+1}^k)^2)
    \le4\sum_{k=1}^K\sum_{h=1}^{H_1}I_h^k\bbV(P_h^k,V_{h+1}^k)
    =4\Sigma_V,
\end{align*}
where the inequality is by
\Cref{lem:var_xy}.  Applying \Cref{lem:linear_freedman} with $c=1$ and
$\eta=\frac{1}{8}$ defines an event $\cE_V$ with
$\bbP(\cE_V^\complement)\le\delta$ on which
\begin{align*}
    F_V\le\frac12\Sigma_V+8\log\frac{1}{\delta}.
\end{align*}

Adding and subtracting $(V_h^k(s_h^k))^2$ in the second sum gives a
term bounded by $2\sum_{k,h}I_h^k(r_h^k+\beta_h^k)\le2K+2B$ and the
stopped telescoping term
\begin{align}
    &\sum_{h=1}^{H_1}I_h^k
    \left[(V_{h+1}^k(s_{h+1}^k))^2-(V_h^k(s_h^k))^2\right]=-I_1^k(V_1^k(s_1^k))^2
    +\sum_{h=1}^{H_1}(I_h^k-I_{h+1}^k)
    (V_{h+1}^k(s_{h+1}^k))^2.\label{eq:variance-closure-stop-v}
\end{align}
If $\tau_k=H_1+1$, the right-hand side is
$-(V_1^k(s_1^k))^2\le0$; if $\tau_k=1$, it is zero; and if
$2\le\tau_k\le H_1$, it equals
$-(V_1^k(s_1^k))^2+(V_{\tau_k}^k(s_{\tau_k}^k))^2\le1$.
Thus there is at most one stopping charge in episode $k$.  Since a
finite stop implies $E^k\cup U^k$, summing \Cref{eq:variance-closure-stop-v}
over episodes is at most $L_{\stoplabel}$.  On
$\cE_{\mathsf{1step}}\cap\cE_V$,
\begin{align*}
    \Sigma_V
    &\le\frac12\Sigma_V+8\log\frac{1}{\delta}
    +2K+2B+L_{\stoplabel},\\
    \Sigma_V
    &\le4K+4B+2L_{\stoplabel}+16\log\frac{1}{\delta}
    \le4K+4B+4L_{\stoplabel}+64\log\frac{1}{\delta}.
\end{align*}

\paragraph{Bound of $\Sigma_{V^*}$.} Fix $(h,k)$ such that $I_h^k = 1 $. Define $\phi(u) = 2u-u^2$.
The optimal Bellman equation for the action actually gives
$V_h^*(s_h^k)\ge r_h^k+P_h^kV_{h+1}^*\ge P_h^kV_{h+1}^*$.  Hence
\begin{align*}
    (V_{h+1}^*(s_{h+1}^k))^2-(P_h^kV_{h+1}^*)^2
    &=((V_h^*(s_h^k))^2-(P_h^kV_{h+1}^*)^2)+((y:=V_{h+1}^*(s_{h+1}^k))^2-(V_h^*(s_h^k))^2)\\
    &\le2(V_h^*(s_h^k)-P_h^kV_{h+1}^*)+(V_{h+1}^*(s_{h+1}^k))^2-(V_h^*(s_h^k))^2\nonumber
    \\ &=2(V_{h+1}^*(s_{h+1}^k)-P_h^kV_{h+1}^*)+\phi(V_h^*(s_h^k))-\phi(V_{h+1}^*(s_{h+1}^k)).
\end{align*}
Consequently,
\begin{align*}
    \Sigma_{V^*}\le F_{V^*}
    +\sum_{k=1}^K\sum_{h=1}^{H_1}I_h^k
    \left[\phi(V_h^*(s_h^k))
    -\phi(V_{h+1}^*(s_{h+1}^k))\right],
\end{align*}
where
\begin{align*}
    F_{V^*}:={}&\sum_{k=1}^K\sum_{h=1}^{H_1}I_h^k
    \left[\phi(V_{h+1}^*(s_{h+1}^k))
    -P_h^k\phi(V_{h+1}^*)\right].
\end{align*}

Since $\phi([0,1])\subseteq[0,1]$, each increment of $F_{V^*}$ has
absolute value at most one.  If $s,s'$ are independent with law
$P_h^k$, then $\phi$ is $2$-Lipschitz on $[0,1]$, and therefore
\begin{align*}
    \bbV(\phi(V_{h+1}^*(s)))
    &=\frac12\bbE\left[
    (\phi(V_{h+1}^*(s))-\phi(V_{h+1}^*(s')))^2\right]\le4\bbV(V_{h+1}^*(s)).
\end{align*}
Thus the predictable quadratic variation of $F_{V^*}$ is at most
$4\Sigma_{V^*}$.  A linear Freedman application (\Cref{lem:linear_freedman}) with $c=1$ and
$\eta=\frac{1}{8}$ defines an event $\cE_{V^*}$ with
$\bbP(\cE_{V^*}^\complement)\le\delta$ on which
\begin{align*}
    F_{V^*}\le\frac12\Sigma_{V^*}+8\log\frac{1}{\delta}.
\end{align*}

The remaining potential term telescopes as
\begin{align*}
    &\sum_{h=1}^{H_1}I_h^k
    \left[\phi(V_h^*(s_h^k))
    -\phi(V_{h+1}^*(s_{h+1}^k))\right]=I_1^k\phi(V_1^*(s_1^k))
    -\sum_{h=1}^{H_1}(I_h^k-I_{h+1}^k)
    \phi(V_{h+1}^*(s_{h+1}^k)).
\end{align*}
It equals $\phi(V_1^*(s_1^k))\le1$ when $\tau_k=H_1+1$,
zero when $\tau_k=1$, and
$\phi(V_1^*(s_1^k))-\phi(V_{\tau_k}^*(s_{\tau_k}^k))\le1$
when $2\le\tau_k\le H_1$. Therefore, on $\cE_{V^*}$, we have that 
\begin{align*}
    \Sigma_{V^*}
    &\le\frac12\Sigma_{V^*}+K+8\log\frac{1}{\delta},\\
    \Sigma_{V^*}
    &\le2K+16\log\frac{1}{\delta}
    \le8K+48\log\frac{1}{\delta}.
\end{align*}

\paragraph{Bound of $\Sigma_D$.}
Return to $\cE_{\mathsf{1step}}$.  Optimism gives
$D_h^k\in[0,1]^S$ for $h\in[H_1]$, while the terminal convention gives
$D_{H_1+1}^k=0$.  Fix $(h,k)$ such that  $I_h^k=1$. Subtracting the optimal
Bellman lower bound
$V_h^*(s_h^k)\ge r_h^k+P_h^kV_{h+1}^*$ from the accepted local
Bellman upper bound gives
\begin{align*}
    D_h^k(s_h^k)-P_h^kD_{h+1}^k\le\beta_h^k.
\end{align*}
Thus, also when the left-hand drift is negative,
\begin{align*}
    (D_h^k(s_h^k))^2-(P_h^kD_{h+1}^k)^2\le2\beta_h^k.
\end{align*}
Decompose
\begin{align*}
    \Sigma_D={}&F_D+
    \sum_{k=1}^K\sum_{h=1}^{H_1}I_h^k
    \left[(D_{h+1}^k(s_{h+1}^k))^2
    -(P_h^kD_{h+1}^k)^2\right],\\
    F_D:={}&\sum_{k=1}^K\sum_{h=1}^{H_1}I_h^k
    \left[P_h^k(D_{h+1}^k)^2
    -(D_{h+1}^k(s_{h+1}^k))^2\right].
\end{align*}
The increments of $F_D$ are centered and bounded in absolute value by
one.  Since $D_{h+1}^k\in[0,1]^S$,
\Cref{lem:var_xy} gives the pathwise predictable-variation bound
\begin{align*}
    \sum_{k,h}I_h^k\bbV(P_h^k,(D_{h+1}^k)^2)
    \le4\Sigma_D.
\end{align*}
Applying \Cref{lem:linear_freedman} again with $c=1$
and $\eta=\frac{1}{8}$,  there exists an event $\cE_D$ with
$\bbP(\cE_D^\complement)\le\delta$ on which
\begin{align*}
    F_D\le\frac12\Sigma_D+8\log\frac{1}{\delta}.
\end{align*}

After adding and subtracting $(D_h^k(s_h^k))^2$, the square-drift term
is at most $2B$.  The remaining stopped telescoping term has the same
form as \Cref{eq:variance-closure-stop-v}: it is
$-(D_1^k(s_1^k))^2\le0$ if $\tau_k=H_1+1$, zero if $\tau_k=1$, and
$-(D_1^k(s_1^k))^2+(D_{\tau_k}^k(s_{\tau_k}^k))^2\le1$ for an interior
stop.  Hence it contributes at most one stopping charge per episode and
at most $L_{\stoplabel}$ in total.  On
$\cE_{\mathsf{1step}}\cap\cE_D$,
\begin{align*}
    \Sigma_D
    &\le\frac12\Sigma_D+8\log\frac{1}{\delta}
    +2B+L_{\stoplabel},\\
    \Sigma_D
    &\le4B+2L_{\stoplabel}+16\log\frac{1}{\delta}
    \le4B+4L_{\stoplabel}+64\log\frac{1}{\delta}.
\end{align*}

Finally, let
$\cE_{\mathsf{var}}:=\cE_V\cap\cE_{V^*}\cap
\cE_D$.  The three applications above are the only new random
events, so a union bound gives
$\bbP(\cE_{\mathsf{var}}^\complement)\le3\delta$.  On its intersection
with $\cE_{\mathsf{1step}}$, all three bounds hold
simultaneously. 
\end{proof}

\begin{lemma}\label{lemma:emp_to_exp}
 Let
$\cE_{\mathsf{count}}$ be the event in
\Cref{lem:inverse-count-section2}.  There exists an event
$\cE_{\mathsf{emp}}$ satisfying $
\bbP(\cE_{\mathsf{emp}}^\complement)\leq2S^2AL_N\delta$ 
such that, on $\cE_{\mathsf{emp}}\cap\cE_{\mathsf{count}}$, the
following holds simultaneously for every $k\in [K]$, every learned state-action pair $(s,a)\notin\cO^k$, and every $v\in [0,1]^S$:
\begin{align*}
\bbV\left( \wh{P}^k_{s,a}, v\right) \leq 2 \bbV(P_{s,a} ,v)  + \frac{2S\log \frac{1}{\delta}}{N^k(s,a)}.\notag
\end{align*}
Moreover, the failure probability is at most $\left(2S^2AL_N+4S^3A+2SA\right)\delta$.
\end{lemma}

\begin{proof}
For each $(s,a)$, let
\begin{align*}
    \cN(s,a):=
    \left\{2^i:i\in\bbN_0,
    \frac{\upsilon\log\frac1\delta}{504S^2\log S}\cdot U(s,a)\le2^i\le
    \frac{KU(s,a)}{\delta}\right\}\cup\{1\}.
\end{align*}
This is the deterministic possible-count set used in the proof of
\Cref{lemma:one_step_conc}.  On $\cE_{\mathsf{count}}$, it contains every possible frozen count $N^k(s,a)$ and satisfies $|\cN(s,a)|\le L_N$.

Couple the successors observed from each  state-action pair with an infinite i.i.d. sequence $Y_1(s,a),Y_2(s,a),\ldots\sim P_{s,a}$, where the $i$-th visit to $(s,a)$ uses $Y_i(s,a)$.  For $n\ge1$, let $\widehat P_{s,a}^{(n)}$ be the empirical distribution of the first $n$ entries.  Thus, for every learned state-action pair $
    \wh P_{s,a}^k=\widehat P_{s,a}^{(N^k(s,a))}.$

Fix $(s,a)$, $s'$, and $n\in\cN(s,a)$.  Applying both one-sided inequalities in \Cref{lemma:bennet} to the Bernoulli variables $\bbI \{Y_i(s,a)=s'\}$ gives, except on an event of probability at most $2\delta$,
\begin{align*}
\abs{\widehat P_{s,a,s'}^{(n)}-P_{s,a,s'}}
&\le \sqrt{\frac{2P_{s,a,s'}(1-P_{s,a,s'})\log \frac{1}{\delta}}{n}}
 +\frac{\log \frac{1}{\delta}}{3n}\le \sqrt{\frac{2P_{s,a,s'}\log \frac{1}{\delta}}{n}}
 +\frac{\log \frac{1}{\delta}}{n}.
\end{align*}
A union bound over all $(s,a,s')$ triples, and the at most $L_N$ values in each deterministic set $\cN(s,a)$ gives an event $\cE_{\mathsf{emp}}$ satisfying
$ \bbP(\cE_{\mathsf{emp}}^\complement)\leq2S^2AL_N\delta.$ 

Assume $\cE_{\mathsf{emp}}\cap\cE_{\mathsf{count}}$.  For every $k$, every $(s,a)\notin\cO^k$, and every $s'$, Young's inequality $\sqrt{2xy}\le \frac{x}{2}+y$ gives
\begin{align*}
    \wh{P}_{s,a,s'}^k
    &\le P_{s,a,s'}+
    \sqrt{\frac{2P_{s,a,s'}\log \frac{1}{\delta}}{N^k(s,a)}}+
    \frac{\log \frac{1}{\delta}}{N^k(s,a)}\le \frac32P_{s,a,s'}+
    \frac{2\log \frac{1}{\delta}}{N^k(s,a)}
    \le 2P_{s,a,s'}+
    \frac{2\log \frac{1}{\delta}}{N^k(s,a)}.
\end{align*}
This event is coordinatewise and therefore independent of the choice of $v$.

Now fix arbitrary $k\in[K]$, $(s,a)\notin\cO^k$, and $v\in[0,1]^S$.  We have that
\begin{align*}
\bbV\left(\wh{P}_{s,a}^k, v\right)
&= \min_{c\in \bbR} \sum_{s'\in \cS} \wh{P}_{s,a,s'}^k (v(s') - c)^2 \\
&\le \sum_{s'\in \cS} \wh{P}_{s,a,s'}^k (v(s') - P_{s,a}v)^2 \\
&\le \sum_{s'\in\cS} \left(
        2P_{s,a,s'} +
        \frac{2\log \frac{1}{\delta}}{N^k(s,a)}
    \right)
    (v(s')-P_{s,a} v)^2\\
&= 2\bbV(P_{s,a},v) +
    \frac{2\log \frac{1}{\delta}}{N^k(s,a)}
    \sum_{s'\in\cS}(v(s')-P_{s,a} v)^2\\
&\le 2\bbV(P_{s,a},v) +
    \frac{2S\log \frac{1}{\delta}}{N^k(s,a)},
\end{align*}
because $v(s')$ and $P_{s,a}v$ both lie in $[0,1]$.
Combining the failure probability for $\cE_{\mathsf{emp}}$ and
$\cE_{\mathsf{count}}$ gives the final probability statement.
\end{proof}

The next lemma sums the Bellman-error envelope.
It uses the inverse-count bound (\Cref{lem:inverse-count-section2}), the true-to-empirical variance comparison, the total-variance closure, and the clipped-variance bound for the inhomogeneous optimal value functions (Corollary~\ref{cor:hp_eps_var_from_total_deviation}).

\begin{lemma}[Summing the Bellman errors]
\label{lem:sum-bellman-section2} Recall that $L_N = 2\log \left(\frac{504S^2K\log S}{\upsilon\delta\log\frac1\delta}\right)$ (\Cref{eq:def_L_N}).
With probability at least $
1-\left(
\frac{6S^2AL_N}{\epsilon^2}
+3S^2AL_N+4S^3A+2SA+4
\right)\delta$, 
it holds that
\begin{align}
    B \le 14000
    \bigg(
    &
    \sqrt{
        SA L_{N}\log\left(\frac{1}{\delta}\right) \left(
            K+L_{\stoplabel}+16\log \frac{1}{\delta}
        \right)
    } \notag
    \\ & \qquad \qquad+
    \sqrt{
        S^2A L_N\log\left(\frac{1}{\delta}\right) \left(
            L_{\stoplabel} +
            \epsilon S\left(K+\log \frac{1}{\delta}\right)
        \right)
    }
    +3200S^2AL_N\log \frac{1}{\delta}
    \bigg).\label{eq:bound_bonus}
\end{align}
\end{lemma}

\begin{proof}
Write $D_{h+1}^k:=V_{h+1}^k-V_{h+1}^*$.
Expanding the fixed definitions of $\beta_h^k$ and
\begin{align*}
    b(p,v,n)
    =3\sqrt{\frac{\bbV(p,v)\log\frac{1}{\delta}}{n}}
    +5\sqrt{\frac{S\bbV_{5\epsilon}(p,v)\log\frac{1}{\delta}}{n}}
    +20\frac{S\log\frac{1}{\delta}}{n}
\end{align*}
gives the exact identity
\begin{align}
    B=300T_1+3T_2+2T_3+500T_4+5T_5+2030T_6,\label{eq:sum-bellman-decomposition}
\end{align}
where
\begin{align*}
T_1&:=\sum_{k=1}^K\sum_{h=1}^{H_1}I_h^k
\sqrt{\frac{\bbV(\wh P_h^k,V_{h+1}^k)\log\frac{1}{\delta}}{N_h^k}},\\
T_2&:=\sum_{k=1}^K\sum_{h=1}^{H_1}I_h^k
\sqrt{\frac{\bbV(\wh P_h^k,V_{h+1}^*)\log\frac{1}{\delta}}{N_h^k}},\\
T_3&:=\sum_{k=1}^K\sum_{h=1}^{H_1}I_h^k
\sqrt{\frac{S\bbV(\wh P_h^k,D_{h+1}^k)\log\frac{1}{\delta}}{N_h^k}},\\
T_4&:=\sum_{k=1}^K\sum_{h=1}^{H_1}I_h^k
\sqrt{\frac{S\bbV_{5\epsilon}(\wh P_h^k,V_{h+1}^k)
\log\frac{1}{\delta}}{N_h^k}},\\
T_5&:=\sum_{k=1}^K\sum_{h=1}^{H_1}I_h^k
\sqrt{\frac{S\bbV_{5\epsilon}(\wh P_h^k,V_{h+1}^*)
\log\frac{1}{\delta}}{N_h^k}},\\
T_6&:=\sum_{k=1}^K\sum_{h=1}^{H_1}I_h^k
\frac{S\log\frac{1}{\delta}}{N_h^k}.
\end{align*}
Indeed, the two copies of $b$ contribute $2000T_6+20T_6$, and the
last term in $\beta_h^k$ contributes the remaining $10T_6$.

Let $\cE_{\mathsf{emp}}$ and $\cE_{\mathsf{bdev}}$ be the base
concentration events constructed in the proofs of
\Cref{lemma:emp_to_exp,lemma:bdev}, let $\cE_{\mathsf{one}}$ be the base
concentration event in \Cref{lemma:one_step_conc}, and let $\cE_{\dev}$
be the event in \Cref{cor:hp_eps_var_from_total_deviation}.  Intersect
these with the shared event $\cE_{\mathsf{count}}$ from
\Cref{lem:inverse-count-section2} and $\cE_{\mathsf{var}}$
from \Cref{lem:variance-closure-section2}, and call the resulting event
$\cE$.  The six events are each charged only once, so
\begin{align}
\bbP(\cE^\complement)
&\le\left(
\frac{6S^2AL_N}{\epsilon^2}
+(2S^3A+2SA)+2S^2AL_N+3+S^2AL_N+1
\right)\delta\notag\\
&=\left(
\frac{6S^2AL_N}{\epsilon^2}
+3S^2AL_N+2S^3A+2SA+4
\right)\delta.\label{eq:sum-bellman-event-ledger}
\end{align}
Assume $\cE$ for the rest of the proof.



Set
\begin{align*}
q&:=S^2AL_N\log\frac{1}{\delta},\\
\mathrm{Term}_1&:=\sqrt{SAL_N\log\frac{1}{\delta}
\left(K+B+L_{\stoplabel}+16\log\frac{1}{\delta}\right)},\\
\mathrm{Term}_2&:=\sqrt{q\left(B+L_{\stoplabel}
+\epsilon S\left(K+\log\frac{1}{\delta}\right)\right)}.
\end{align*}
The first inverse-count inequality (\Cref{lem:inverse-count-section2}) gives
\begin{align}
T_6\le q.\label{eq:t6}
\end{align}

We next bound the three ordinary-variance terms.  The event
$\cE_{\mathsf{emp}}$, followed by the three simultaneous bounds
on $\cE_{\mathsf{1step}}\cap\cE_{\mathsf{var}}$, gives
\begin{align*}
\sum_{k,h\le H_1}I_h^k\bbV(\wh P_h^k,V_{h+1}^k)
&\le2\Sigma_V+2q\le8K+8B+8L_{\stoplabel}+128\log\frac{1}{\delta}+2q.
\end{align*}
Therefore, using the weighted inverse-count inequality in \Cref{lem:inverse-count-section2} and
$SAL_N\log\frac{1}{\delta}=\frac{q}{S}$,
\begin{align}
T_1
&\le\sqrt{SAL_N\log\frac{1}{\delta}
\left(8K+8B+8L_{\stoplabel}+128\log\frac{1}{\delta}+2q\right)}
\notag\\
&\le\sqrt8\,\mathrm{Term}_1+\sqrt{\frac2S}\,q
\le3\mathrm{Term}_1+q.\label{eq:sum-bellman-t1}
\end{align}
Likewise, we have that
\begin{align*}
\sum_{k,h\le H_1}I_h^k\bbV(\wh P_h^k,V_{h+1}^*)
&\le2\Sigma_{V^*}+2q\le16K+96\log\frac{1}{\delta}+2q.
\end{align*}
Consequently,
\begin{align}
T_2
&\le\sqrt{SAL_N\log\frac{1}{\delta}
\left(16K+96\log\frac{1}{\delta}+2q\right)}\le4\mathrm{Term}_1+\sqrt{\frac2S}\,q
\le4\mathrm{Term}_1+q.\label{eq:sum-bellman-t2}
\end{align}
For the gap, the use of \Cref{lemma:emp_to_exp} is legitimate on the
intersection with $\cE_{\mathsf{1step}}$, where
$D_{h+1}^k\in[0,1]^S$.  Hence
\begin{align*}
\sum_{k=1}^K \sum_{h=1}^{H_1}I_h^k\bbV(\wh P_h^k,D_{h+1}^k)
&\le2\Sigma_D+2q\le8B+8L_{\stoplabel}+128\log\frac{1}{\delta}+2q.
\end{align*}
Since $S\geq200$ and
$A\geq 8$, it holds that $128\log\frac{1}{\delta}\le q$.  By the weighted inverse-count inequality in \Cref{lem:inverse-count-section2},
\begin{align}
T_3
&\le\sqrt{q\left(8B+8L_{\stoplabel}
+128\log\frac{1}{\delta}+2q\right)}\le3\mathrm{Term}_2+2q.\label{eq:sum-bellman-t3}
\end{align}

We turn to the clipped optimal-value term.   By
\Cref{lemma:bdev} at clipping level $5\epsilon$ and
\Cref{cor:hp_eps_var_from_total_deviation},
\begin{align*}
&\sum_{k=1}^K\sum_{h=1}^{H_1}I_h^k
\bbV_{5\epsilon}(\wh P_h^k,V_{h+1}^*)\notag\\
&\quad\le
4\sum_{k=1}^K\sum_{h=1}^{H_1}I_h^k
\bbV_{5\epsilon}(P_h^k,V_{h+1}^*)
+50\epsilon^2T_6
+20\epsilon\sum_{k=1}^K\sum_{h=1}^{H_1}I_h^k
\abs{(\wh P_h^k-P_h^k)V_{h+1}^*}\notag\\
&\quad\le60\epsilon SK+80\epsilon S\log\frac{1}{\delta}+50\epsilon^2q
+20\epsilon\left(3T_2+5T_5+20T_6\right).
\end{align*}
Here the last term $\abs{(\wh P_h^k-P_h^k)V_{h+1}^*}$ is bounded on
$\cE_{\mathsf{1step}}$ by
$\underline b(\wh P_h^k,V_{h+1}^*,N_h^k)
\le b(\wh P_h^k,V_{h+1}^*,N_h^k)$, whose stopped sum expands exactly as
$3T_2+5T_5+20T_6$.
Since all terms in \Cref{eq:sum-bellman-decomposition} are nonnegative,
$3T_2+5T_5\le B$.  Using also $T_6\le q$ and
$\epsilon=\frac{1}{S^2}\le\frac{1}{40000}$, the last inequality gives
\begin{align}
\sum_{k,h\le H_1}I_h^k
\bbV_{5\epsilon}(\wh P_h^k,V_{h+1}^*)
&\le20\epsilon B
+80\epsilon S\left(K+\log\frac{1}{\delta}\right)
+450\epsilon q\notag\\
&\le80\left(B+L_{\stoplabel}
+\epsilon S\left(K+\log\frac{1}{\delta}\right)\right)
+450\epsilon q.\label{eq:gap_v}
\end{align}
Applying \Cref{lem:inverse-count-section2} and the fact that  $\sqrt{450\epsilon}\le1$, we have
\begin{align}
T_5
&\le\sqrt{q\sum_{k,h\le H_1}I_h^k
\bbV_{5\epsilon}(\wh P_h^k,V_{h+1}^*)}\le9\mathrm{Term}_2+q.\label{eq:bound_t5}
\end{align}

For the clipped optimistic-value term, retain both factors of two in
\begin{align*}
\bbV_{5\epsilon}(p,u+v)
\le2\bbV_{5\epsilon}(p,u)+2\bbV(p,v).
\end{align*}
With $u=V_{h+1}^*$ and $v=D_{h+1}^k$, summing only over
$h\in[H_1]$ and using \Cref{eq:gap_v} and the empirical gap-variance
bound above gives
\begin{align*}
&\sum_{k=1}^K\sum_{h=1}^{H_1}I_h^k
\bbV_{5\epsilon}(\wh P_h^k,V_{h+1}^k)\notag\\
&\quad\le
2\left(20\epsilon B
+80\epsilon S\left(K+\log\frac{1}{\delta}\right)+450\epsilon q\right)
+2\left(8B+8L_{\stoplabel}+128\log\frac{1}{\delta}+2q\right)
\notag\\
&\quad=(16+40\epsilon)B+16L_{\stoplabel}
+160\epsilon S\left(K+\log\frac{1}{\delta}\right)
+256\log\frac{1}{\delta}+(4+900\epsilon)q\notag\\
&\quad\le160\left(B+L_{\stoplabel}
+\epsilon S\left(K+\log\frac{1}{\delta}\right)\right)+6q.
\end{align*}
For the last inequality, we use $16+40\epsilon\le17$,
$4+900\epsilon\le5$, and $256\log\frac{1}{\delta}\le q$, which follow from
$S\ge200$, $\epsilon=\frac{1}{S^2}$, and $L_N\ge8$.  A weighted
inverse-count application from \Cref{lem:inverse-count-section2} gives
\begin{align}
T_4
&\le\sqrt{q\sum_{k,h\le H_1}I_h^k
\bbV_{5\epsilon}(\wh P_h^k,V_{h+1}^k)}\le13\mathrm{Term}_2+3q.\label{eq:bound_t4}
\end{align}

Substituting \Cref{eq:t6,eq:sum-bellman-t1,eq:sum-bellman-t2,eq:sum-bellman-t3,eq:bound_t5,eq:bound_t4} into
\Cref{eq:sum-bellman-decomposition} gives
\begin{align}
B
&\le912\mathrm{Term}_1+6551\mathrm{Term}_2+3842q.\label{eq:sum-bellman-self-bound}
\end{align}

It remains to remove \(B\) from both square roots. We have
\begin{align*}
\mathrm{Term}_1
&=\sqrt{\frac{q}{S}\left(
B+K+L_{\stoplabel}+16\log\frac{1}{\delta}
\right)},\\
\mathrm{Term}_2
&=\sqrt{q\left(
B+L_{\stoplabel}
+\epsilon S\left(K+\log\frac{1}{\delta}\right)
\right)}.
\end{align*}
Since \(\sqrt{B+u}\le \sqrt{B}+\sqrt{u}\) for every \(u\ge 0\),
\Cref{eq:sum-bellman-self-bound} implies
\begin{align*}
B
&\le
\left(912\sqrt{\frac{q}{S}}+6551\sqrt{q}\right)\sqrt{B}\\
&\quad
+912\sqrt{\frac{q}{S}\left(
K+L_{\stoplabel}+16\log\frac{1}{\delta}
\right)}\\
&\quad
+6551\sqrt{q\left(
L_{\stoplabel}
+\epsilon S\left(K+\log\frac{1}{\delta}\right)
\right)}
+3842q.
\end{align*}
By Young's inequality, we have that 
\begin{align*}
\left(912\sqrt{\frac{q}{S}}+6551\sqrt{q}\right)\sqrt{B}
&\le \frac{B}{2}
+\frac{1}{2}
\left(912\sqrt{\frac{q}{S}}+6551\sqrt{q}\right)^2.
\end{align*}
Rearranging and using \(S\ge 200\), we obtain
\begin{align*}
B
&\le
1824\sqrt{\frac{q}{S}\left(
K+L_{\stoplabel}+16\log\frac{1}{\delta}
\right)}\\
&\quad
+13102\sqrt{q\left(
L_{\stoplabel}
+\epsilon S\left(K+\log\frac{1}{\delta}\right)
\right)}
+\left(6616^2+7684\right)q\\
&\le
14000\Bigg(
\sqrt{\frac{q}{S}\left(
K+L_{\stoplabel}+16\log\frac{1}{\delta}
\right)}\\
&\hspace{3.7cm}
+\sqrt{q\left(
L_{\stoplabel}
+\epsilon S\left(K+\log\frac{1}{\delta}\right)
\right)}
+3200q
\Bigg).
\end{align*}

Finally, substituting the definition of \(q\) proves
\Cref{eq:bound_bonus}.
The probability is exactly the one recorded in
\Cref{eq:sum-bellman-event-ledger}.
\end{proof}

\subsection{Bound of Exploration Error}\label{sec:b_exp}

\begin{restatable}{lemma}{restatableBdEk}\label{lemma:bdEk}
 With probability at least $1-4S^3A^2\delta$, it holds that
\begin{align*}
\sum_{k=1}^K\bbI[E^k]
&\leq SA\left\lceil N_{\known}\right\rceil +216000S\bigg(
S^2A\left\lceil N_{\tref}\right\rceil
+SA(S^2A+1)
\left(58320S^2A N_{\tref}+1+\log\frac1\delta\right)
\bigg)\\ 
&\le O \left(S^8A^3\log\frac1\delta \right).
\end{align*}
\end{restatable}

The proof of \Cref{lemma:bdEk} is deferred to \Cref{sec:exps4}.

\begin{lemma}\label{lemma:bd_Lstop}
 With probability at least $1-4S^3A^2\delta$,
it holds that
\begin{align}
L_{\stoplabel}
&\leq SA\left\lceil N_{\known}\right\rceil
+216000S\bigg(
S^2A\left\lceil N_{\tref}\right\rceil
\notag\\
&\quad
+SA(S^2A+1)
\left(58320S^2A N_{\tref}+1+\log\frac1\delta\right)
\bigg) +SA\log_2\frac{504S^2K\log S}{\upsilon\delta\log\frac1\delta}.
\label{eq:bd-Lstop-explicit} \\
&\le O \left(S^8A^3\log\frac1\delta +SA\log\frac{S K}{\upsilon\delta\log\frac1\delta}\right). \notag
\end{align}
\end{lemma}

\begin{proof}[Proof of \Cref{lemma:bd_Lstop}]
By definition,
$L_{\stoplabel}=\sum_{k=1}^K\bbI[E^k]+\sum_{k=1}^K\bbI[U^k]$.
Let $\cE_{\mathsf{count}}$ be the count event in
\Cref{lem:inverse-count-section2} and assume this event.  For each
episode,
\begin{align*}
\bbI[U^k]
\leq\sum_{(s,a)\notin\cO^k}
\bbI\left[N^{k+1}(s,a)\geq2N^k(s,a)\right].
\end{align*}
Recall that  $k_0(s,a)$ denotes the first $k$ such that $(s,a)\notin \mathcal{O}^k$. As a result,
\begin{align*}
\sum_{k=k_0(s,a)}^K
\bbI\left[N^{k+1}(s,a)\geq2N^k(s,a)\right]
&\leq\sum_{k=k_0(s,a)}^K
\log_2\frac{N^{k+1}(s,a)}{N^k(s,a)}=\log_2\frac{N^{K+1}(s,a)}{N^{k_0(s,a)}(s,a)}.
\end{align*}
  Summing
over pairs and using \Cref{eq:inverse-count-ratio} gives
\begin{align}
\sum_{k=1}^K\bbI[U^k]
\leq SA\log_2\frac{504S^2K\log S}{\upsilon\delta\log\frac1\delta}.
\label{eq:bd-Lstop-doubling}
\end{align}

For the exploration term, \Cref{lemma:bdEk} gives exactly the first two lines
on the right-hand side of \Cref{eq:bd-Lstop-explicit}.  The probability
events can be combined without paying twice for their common reference-model
event.  
Therefore the common intersection has failure probability at most
\begin{align*}
\left(2S^3A+2SA+SA(S^2A+1)\right)\delta
\leq4S^3A^2\delta.
\end{align*}
On this intersection, adding \Cref{eq:bd-Lstop-doubling} to the bound of
\Cref{lemma:bdEk} proves \Cref{eq:bd-Lstop-explicit}.
 The proof is finished.
\end{proof}

\subsection{Putting All Together}\label{sec:putall}
\begin{proof}[Proof of \Cref{thm:reg}]
The prescribed internal confidence satisfies $\delta\in(0,1)$ and $
    \log\frac1\delta
    =\log\frac{210S^7A^2K}{\delta_0^2}$.

If $K<1000S^2A\log\frac1\delta$, then the total-reward normalization gives
$\Regret_H(K)\leq K$, and hence
\begin{align*}
    \Regret_H(K)
    \leq7000S^2A\log\frac{SAK}{\delta_0}
    =O\left(S^8A^3\log^2\frac{SAK}{\delta_0}\right).
\end{align*}
Thus the theorem holds deterministically in this case.  For the remainder of
the proof assume $K\geq1000S^2A\log\frac1\delta$, which is the standing
regime of the lower-level lemmas.

Let $G=G(\delta)$ denote the deterministic right-hand side in
\Cref{eq:bd-Lstop-explicit}, namely
\begin{align*}
G:={}&SA\left\lceil N_{\known}\right\rceil
+216000S\bigg(
S^2A\left\lceil N_{\tref}\right\rceil
+SA(S^2A+1)
\left(58320S^2A N_{\tref}+1+\log\frac1\delta\right)
\bigg)\\
&\qquad \qquad +SA\log_2\frac{504S^2K\log S}{\upsilon\delta\log\frac1\delta}.
\end{align*}
Define the deterministic Bellman and variance envelopes
\begin{align}
\overline B:={}&14000\bigg(\sqrt{SA L_N\log\frac1\delta\left(K+G+16\log\frac1\delta\right)} \notag \\
&+\sqrt{S^2A L_N\log\frac1\delta\left(G+\epsilon S\left(K+\log\frac1\delta\right)\right)} \notag \\
&+3200S^2A L_N\log\frac1\delta\bigg), \label{eq:theorem-Bbar}\\
\overline V:={}&4K+4\overline B+4G+64\log\frac1\delta.
\label{eq:theorem-Vbar}
\end{align}

We next construct the two theorem-level martingale events.  Flatten the
stopped transitions in chronological $(k,h)$ order and let
$\cF_{k,h}^{-}$ be the sigma-field immediately before
$s_{h+1}^k$ is drawn.  Then $I_h^k$, $P_h^k$, and $V_{h+1}^k$ are
$\cF_{k,h}^{-}$-measurable, and $
    \xi_{k,h}:=I_h^k\left(
        P_h^kV_{h+1}^k-V_{h+1}^k(s_{h+1}^k)
    \right)$ 
is conditionally centered, satisfies $|\xi_{k,h}|\leq1$, and has predictable
quadratic variation
\begin{align*}
    \sum_{k=1}^K\sum_{h=1}^{H_1}
    \bbE\left[\xi_{k,h}^2\mid\cF_{k,h}^{-}\right]
    =\Sigma_V.
\end{align*}
Set $
    \eta:=\min\left\{1,
    \sqrt{\frac{\log\frac1\delta}{\overline V}}\right\}.$
This is a deterministic number in $(0,1]$.  By
\Cref{lem:linear_freedman}, there is an event
$\cE_{L_1}$ with failure probability at most $\delta$ on which
\begin{align}
    L_1
    \leq\eta\Sigma_V+\frac{\log\frac1\delta}{\eta}.
\label{eq:theorem-L1-freedman}
\end{align}
In particular, whenever $\Sigma_V\leq\overline V$,
\begin{align}
    L_1
    \leq2\sqrt{\overline V\log\frac1\delta}+\log\frac1\delta.
\label{eq:theorem-L1-envelope}
\end{align}
Indeed, if $\overline V\geq\log\frac1\delta$ the two terms in
\Cref{eq:theorem-L1-freedman} are equal, while if
$\overline V<\log\frac1\delta$ then $\eta=1$ and
$\overline V+\log\frac1\delta
\leq2\sqrt{\overline V\log\frac1\delta}+\log\frac1\delta$.

For the realized-reward residual, define a nested episode filtration as
follows.  Let $\cG_{k-1}$ be the sigma-field just after the initial
state $s_1^k$ has been revealed and the policy $\pi^k$ has been selected, but before the episode-$k$ trajectory is drawn.
After episode $k$, reveal the next initial state and select the next behavior
rule to obtain $\cG_k$. Let $
    Z_k:=\sum_{h=1}^{H_1}r_h^k-V_1^{\pi^k}(s_1^k)$. Then
the defining conditional-value identity is
\begin{align*}
\bbE\left[\left.\sum_{h=1}^{H_1}r_h^k\,\right|\cG_{k-1}\right]
    =V_1^{\pi^k}(s_1^k).
\end{align*}
Thus $\bbE[Z_k\mid\cG_{k-1}]=0$ and $Z_k\in[-1,1]$.
Applying \Cref{lemma:hoeffding} to $\frac{Z_k+1}{2}$ with confidence parameter
$\delta$ defines an event $\cE_{\mathsf{rew}}$ with failure
probability at most $\delta$ on which
\begin{align}
    \sum_{k=1}^K Z_k
    \leq\sqrt{2K\log\frac2\delta}.
\label{eq:theorem-reward-residual}
\end{align}

Let $\cE_B$ be the  event from
\Cref{lem:sum-bellman-section2}.  It already contains
$\cE_{\mathsf{1step}}$, $\cE_{\mathsf{count}}$, and
$\cE_{\mathsf{var}}$, together with the empirical and clipped-variance
events used in that lemma.  Let $\cE_{\mathsf{cell}}$ be the
additional exploration-cell event used in the proof of
\Cref{lemma:bd_Lstop}.  Since the shared count/reference events are already
inside $\cE_B$, the new charge for this event is at most
$SA(S^2A+1)\delta$.  Define  $
    \cE_{\mathsf{glob}}
    :=\cE_B\cap\cE_{\mathsf{cell}}
      \cap\cE_{L_1}\cap\cE_{\mathsf{rew}}.$ 
The total failure probability is
\begin{align}
\bbP(\cE_{\mathsf{glob}}^\complement)
\leq{}&\bigg[
\left(
\frac{6S^2AL_N}{\epsilon^2}
+3S^2AL_N+2S^3A+2SA+4
\right)
+SA(S^2A+1)+1+1
\bigg]\delta
\notag\\
={}&\left(
6S^6AL_N+3S^2AL_N+S^3A^2+2S^3A+3SA+6
\right)\delta.
\label{eq:theorem-global-ledger}
\end{align}

Assume $\cE_{\mathsf{glob}}$ in the rest of the analysis.    Since $\cE_{\mathsf{var}}\subseteq\cE_B$,
\Cref{lem:variance-closure-section2} then yields
\begin{align*}
    \Sigma_V
    \leq4K+4B+4L_{\stoplabel}+64\log\frac1\delta
    \leq\overline V.
\end{align*}
Thus \Cref{eq:theorem-L1-envelope} applies.  \Cref{lemma:bd_Lstop} gives $L_{\stoplabel}\leq G$.   Substituting this bound in \Cref{lem:sum-bellman-section2}  gives
$B\leq\overline B$. The event
$\cE_{\mathsf{1step}}\subseteq\cE_B$ also permits the direct
use of \Cref{lem:stopped-regret-section2}.  Combining that decomposition,
\Cref{eq:theorem-L1-envelope,eq:theorem-reward-residual} and the
truncation comparison gives that 
\begin{align}
\Regret_H(K)
\leq{}&K\upsilon+G+\overline B
+2\sqrt{\overline V\log\frac1\delta}+\log\frac1\delta
+\sqrt{2K\log\frac2\delta}.
\label{eq:theorem-regret-envelope}
\end{align}

It remains to convert the internal confidence and simplify the envelope.
From the definition of $\upsilon$, $
\frac1\upsilon
=\max\left\{
\sqrt{\frac{1000K}{SA}},\,20S\log S
\right\}
\leq60S^2K.$ 
Therefore the theorem's confidence choice gives
\begin{align*}
\frac{504S^2K\log S}{\upsilon\delta\log\frac1\delta}
&\leq
\frac{72S^2K}{\upsilon\delta}
\leq
\frac{907200S^{11}A^2K^3}{\delta_0^2}
\leq
\left(\frac{SAK}{\delta_0}\right)^{11}.
\end{align*}

Consequently, $
    L_N
    \leq22\log\frac{SAK}{\delta_0}.$ 
Let $C_{\mathsf{glob}}$ denote the coefficient of $\delta$ in the second
line of \Cref{eq:theorem-global-ledger}.  Then we have 
\begin{align*}
C_{\mathsf{glob}}
&\leq9S^6A L_N+12S^3A^2\leq198S^6A\log\frac{SAK}{\delta_0}+12S^3A^2\leq\frac{210S^7A^2K}{\delta_0}.
\end{align*}
 Hence $
    \bbP(\cE_{\mathsf{glob}}^\complement)
    \leq C_{\mathsf{glob}}\delta
    \leq\delta_0$. That is, the global event has the probability claimed in the theorem.

Finally, we have
\begin{align}
    G
    =O\left(S^8A^3\log\frac1\delta+SA L_N\right)
    =O\left(S^8A^3\log\frac{SAK}{\delta_0}\right).
\label{eq:theorem-G-asymptotic}
\end{align}
Substituting this into \Cref{eq:theorem-Bbar} yields
\begin{align*}
\overline B
=O\bigg(&
\sqrt{SAK}\log\frac{SAK}{\delta_0}
+\sqrt{SA G}\log\frac{SAK}{\delta_0}\\
&+S\sqrt{AG}\log\frac{SAK}{\delta_0}
+\sqrt{SA}\left(\log\frac{SAK}{\delta_0}\right)^{3/2}+S^2A\log^2\frac{SAK}{\delta_0}
\bigg).
\end{align*}
Therefore,
\begin{align}
    \overline B
    =O\left(
        \sqrt{SAK}\log\frac{SAK}{\delta_0}
        +S^8A^3\log^2\frac{SAK}{\delta_0}
    \right).
\label{eq:theorem-B-asymptotic}
\end{align}
\Cref{eq:theorem-Vbar,eq:theorem-G-asymptotic,eq:theorem-B-asymptotic} imply
\begin{align*}
&2\sqrt{\overline V\log\frac1\delta}+\log\frac1\delta\\
&=O\bigg(
\sqrt{K\log\frac{SAK}{\delta_0}}
+(SAK)^{1/4}\log\frac{SAK}{\delta_0}
+S^4A^{3/2}\left(\log\frac{SAK}{\delta_0}\right)^{3/2}
+\log\frac{SAK}{\delta_0}
\bigg) \\
&\le O \left(\sqrt{SAK}\log\frac{SAK}{\delta_0}+S^8A^3\log^2\frac{SAK}{\delta_0}\right).
\end{align*}

By absorbing the remaining polynomial terms and logarithmic factors into \Cref{eq:theorem-regret-envelope}, we obtain that
\begin{align*}
    \Regret_H(K)
    =O\left(
        \sqrt{SAK}\log\frac{SAK}{\delta_0}
        +S^8A^3\log^2\frac{SAK}{\delta_0}
    \right)
\end{align*}
on an event of probability at least $1-\delta_0$.
\end{proof}

\newpage
\section{Proof of Optimism}\label{sec:p_opt}

In this section, we prove optimism of the planning values.  Specifically, we
show that, with high probability, $Q_h^k(s,a)\ge Q_h^\star(s,a)$ and
$V_h^k(s)\ge V_h^\star(s)$ for all relevant episodes $k$, layers
$h\in[H_1]$, states $s$, and actions $a$.

The proof has three parts.  We first introduce the notation needed for the
bonus analysis in \Cref{app:bnotations}.  We then state the concentration
events for the optimal value sequence $\{V_h^\star\}_{h=1}^H$ in
\Cref{app:con_p}.  Finally, in \Cref{app:mono_pf}, we prove the relaxed
monotonicity property of the bonus, which allows the optimism induction to go
through despite the discontinuity introduced by the cutting operation.

\subsection{Additional Notations}\label{app:bnotations}

Fix $p\in \Delta^{S}$, $n\in \bbN$, $v\in [0,1]^{S}$, and $0<\epsilon<1$.
Recall the projection and cutting operators from \Cref{def:proj,def:cut}.
Then
\begin{align*}
    |pv - \ell_{\epsilon}(p,v)| &\le |pv - p  \Proj_{\epsilon}(v)| + |p \Proj_{\epsilon}(v)-\ell_{\epsilon}(p,v)|\le \epsilon+\epsilon =2\epsilon.
\end{align*}
Moreover, the definition of $\CutProj_{\epsilon}$ gives
\begin{align*}
    \norm{v-\CutProj_{\epsilon}(p,v)-\ell_{\epsilon}(p,v)\II}_{\infty} \le 3\epsilon.
\end{align*}
Indeed, writing $w=\Proj_{\epsilon}(v)$ and $r_s=v(s)-w(s)\in[0,\epsilon)$, the residual is at most $2\epsilon+r_s<3\epsilon$ in the positive tail, at least $-3\epsilon+r_s\ge -3\epsilon$ in the negative tail, and has magnitude at most $2\epsilon+r_s<3\epsilon$ on the central set.

To facilitate the analysis, define the three sets:
\begin{align*}
\cI_{\epsilon}(p,v) &:= \{s: |\Proj_{\epsilon}(v)(s)-\ell_{\epsilon}(p,v)|\le 2\epsilon\},\notag\\
\cI_{\epsilon}^+(p,v) &:= \{s: \Proj_{\epsilon}(v)(s)-\ell_{\epsilon}(p,v)>2\epsilon\},\notag\\
\cI_{\epsilon}^{-}(p,v) &:= \{s: \Proj_{\epsilon}(v)(s)-\ell_{\epsilon}(p,v)<-2\epsilon\}.
\end{align*}

\paragraph{$x$-clipped variance.}
Given $p\in\Delta(S)$ and $v\in\bbR^S$, recall
\begin{align*}
    \bbV(p,v) := \sum_s p(s) (v(s)-pv)^2,
    \qquad \bbV_x(p,v) := \sum_s p(s)\min\{(v(s)-pv)^2,x^2\}, \quad x>0.
\end{align*}

 Recall bonus $b$ (\Cref{eq:bonus}), and define the auxiliary bonus $\ud{b}$: 
\begin{align*}
 b(p,v,n) &= 3\sqrt{\frac{\bbV(p, v) \log \frac{1}{\delta}}{n}} + 5\sqrt{\frac{S \bbV_{5\epsilon}(p, v) \log \frac{1}{\delta} }{n}} + 20\frac{S \log \frac{1}{\delta}}{n},\notag \\
 \ud{b}(p,v,n) &= 3\sqrt{\frac{\bbV(p, v) \log \frac{1}{\delta}}{n}} + 5\sqrt{\frac{S \bbV(p, v - \CutProj_{\epsilon}(p,v) ) \log \frac{1}{\delta} }{n}} + 20\frac{S \log \frac{1}{\delta}}{n}.
\end{align*}

\subsection{Concentration Property}\label{app:con_p}

\begin{lemma}\label{lemma:cut_complexity}
Fix $0<\epsilon<1$, and recall \Cref{def:cut} for the definition of $\CutProj_{\epsilon}(\cdot,\cdot)$.
Define
\begin{align*}
    \mathfrak C_{\epsilon} :=\{\CutProj_{\epsilon}(p,V_h^*):p\in\Delta(\cS),\ h\in[H]\}.
\end{align*}
Then $|\mathfrak C_{\epsilon}|\le \frac{4S}{\epsilon^2}$.
Moreover, $\norm{\CutProj_{\epsilon}(p,V_h^*)}_{\infty}\le1$ and $0 \le (V_h^*-\CutProj_{\epsilon}(p,V_h^*)) (s) \le 1$ for any $p \in \Delta(\cS)$ and $(h, s) \in [H] \times \cS$.
\end{lemma}

\begin{proof}
First, the monotonicity $V_h^*(s)\ge V_{h+1}^*(s)$ and the bound $0\le V_h^*(s)\le1$ imply
\begin{align*}
    |\{\Proj_{\epsilon}(V_h^*):h\in[H]\}| \le 1+S\Big\lfloor\frac{1}{\epsilon}\Big\rfloor \le 1+\frac{S}{\epsilon} \le \frac{2S}{\epsilon}.
\end{align*}
Indeed, the integer vector $\lfloor \frac{V_h^*}{\epsilon}\rfloor$ is coordinatewise nonincreasing in $h$, each coordinate lies in $\{0,\ldots,\lfloor\frac{1}{\epsilon}\rfloor\}$, and every strict change of the projected vector decreases the integer potential $\sum_s\lfloor \frac{V_h^*(s)}{\epsilon}\rfloor$ by at least one.

Fix $w=\Proj_{\epsilon}(V_h^*)$.
Once $w$ is fixed, $\CutProj_{\epsilon}(p,V_h^*)$ is determined by
\begin{align*}
    \ell_{\epsilon}(p,V_h^*)=\Proj_{\epsilon}(p w).
\end{align*}
Since $p w\in[0,1]$, this scalar belongs to the grid $\{j\epsilon:0\le j\le \lfloor\frac{1}{\epsilon}\rfloor\}$, with the endpoint $1$ included only when it lies on the grid.  Hence it has at most $1 + \lfloor\frac{1}{\epsilon}\rfloor \le \frac{2}{\epsilon}$ possible values.
Therefore $|\mathfrak C_{\epsilon}|\le \left(\frac{2S}{\epsilon}\right)\left(\frac{2}{\epsilon}\right)=\frac{4S}{\epsilon^2}$.

It remains to check the range claims.
Write $w=\Proj_{\epsilon}(V_h^*)$, $\ell=\ell_{\epsilon}(p,V_h^*)$, and $c=\CutProj_{\epsilon}(p,V_h^*)$.
If $s\in\cI_\epsilon(p,V_h^*)$, then $c(s)=0$ and $V_h^*(s)-c(s)=V_h^*(s)\in[0,1]$.
If $s\in\cI^+_\epsilon(p,V_h^*)$, write $w(s)-\ell=m\epsilon$ with $m\ge3$ and $V_h^*(s)=w(s)+r_s$ for $r_s\in[0,\epsilon)$.
Since $m\epsilon=w(s)-\ell\le1$ and $m\ge3$,
\begin{align*}
    c(s)=(m-2)\epsilon\in[0,1],\qquad
    V_h^*(s)-c(s)=\ell+2\epsilon+r_s\in[0,V_h^*(s)]\subseteq[0,1].
\end{align*}
If $s\in\cI^-_\epsilon(p,V_h^*)$, write $w(s)-\ell=-m\epsilon$ with $m\ge3$.
Since $w(s)\ge0$, $m\epsilon\le \ell\le1$, and hence
\begin{align*}
    c(s)=(-m+3)\epsilon\in[-1,0],\qquad
    V_h^*(s)-c(s)=\ell-3\epsilon+r_s\in[0,1],
\end{align*}
where the upper bound on $V_h^*(s)-c(s)$ uses $\ell\le1$ and $r_s<\epsilon$.
Thus $\norm c_\infty\le1$ and $V_h^*-c\in[0,1]^S$.
\end{proof}

\begin{lemma}\label{lemma:one_step_conc}
 Let $\cE_{\mathsf{count}}$ be the event in
\Cref{lem:inverse-count-section2}.  There exists an event
$\cE_{\mathsf{one}}$ with  $
\bbP(\cE_{\mathsf{one}}^\complement)
\leq\frac{6S^2AL_N}{\epsilon^2}\delta$
such that, on
$\cE_{\mathsf{1step}}:=\cE_{\mathsf{one}}\cap
\cE_{\mathsf{count}}$, the following two inequalities hold simultaneously
for all $(k,h)\in[K]\times[H]$ and all 
$(s,a)\notin\cO^k$:
\begin{align}
&\left|\left(\wh{P}_{s,a}^k-P_{s,a}\right)
\CutProj_{\epsilon}\left(\wh{P}_{s,a}^k,V_h^*\right)\right| \le 2\sqrt{\frac{\bbV\left(\wh{P}_{s,a}^k,\CutProj_{\epsilon}\left(\wh{P}_{s,a}^k,V_h^*\right)\right)\log \frac{1}{\delta}}{N^k(s,a)}} +10\frac{\log \frac{1}{\delta}}{N^k(s,a)}, \label{eq:argument3_explicit}
\end{align}
and for every $u\in[0,1]^S$,
\begin{align}
&\abs{\left(\wh{P}_{s,a}^k-P_{s,a}\right)u} \le 2\sqrt{\frac{S\bbV\left(\wh{P}_{s,a}^k,u\right)\log \frac{1}{\delta}}{N^k(s,a)}} +10\frac{S\log \frac{1}{\delta}}{N^k(s,a)}.
\label{eq:any_u_coordinate_conc}
\end{align}
Specifically, taking $u = V_h^*-\CutProj_{\epsilon}\left(\wh{P}_{s,a}^k,V_h^*\right)$ gives
\begin{align}
&\left|\left(\wh{P}_{s,a}^k-P_{s,a}\right)
\left(V_h^*-\CutProj_{\epsilon}\left(\wh{P}_{s,a}^k,V_h^*\right)\right)\right| \le 2\sqrt{\frac{S\bbV\left(\wh{P}_{s,a}^k,V_h^*-\CutProj_{\epsilon}\left(\wh{P}_{s,a}^k,V_h^*\right)\right)\log \frac{1}{\delta}}{N^k(s,a)}} +10\frac{S\log \frac{1}{\delta}}{N^k(s,a)}. \label{eq:argument4_explicit}
\end{align}
Moreover,  $
\bbP(\cE_{\mathsf{1step}})
\geq1-\left(
\frac{6S^2AL_N}{\epsilon^2}+4S^3A+2SA
\right)\delta.$ 
\end{lemma}

\begin{proof}
For each $(s,a)$, let
\begin{align*}
    \cN(s,a):=
    \left\{2^i:i\in\bbN_0,
    \frac{\upsilon\log\frac1\delta}{504S^2\log S}U(s,a)\le2^i\le
    \frac{KU(s,a)}{\delta}\right\}\cup\{1\}.
\end{align*}
This is a deterministic set fixed by the MDP and the algorithmic parameters, and does not depend on the randomness in running the algorithm.
On $\cE_{\mathsf{count}}$, every frozen count $N^k(s,a)$
belongs to $\cN(s,a)$: the learned-count floor gives the lower endpoint,
the final raw-count bound gives the upper endpoint, and learned frozen
counts are powers of two.  Moreover,
$|\cN(s,a)|\le L_N$.

Couple the transition draws from each row with an infinite i.i.d. stack $Y_1(s,a),Y_2(s,a),\ldots\sim P_{s,a}$, where the $i$-th visit to $(s,a)$ uses $Y_i(s,a)$.  Thus adaptive visit times do not change the law of the ordered successor stack.  For $n\ge1$, let $\widehat P_{s,a}^{(n)}$ be the empirical distribution of its first $n$ entries.  The doubling rule and the frozen-count definition give $
    \widehat P_{s,a}^k=\widehat P_{s,a}^{(N^k(s,a))}$ for all $(s,a)\notin \mathcal{O}^k$. 

First fix $(s,a)$, $\boldc\in\mathfrak C_\epsilon$, and
$n\in\cN(s,a)$.  If $n\in\{1,2\}$, then
\Cref{lemma:cut_complexity} gives that 
\begin{align*}
    \abs{(P_{s,a}-\widehat P_{s,a}^{(n)})\boldc}
    \le2
    \le\frac{10\log\frac1\delta}{n}.
\end{align*}
For $n\ge4$, set $Z_i=\boldc(Y_i(s,a))+1\in[0,2]$.  The affine shift leaves both the mean difference and the empirical variance unchanged, and
\begin{align*}
    \frac1n\sum_{i=1}^n(Z_i-\overline Z)^2
    =\bbV\left(\widehat P_{s,a}^{(n)},\boldc\right).
\end{align*}
Applying \Cref{lem:bennett_empirical} with range parameter $b=1$ gives, with failure probability at most $\delta$,
\begin{align*}
    \abs{(P_{s,a}-\widehat P_{s,a}^{(n)})\boldc}
    &\le \sqrt{\frac{2\bbV\left(\widehat P_{s,a}^{(n)},\boldc\right)\log \frac{2}{\delta}}{n-1}}
    +\frac{14\log \frac{2}{\delta}}{3(n-1)} \\
    &\le 2\sqrt{\frac{\bbV\left(\widehat P_{s,a}^{(n)},\boldc\right)\log\frac1\delta}{n}}
    +10\frac{\log\frac1\delta}{n}.
\end{align*}
A union bound over all state-action pairs, the at most $L_N$ values in each $\cN(s,a)$, and $\boldc\in\mathfrak C_\epsilon$, together with \Cref{lemma:cut_complexity}, gives a concentration event $\cE_{\mathsf{cut}}$ satisfying the failure probability is at most $
    \bbP(\cE_{\mathsf{cut}}^\complement)
    \leq SA\cdot L_N\cdot\frac{4S}{\epsilon^2}\delta
    =\frac{4S^2AL_N}{\epsilon^2}\delta.$
    
  On $\cE_{\mathsf{cut}}\cap\cE_{\mathsf{count}}$, \Cref{eq:argument3_explicit} follows by taking $\boldc=\CutProj_\epsilon(\widehat P_{s,a}^k,V_h^*)$.

Next fix $(s,a)$, $s'\in\cS$, and $n\in\cN(s,a)$.  If $n\in\{1,2\}$, then deterministically
\begin{align*}
    \abs{P_{s,a,s'}-\widehat P_{s,a,s'}^{(n)}}
    \le1
    \le\frac{10\log\frac1\delta}{n}.
\end{align*}
For $n\ge4$, applying \Cref{lem:bennett_empirical} to the Bernoulli variables $\II\{Y_i(s,a)=s'\}$ gives, with failure probability at most $\delta$,
\begin{align*}
    \abs{P_{s,a,s'}-\widehat P_{s,a,s'}^{(n)}}
    &\le \sqrt{\frac{2\widehat P_{s,a,s'}^{(n)}(1-\widehat P_{s,a,s'}^{(n)})\log \frac{2}{\delta}}{n-1}}
    +\frac{7\log \frac{2}{\delta}}{3(n-1)} \\
    &\le 2\sqrt{\frac{\widehat P_{s,a,s'}^{(n)}\log\frac1\delta}{n}}
    +10\frac{\log\frac1\delta}{n}.
\end{align*}
 A union bound over $(s,a)$, $s'$, and the at most $L_N$ values in $\cN(s,a)$ gives a concentration event $\cE_{\mathsf{coord}}$ satisfying $
    \bbP(\cE_{\mathsf{coord}}^\complement)\leq S^2AL_N\delta$.

Assume $\cE_{\mathsf{coord}}\cap\cE_{\mathsf{count}}$. For any fixed triple $(s,a,s')$, since $(P_{s,a}-\wh P_{s,a}^k)\II=0$,
\begin{align*}
    \abs{(\wh P_{s,a}^k-P_{s,a})u}
    &=\abs{(P_{s,a}-\wh P_{s,a}^k)(u-\wh P_{s,a}^k u\cdot\II)} \\
    &\le 2\sqrt{\frac{\log\frac1\delta}{N^k(s,a)}}
    \sum_{s'}\sqrt{\wh P_{s,a,s'}^k}\abs{u(s')-\wh P_{s,a}^k u}
    +10\frac{\log\frac1\delta}{N^k(s,a)}\sum_{s'}\abs{u(s')-\wh P_{s,a}^k u} \\
    &\le 2\sqrt{\frac{S\bbV(\wh P_{s,a}^k,u)\log\frac1\delta}{N^k(s,a)}}
    +10\frac{S\log\frac1\delta}{N^k(s,a)},
\end{align*}
where the last line uses Cauchy--Schwarz and $\sum_{s'}\abs{u(s')-\wh P_{s,a}^k u}\le S$. 

Set $\cE_{\mathsf{one}}:=\cE_{\mathsf{cut}}\cap
\cE_{\mathsf{coord}}$.  Its exact failure bound is
\begin{align*}
    \bbP(\cE_{\mathsf{one}}^\complement)
    &\leq\left(\frac{4}{\epsilon^2}+1\right)S^2AL_N\delta
    \le\frac{5S^2AL_N}{\epsilon^2}\delta
    \le\frac{6S^2AL_N}{\epsilon^2}\delta.
\end{align*}
On $\cE_{\mathsf{one}}\cap\cE_{\mathsf{count}}$,
\Cref{lemma:cut_complexity} gives
$V_h^*-\CutProj_\epsilon(\wh P_{s,a}^k,V_h^*)\in[0,1]^S$, so
\Cref{eq:argument4_explicit} follows from
\Cref{eq:any_u_coordinate_conc} with this choice of $u$.
Finally, the probability bound for $\cE_{\mathsf{1step}}$ follows by
combining the displayed bound for $\cE_{\mathsf{one}}$ with
$\bbP(\cE_{\mathsf{count}}^\complement)\leq(4S^3A+2SA)\delta$ from
\Cref{lem:inverse-count-section2}.
\end{proof}

\begin{proposition}[Bonus optimism]\label{prop:hfree_concentration_optimism}
Assume
$\cE_{\mathsf{1step}}$.  For all
$(k,h)\in[K]\times[H]$ and $(s,a)\notin \cO^k$,
\Cref{eq:argument3_explicit,eq:argument4_explicit} hold, and
\Cref{eq:any_u_coordinate_conc} holds for every $u\in[0,1]^S$.
Consequently, for every
$(k,h)\in [K]\times[H]$ and every $(s,a)\notin\cO^k$,
\begin{align}
   | \wh{P}^k_{s,a} V_h^*-
    P_{s,a}  V_h^*| \leq \ud{b}\left(\wh{P}_{s,a}^k,V_h^*, N^k(s,a)\right). \label{eq:underline_optimism_explicit}
\end{align}
\end{proposition}

\begin{proof}
Fix
$(k,h)\in[K]\times[H]$ and  $(s,a)\notin\cO^k$.
Let $c=\CutProj_{\epsilon}\left(\wh{P}_{s,a}^k,V_h^*\right)$ and $u=V_h^*-c$.

By \Cref{lemma:one_step_conc}, we have that
\begin{align}
    \abs{\left(\wh{P}_{s,a}^k-P_{s,a}\right)V_h^*}
    &\le 2\sqrt{\frac{\bbV\left(\wh{P}_{s,a}^k,c\right)\log \frac{1}{\delta}}{N^k(s,a)}}
    +2\sqrt{\frac{S\bbV\left(\wh{P}_{s,a}^k,u\right)\log \frac{1}{\delta}}{N^k(s,a)}}
    +10\frac{\log \frac{1}{\delta}}{N^k(s,a)}
    +10\frac{S\log \frac{1}{\delta}}{N^k(s,a)} \notag\\
    &\le 2\sqrt{\frac{\bbV\left(\wh{P}_{s,a}^k,c\right)\log \frac{1}{\delta}}{N^k(s,a)}}
    +2\sqrt{\frac{S\bbV\left(\wh{P}_{s,a}^k,u\right)\log \frac{1}{\delta}}{N^k(s,a)}}
    +20\frac{S\log \frac{1}{\delta}}{N^k(s,a)}.\label{eq:add_cut_residual_conc}
\end{align}
Since $c=V_h^*-u$, it holds that $
    \bbV\left(\wh{P}_{s,a}^k,c\right)
    \le2\bbV\left(\wh{P}_{s,a}^k,V_h^*\right)+2\bbV\left(\wh{P}_{s,a}^k,u\right).$ 
Therefore,
\begin{align*}
    2\sqrt{\frac{\bbV\left(\wh{P}_{s,a}^k,c\right)\log \frac{1}{\delta}}{N^k(s,a)}}
    &\le 2\sqrt2\sqrt{\frac{\bbV\left(\wh{P}_{s,a}^k,V_h^*\right)\log \frac{1}{\delta}}{N^k(s,a)}}
    +2\sqrt2\sqrt{\frac{S\bbV\left(\wh{P}_{s,a}^k,u\right)\log \frac{1}{\delta}}{N^k(s,a)}}.
\end{align*}
Since $2\sqrt2\le3$ and $2\sqrt2+2\le5$, \Cref{eq:add_cut_residual_conc} yields
\begin{align*}
    \abs{\left(\wh{P}_{s,a}^k-P_{s,a}\right)V_h^*}
    &\le 3\sqrt{\frac{\bbV\left(\wh{P}_{s,a}^k,V_h^*\right)\log \frac{1}{\delta}}{N^k(s,a)}} +5\sqrt{\frac{S\bbV\left(\wh{P}_{s,a}^k,u\right)\log \frac{1}{\delta}}{N^k(s,a)}}  +20\frac{S\log \frac{1}{\delta}}{N^k(s,a)} \notag\\
    &=\ud{b}\left(\wh{P}_{s,a}^k,V_h^*,N^k(s,a)\right).
\end{align*}
\end{proof}

\subsection{Monotonic Property}\label{app:mono_pf}

\begin{restatable}{lemma}{restatableLbb}\label{lemma:lbb}
For any $0<\epsilon<1$, $p\in\Delta(S)$ and $v\in[0,1]^S$, $
    \bbV(p,v-\CutProj_{\epsilon}(p,v))\le \bbV_{5\epsilon}(p,v).$
\end{restatable}

\begin{proof} [Proof of \Cref{lemma:lbb}]

Let $w=\Proj_{\epsilon}(v)$, $\ell=\ell_{\epsilon}(p,v)$, and $c=\CutProj_{\epsilon}(p,v)$.
Since variance is the minimum squared error around a constant,
\begin{align*}
    \bbV(p,v-c) \le \sum_s p(s)(v(s)-c(s)-pv)^2.
\end{align*}
We prove that each summand is bounded by the corresponding summand of $\bbV_{5\epsilon}(p,v)$.

If $s\in\cI_\epsilon(p,v)$, then $c(s)=0$ and
\begin{align*}
    |v(s)-pv| \le |v(s)-w(s)|+|w(s)-\ell|+|\ell-pv| \le \epsilon+2\epsilon+2\epsilon =5\epsilon.
\end{align*}
Therefore
\begin{align*}
    (v(s)-c(s)-pv)^2 =(v(s)-pv)^2 =\min\{(v(s)-pv)^2,(5\epsilon)^2\}.
\end{align*}

If $s\in\cI_\epsilon^+(p,v)$, write $w(s)-\ell=m\epsilon$ for an integer $m\ge3$, and write $r_s=v(s)-w(s)\in[0,\epsilon)$.
Then $c(s)=(m-2)\epsilon$ and
\begin{align*}
    d_s:=v(s)-c(s)-pv=r_s+2\epsilon+\ell-pv.
\end{align*}
Using $|pv-\ell|\le2\epsilon$, we get $0\le d_s<5\epsilon$.
Moreover, $
    v(s)-pv=d_s+(m-2)\epsilon\ge d_s$.
     Thus $d_s^2\le\min\{(v(s)-pv)^2,(5\epsilon)^2\}$.

If $s\in\cI_\epsilon^-(p,v)$, write $w(s)-\ell=-m\epsilon$ for an integer $m\ge3$ and $r_s=v(s)-w(s)\in[0,\epsilon)$.
Then $c(s)=(-m+3)\epsilon$ and
\begin{align*}
    d_s:=v(s)-c(s)-pv=r_s-3\epsilon+\ell-pv.
\end{align*}
Again using $|pv-\ell|\le2\epsilon$, we get $-5\epsilon\le d_s<0$.
Moreover,
\begin{align*}
    v(s)-pv=d_s-(m-3)\epsilon\le d_s\le0,
\end{align*}
so $|d_s|\le|v(s)-pv|$.
Hence $d_s^2\le\min\{(v(s)-pv)^2,(5\epsilon)^2\}$.

Combining the three cases gives
\begin{align*}
    \bbV(p,v-\CutProj_{\epsilon}(p,v)) \le \sum_s p(s)\min\{(v(s)-pv)^2,(5\epsilon)^2\} =\bbV_{5\epsilon}(p,v).
\end{align*}

\end{proof}

By \Cref{lemma:lbb}, we have the following corollary.
\begin{corollary}\label{coro1}
For any $p\in \Delta(S)$, $v\in [0,1]^S$ and $n>0$, $b(p,v,n)\geq \ud{b}(p,v,n)$.
\end{corollary}

\begin{lemma}\label{lemma:mono} Recall the definition of the bonus function $b(\cdot, \cdot,\cdot)$ in \Cref{app:missing_alg}.
For any $0<\epsilon<1$, $\delta\in(0,1)$, $n>0$, $u,v\in[0,1]^S$ such that $v\ge u$, and any $p\in\Delta(S)$,
\begin{align*}
    pv+100b(p,v,n)\ge pu+b(p,u,n)\ge pu+\ud{b}(p,u,n).
\end{align*}
\end{lemma}

\begin{proof}
The second inequality follows from \Cref{coro1}.  We prove the first one.  Let $d:=p(v-u)\ge0$.
For any $x>0$, the map $y\mapsto\min\{y^2,x^2\}$ is $2x$-Lipschitz.  Therefore
\begin{align}
\bbV_x(p,u)-\bbV_x(p,v)
&\le 2x\sum_s p(s)\abs{u(s)-pu-v(s)+pv} \notag\\
&\le 2x\sum_s p(s)\left(v(s)-u(s)+d\right)
=4xd. \label{eq:m1}
\end{align}
Combining \Cref{eq:m1} with $\sqrt a-\sqrt b\le\sqrt{\max\{a-b,0\}}$, for any $C>0$ and $\gamma>0$ we obtain
\begin{align}
C\sqrt{\frac{S\log\frac1\delta}{n}}\left(\sqrt{\bbV_x(p,u)}-\sqrt{\bbV_x(p,v)}\right)
&\le 2C\sqrt{\frac{S\log\frac1\delta}{n}\,x d} \le \gamma d+\frac{C^2x}{\gamma}\frac{S\log\frac1\delta}{n}. \label{eq:mono_generic_clip}
\end{align}
Since $u,v\in[0,1]^S$, $\bbV(p,u)=\bbV_1(p,u)$ and $\bbV(p,v)=\bbV_1(p,v)$.  Taking $(x,C,\gamma)=(1,\frac{3}{\sqrt S},\frac{1}{2})$ in \Cref{eq:mono_generic_clip} gives
\begin{align*}
3\sqrt{\frac{\bbV(p,u)\log\frac1\delta}{n}}-3\sqrt{\frac{\bbV(p,v)\log\frac1\delta}{n}}
\le \frac12d+18\frac{\log\frac1\delta}{n}.
\end{align*}
Taking $(x,C,\gamma)=(5\epsilon,5,\frac{1}{2})$ gives
\begin{align*}
5\sqrt{\frac{S\bbV_{5\epsilon}(p,u)\log\frac1\delta}{n}}-5\sqrt{\frac{S\bbV_{5\epsilon}(p,v)\log\frac1\delta}{n}}
\le \frac12d+250\epsilon\frac{S\log\frac1\delta}{n}.
\end{align*}
  Hence, using $0<\epsilon<1$, $
    b(p,u,n)-b(p,v,n)
    \le d+(18+250\epsilon)\frac{S\log\frac1\delta}{n}
    \le d+268\frac{S\log\frac1\delta}{n}.$  
Since $b(p,v,n)\ge20\frac{S\log\frac1\delta}{n}$ and $268\le99\cdot20$,
\begin{align*}
    pu+b(p,u,n)
    \le pv+b(p,v,n)+268\cdot \frac{S\log\frac1\delta}{n}
    \le pv+100b(p,v,n),
\end{align*}
which proves the first inequality.
\end{proof}

\newpage
\section{Bound of the Total Clipped Variance}\label{sec:p_var}
In this section, we bound the total clipped variance term
$S\sum_{k=1}^{K}\sum_{h=1}^{H_1}
\bbV_{5\epsilon}\left(\wh{P}_{h}^k,V_{h+1}^k\right)$.  The main step is to first obtain
a horizon-free bound on the expected total clipped variance under the true
transition model.  We then combine this bound with the error analysis for the
empirical transition estimates and a rollout argument to transfer it to the
realized trajectories and to the optimistic value functions $V_{h+1}^k$.

By $\min \{(a+b)^2, c^2\} \le 2 \min\{a^2, c^2\} + 2 \min\{b^2, c^2\}$, we have that
\begin{align*}
\sum_{k=1}^{K}\sum_{h=1}^{H_1} \bbV_{5\epsilon}\left(\wh{P}_{h}^k,V_{h+1}^k\right) &\le 2\sum_{k=1}^{K}\sum_{h=1}^{H_1} \bbV_{5\epsilon}\left(\wh{P}_{h}^k,V_{h+1}^*\right) +2\sum_{k=1}^{K}\sum_{h=1}^{H_1} \bbV_{5\epsilon}\left(\wh{P}_{h}^k,V_{h+1}^k-V_{h+1}^*\right) \notag\\
&\le 2\sum_{k=1}^{K}\sum_{h=1}^{H_1}\bbV_{5\epsilon}\left(\wh{P}_{h}^k,V_{h+1}^*\right) +2\sum_{k=1}^{K}\sum_{h=1}^{H_1} \bbV\left(\wh{P}_{h}^k,V_{h+1}^k-V_{h+1}^*\right) .
\end{align*}

\begin{lemma}\label{lemma:bdev}
Recall the definition of $L_N$ in
\Cref{eq:def_L_N}.  Let
$\cE_{\mathsf{count}}$ be the event in
\Cref{lem:inverse-count-section2}.  There exists an event
$\cE_{\mathsf{bdev}}$ satisfying $
\bbP(\cE_{\mathsf{bdev}}^\complement)\leq S^2A L_N\delta$,
such that, on $\cE_{\mathsf{bdev}}\cap\cE_{\mathsf{count}}$, the
following holds simultaneously for every $k\in[K]$, every
$(s,a)\notin\cO^k$, and every $v\in[0,1]^S$:
\begin{align*}
\bbV_{5\epsilon}\left(\wh{P}_{s,a}^k,v\right)
&\le 4\bbV_{5\epsilon}(P_{s,a},v)
+50\epsilon^2S\frac{\log\frac1\delta}{N^k(s,a)}
+20\epsilon\abs{\left(\wh{P}_{s,a}^k-P_{s,a}\right)v}.
\end{align*}
In particular, for every $(k,h)\in[K]\times[H]$ with
$(s_h^k,a_h^k)\notin\cO^k$.
Consequently, the above inequality hold with probability at least $
1-\left(S^2AL_N+2S^3A+2SA\right)\delta.$
\end{lemma}

\begin{proof} [Proof of \Cref{lemma:bdev}]
For each $(s,a)$, let
\begin{align*}
    \cN(s,a):=
    \left\{2^i:i\in\bbN_0,
    \frac{\upsilon\log\frac1\delta}{504S^2\log S}U(s,a)\le2^i\le
    \frac{KU(s,a)}{\delta}\right\}\cup\{1\}.
\end{align*}
This is the deterministic possible-count set used in
\Cref{lemma:one_step_conc}.  On $\cE_{\mathsf{count}}$, every
frozen count $N^k(s,a)$ belongs to $\cN(s,a)$, and
$|\cN(s,a)|\le L_N$ by the count-set calculation in the proof of that
lemma.

Couple the transitions from each state-action pair $(s,a)$ with an infinite i.i.d. successor stack
$Y_1(s,a),Y_2(s,a),\ldots\sim P_{s,a}$, with the $i$-th visit to $(s,a)$
using $Y_i(s,a)$. For $n\ge1$, recall the definition of $\wh P_{s,a}^{(n)}$ in the proof of \Cref{lemma:one_step_conc} and set 
\begin{align*}
    \wh P_{s,a,s'}^{(n)}:=\frac1n\sum_{i=1}^n
    \II\{Y_i(s,a)=s'\}.
\end{align*}
For every $(s,a)\notin \mathcal{O}^k$, the frozen-count definition gives
$\wh P_{s,a}^k=\wh P_{s,a}^{(N^k(s,a))}$.
For fixed $(s,a)$, $s'$, and $n\in\cN(s,a)$, the upper-tail part of
Bennett's inequality in \Cref{lemma:bennet}, applied to these Bernoulli
variables, gives with failure probability at most $\delta$,
\begin{align*}
\wh P_{s,a,s'}^{(n)}-P_{s,a,s'}
&\le \sqrt{\frac{2P_{s,a,s'}(1-P_{s,a,s'})\log\frac1\delta}{n}}
+\frac{\log\frac1\delta}{3n}\\
&\le \sqrt{\frac{2P_{s,a,s'}\log\frac1\delta}{n}}
+\frac{\log\frac1\delta}{3n}.
\end{align*}
A union bound over all $(s,a,s')$ triples, and the at
most $L_N$ possible counts yields an event $\cE_{\mathsf{bdev}}$ of
probability at least $1-S^2A L_N\delta$ on which this inequality holds for
all these indices. 

Assume $\cE_{\mathsf{bdev}}\cap\cE_{\mathsf{count}}$ and fix $(s,a)\notin \mathcal{O}^k$.
Then $P_{s,a},\wh P_{s,a}^k\in\Delta(\cS)$ and, coordinate-wise,
\begin{align*}
    \wh P_{s,a,s'}^k
    &\le P_{s,a,s'}
    +\sqrt{2P_{s,a,s'}\frac{\log\frac1\delta}{N^k(s,a)}}
    +\frac{\log\frac1\delta}{3N^k(s,a)} \le2\left(P_{s,a,s'}+\frac{\log\frac1\delta}{N^k(s,a)}\right),
\end{align*}
where we use 
$\sqrt{2P_{s,a,s'}\frac{\log\frac1\delta}{N^k(s,a)}}
\le \frac{P_{s,a,s'}}{2}+\frac{\log\frac1\delta}{N^k(s,a)}$ in the last inequality.
Hence, simultaneously for every
$v\in[0,1]^S$ and every cutoff $x>0$,
\begin{align*}
\bbV_x(\wh P_{s,a}^k,v)
&=\sum_{s'}\wh P_{s,a,s'}^k
    \min\{(v(s')-\wh P_{s,a}^k v)^2,x^2\}\\
&\le 2\sum_{s'}\left(P_{s,a,s'}+\frac{\log\frac1\delta}{N^k(s,a)}\right)
    \min\{(v(s')-\wh P_{s,a}^k v)^2,x^2\}\\
&\le 2x^2S\frac{\log\frac1\delta}{N^k(s,a)}
    +2\sum_{s'}P_{s,a,s'}
    \min\{(v(s')-\wh P_{s,a}^k v)^2,x^2\}\\
&\le 2x^2S\frac{\log\frac1\delta}{N^k(s,a)}
    +4\bbV_x(P_{s,a},v)
    +4\min\{((\wh P_{s,a}^k-P_{s,a})v)^2,x^2\}\\
&\le 4\bbV_x(P_{s,a},v)
    +2x^2S\frac{\log\frac1\delta}{N^k(s,a)}
    +4x\abs{(\wh P_{s,a}^k-P_{s,a})v}.
\end{align*}
The second last inequality uses
$\min\{(a+b)^2,x^2\}\le2\min\{a^2,x^2\}+2\min\{b^2,x^2\}$ for any $a,b,x\in \mathbb{R}$, and the last inequality uses
$\min\{a^2,x^2\}\le x|a|$.  Substituting $x=5\epsilon$ gives the stated
coefficients.
Combining the failure bounds for $\cE_{\mathsf{bdev}}$ and
$\cE_{\mathsf{count}}$ gives the final probability statement.
\end{proof}

\begin{lemma} \label{lem:expected_total_deviation}
Under the standing assumption that rewards are nonnegative and every trajectory
has total reward at most one, the following holds.
Set $V_{H_1+1}^* \equiv 0$ and $\Phi_{H_1+1} \equiv 0$.
For $h \in [H_1]$, define
\begin{align*}
    \Phi_h(s) := 2\sum_{u\in\cS}\min\{V_h^*(s),V_h^*(u)\}.
\end{align*}
Then, for every $(h, s, a) \in [H_1] \times \SA$,
\begin{align*}
    \Phi_h(s) \ge \sum_{s'\in\cS} P_{s,a,s'}
    (
        |V_{h+1}^*(s') - P_{s, a} V_{h+1}^*| +
        \Phi_{h+1}(s')
    ).
\end{align*}
For every deterministic Markov policy $\pi\in\Pi$, define
\begin{align}
    D_h^\pi (s) := \bbE_\pi\!\left[
        \sum_{t=h}^{H_1} |V_{t+1}^*(s_{t+1}) - P_{s_t, a_t} V_{t+1}^*|
        \,\middle|\, s_h=s
    \right],
    \qquad D_{H_1+1}^\pi \equiv 0 . \label{eq:def_dev}
\end{align}

Then, for every $h\in[H_1]$ and $s\in\cS$,
\begin{align*}
    D_h^\pi (s) = \sum_{s' \in \cS} P_{s, \pi_h (s),s'} (
        |V_{h+1}^*(s') - P_{s, \pi_h (s)} V_{h+1}^*| +
        D_{h+1}^\pi (s')
    ) \le \Phi_h (s) \le 2 S .
\end{align*}
Also $D_{H_1+1}^\pi=\Phi_{H_1+1}=0$, so $D_h^\pi(s)\le\Phi_h(s)\le2S$
holds for all $h\in[H_1]$.
\end{lemma}

\begin{proof}
[Proof of \Cref{lem:expected_total_deviation}]
We first note two elementary facts.
Since rewards are non-negative and the total reward is at most $1$, we have $0 \le V_h^*(s) \le 1$ for every $(h,s)$.
Moreover, $V_h^*(s) \ge V_{h+1}^*(s)$ for every $h \in [H_1]$ and $s \in \cS$.
Indeed, the claim is immediate at $h=H_1$ because $V_{H_1+1}^*\equiv0$ and rewards are non-negative.
For $h\le H_1-1$, the backward-induction step is
\begin{align*}
    V_h^*(s) = \max_{a\in\cA}\{r(s,a)+P_{s,a}V_{h+1}^*\} \ge \max_{a\in\cA}\{r(s,a)+P_{s,a}V_{h+2}^*\} = V_{h+1}^*(s),
\end{align*}
using the induction hypothesis $V_{h+1}^*\ge V_{h+2}^*$.

Fix $h, s, a$.
By optimality of $V_h^*$ and non-negativity of rewards,
\begin{align*}
    V_h^*(s) \ge r(s,a) + P_{s, a} V_{h+1}^* \ge P_{s, a} V_{h+1}^*.
\end{align*}

By the monotonicity above, $V_{h+1}^*(u)\le V_h^*(u)$ for every $u\in\cS$.
Hence for $h \leq H_1-1$,
\begin{align}
    \Phi_{h+1}(s') &= 2\sum_{u\in\cS}
    \min\{V_{h+1}^*(s'),V_{h+1}^*(u)\} \notag \\
    &\le 2\sum_{u\in\cS}
    \min\{V_{h+1}^*(s'),V_h^*(u)\}. \label{eq:potential_cross_step}
\end{align}
For $h = H_1$, we also have that $\Phi_{H_1+1}(s') = 0  = 2\sum_{u\in \cS}\min\{  V_{H_1+1}^*(s'), V_{H_1}^*(u) \} =0$.

Therefore it is enough to show that
\begin{align}
    2\sum_{u\in\cS}\min\{V_h^*(s),V_h^*(u)\} \ge \sum_{s'\in\cS} P_{s,a,s'}
    \left(
        |V_{h+1}^*(s') - P_{s, a} V_{h+1}^*| +
        2\sum_{u\in\cS}\min\{V_{h+1}^*(s'),V_h^*(u)\}
    \right). \label{eq:exp_total_dev_recursion}
\end{align}

Using $
    2\min\{\alpha,\beta\} = \alpha+\beta-|\alpha-\beta|$, 
we obtain
\begin{align*}
    &2\sum_{u\in\cS}\min\{V_h^*(s),V_h^*(u)\} -
    \sum_{s'\in\cS} P_{s,a,s'}
    \cdot 2\sum_{u\in\cS}\min\{V_{h+1}^*(s'),V_h^*(u)\} \\
    &= \sum_{u\in\cS}
    \left(
        V_h^*(s) -
        P_{s, a} V_{h+1}^* +
        \sum_{s'\in\cS} P_{s,a,s'}
        |V_{h+1}^*(s')-V_h^*(u)| - |V_h^*(s)-V_h^*(u)|
    \right).
\end{align*}

For every $u\in\cS$, the corresponding summand is non-negative.
Indeed,
\begin{align}
    |V_h^*(s)-V_h^*(u)| &= | (
            P_{s, a} V_{h+1}^* -
            V_h^*(u)
        ) + (
            V_h^*(s) -
            P_{s, a} V_{h+1}^*
        )
    | \notag \\
    &\le |
        P_{s, a} V_{h+1}^* -
        V_h^*(u)
    | +
    V_h^*(s) -
    P_{s, a} V_{h+1}^* \notag\\
    &\le \sum_{s' \in\cS} P_{s, a} (s')
    |V_{h+1}^*(s')-V_h^*(u)| +
    V_h^*(s) -
    P_{s, a} V_{h+1}^*, \label{eq:u_term_nonneg}
\end{align}
where the last step is Jensen's inequality applied to $x\mapsto |x - V_h^* (u)|$.

Now consider the special term with $u=s$.
It equals
\begin{align*}
    V_h^*(s) -
    P_{s, a} V_{h+1}^* +
    \sum_{s'\in\cS} P_{s,a,s'}
    |V_{h+1}^*(s')-V_h^*(s)|.
\end{align*}
For every $s'\in\cS$,
\begin{align*}
    |
        V_{h+1}^*(s') -
        P_{s, a} V_{h+1}^*
    | &\le |V_{h+1}^*(s')-V_h^*(s)| + |
        V_h^*(s) -
        P_{s, a} V_{h+1}^*
    | \\
    &= |V_{h+1}^*(s')-V_h^*(s)| +
    V_h^*(s) -
    P_{s, a} V_{h+1}^*,
\end{align*}
because $V_h^*(s)-P_{s,a}V_{h+1}^*\ge0$.
Averaging over $s'$ yields
\begin{align*}
    V_h^*(s) -
    P_{s, a} V_{h+1}^* +
    \sum_{s'\in\cS} P_{s,a,s'}
    |V_{h+1}^*(s')-V_h^*(s)| \ge \sum_{s'\in\cS} P_{s,a,s'} |
        V_{h+1}^*(s') -
        P_{s, a} V_{h+1}^*
    |.
\end{align*}
Combining with \Cref{eq:u_term_nonneg}, \Cref{eq:exp_total_dev_recursion} follows.
Combining this with \Cref{eq:potential_cross_step} gives the one-step inequality
for the fixed $(h,s,a)$:
\begin{align*}
    \Phi_h (s) \ge \sum_{s'\in\cS}P_{s,a}(s') (|V_{h+1}^*(s')-P_{s,a}V_{h+1}^*| +\Phi_{h+1}(s')).
\end{align*}
For a deterministic Markov policy $\pi\in\Pi$, the displayed recursion for
$D_h^\pi$ follows by conditioning on $s_{h+1}$ and using $D_{H_1+1}^\pi\equiv0$.
Since $D_{H_1+1}^\pi=\Phi_{H_1+1}=0$, the one-step inequality applied with
$a=\pi_h(s)$ and backward induction on $h$ give $D_h^\pi(s)\le \Phi_h(s)$.
Finally, since $0\le V_h^*\le1$, we have $\Phi_h(s)\le 2\sum_{u\in\cS}1=2S$.
\end{proof}

\restatableExpectedEpsVar*

\begin{proof} [Proof of \Cref{cor:expected_eps_var_from_total_deviation}] 

Since Lemma~\ref{lem:expected_total_deviation} holds for general horizons, by a slight abuse of notation, we apply it with horizon $H$. 
Fix a deterministic Markov policy $\pi\in\Pi$.  For every $h\in [H]$ and every cutoff $x>0$,
\begin{align}
    \bbV_x(P_{s_h,a_h},V_{h+1}^*) \le x\sum_{s'\in\cS}P_{s_h,a_h,s'} |V_{h+1}^*(s')-P_{s_h,a_h}V_{h+1}^*|, \label{eq:clip_var_bound_dev}
\end{align}
derived with $\min \{a^2, x^2\} \le x \abs{a}$.
Taking the conditional expectation over $s_{h+1}$, summing over $h$, and applying \Cref{lem:expected_total_deviation} gives
\begin{align*}
    \bbE_\pi\!\left[
        \sum_{h=1}^{H}\bbV_x(P_{s_h,a_h},V_{h+1}^*)
        \,\middle|\, s_1=s
    \right]
    \le xD_1^\pi(s)\le 2xS,
\end{align*}
completing the proof.
\end{proof}

\begin{lemma} \label{lem:hp_total_deviation}
Consider any sequence of policies $\pi^1,\ldots,\pi^{K}$, where $\pi^k$ may depend on the data collected before episode $k$.
Then, for every $\delta\in(0,1)$, with probability at least $1-\delta$,
\begin{align*}
    \sum_{k=1}^{K} \sum_{h=1}^{H_1} \sum_{s \in \cS} P_{h,s}^k|V_{h+1}^*(s)-P_h^kV_{h+1}^*|
    \le 3S K + 4S\log \frac{1}{\delta} .
\end{align*}
\end{lemma}

\begin{proof} [Proof of \Cref{lem:hp_total_deviation}]
Let $\cG_k$ be the sigma-field generated by all randomness before episode $k$.
For each episode $k$, define the within-episode filtration
\begin{align*}
    \cF_h^k := \sigma(\cG_k,s_1^k,a_1^k,\ldots,s_h^k,a_h^k),
    \qquad h\in[H_1],
\end{align*}
and let $\cF_{H_1+1}^k$ also include $s_{H_1+1}^k$.
Set
\begin{align*}
    X_h^k:=\Phi_h(s_h^k)\quad(h\in[H_1+1]),
    \qquad d_h^k := \sum_{s \in \cS}P_{h,s}^k|V_{h+1}^*(s)-P_h^kV_{h+1}^*|\quad(h\in[H_1]),
\end{align*}
where $\Phi_h (s)$ is defined in \Cref{lem:expected_total_deviation}, and we have $0\le X_h^k\le2S$.
Moreover, the one-step inequality in \Cref{lem:expected_total_deviation}, applied with $(s,a)=(s_h^k,a_h^k)$, gives
\begin{align*}
    \bbE[X_{h+1}^k\mid\cF_h^k] &= \sum_{s\in\cS}P_{h,s}^k\Phi_{h+1}(s) \le \Phi_h(s_h^k) -
    \sum_{s\in\cS}P_{h,s}^k
    |V_{h+1}^*(s)-P_h^kV_{h+1}^*| = X_h^k-d_h^k .
\end{align*}
We now apply \Cref{lem:hp_drift_bound} once to the whole $K$-episode trajectory.
Flatten the ordered list
\begin{align*}
    (1,1),(1,2),\ldots,(1,H_1+1),(2,1),\ldots,(K,H_1+1)
\end{align*}
into a single filtration, identifying each pair with its position in this order.
There are $K(H_1+1)$ displayed states, so the number of transitions is
$K(H_1+1)-1$.   By \Cref{lem:hp_drift_bound} with $B=2S$ and $R=K-1$, with probability at least $1-\delta$,
\begin{align*}
    \sum_{k=1}^{K}\sum_{h=1}^{H_1}d_h^k
    \le 4S\log \frac{2^K}{\delta}
    =4S K\log 2+4S\log \frac{1}{\delta}
    \le 3S K+4S\log \frac{1}{\delta},
\end{align*}
which is the desired bound.
\end{proof}

\begin{corollary} \label{cor:hp_eps_var_from_total_deviation}
With probability at least $1-\delta$, the following holds simultaneously for every cutoff $x>0$:
\begin{align*}
    \sum_{k=1}^{K}\sum_{h=1}^{H_1}\bbV_{x}(P_h^k,V_{h+1}^*) \le x\left(
        3S K +
        4S\log \frac{1}{\delta}
    \right).
\end{align*}
\end{corollary}

\begin{proof} [Proof of \Cref{cor:hp_eps_var_from_total_deviation}]
This is a direct implication of \Cref{eq:clip_var_bound_dev} and \Cref{lem:hp_total_deviation}.
\end{proof}

\newpage

\section{Explicit Exploration}\label{app:expexp}

We present the analysis related to \Cref{alg:explore}.
In particular, we state and prove \Cref{lemma:exp1} and \Cref{lemma:bdEk}.
The analysis largely follows \cite{zhang2022horizon}, with minor modifications to match our notation and the way \Cref{alg:new} invokes the exploration subroutine.

In \Cref{sec:exps1}, we introduce the notation used in the exploration analysis.
In \Cref{sec:exps2}, we establish the approximation properties of the reference model $P^{\tref,k}$.
In \Cref{sec:exps-common}, we collect the events and bounds used in both proofs.
In \Cref{sec:exps3}, we prove \Cref{lemma:exp1}, which gives a lower bound on the number of samples collected from a learned state-action pair.
Finally, in \Cref{sec:exps4}, we prove \Cref{lemma:bdEk}, which bounds the total number of visits to the unlearned set $\cO$.

\subsection{Additional Notations for \Cref{alg:explore}}\label{sec:exps1}

This section collects the notation used only in the analysis of \Cref{alg:explore}.
We keep it separate from \Cref{app:notation} because the exploration subroutine follows notation of~\cite{zhang2022horizon}.

\paragraph{Clipped and cut true models.}
Recall that $\cK^k\subseteq\cS\times\cA\times\cS$ denotes the set of known triples at the beginning of episode $k$.
For each $(s,a)$, define the known and unknown successor sets
\begin{align*}
    \cK^k(s,a)
    :=
    \{s'\in\cS:(s,a,s')\in\cK^k\},
    \qquad
    \cU^k(s,a)
    :=
    \cS\setminus \cK^k(s,a).
\end{align*}
Let $ \bar{\cS}:=\cS\cup\{z,z'\},$ where $z$ and $z'$ are two virtual states.
In episode $k$, we define the clipped true model $\bar P^k$ on $\bar{\cS}$ as follows.
For every ordinary state-action pair $(s,a)\in\cS\times\cA$ and every ordinary successor $s'\in\cS$,
\begin{align*}
    \bar P^k_{s,a,s'}
    :=
    \begin{cases}
        P_{s,a,s'}, & s'\in\cK^k(s,a),\\
        0, & s'\in\cU^k(s,a).
    \end{cases}
\end{align*}
The probability mass of unknown outgoing triples is redirected to $z$:
\begin{align*}
    \bar P^k_{s,a,z}
    :=
    \sum_{s' \in\cU^k(s,a)} P_{s,a,s'},
    \qquad
    \bar P^k_{s,a,z'}:=0 .
\end{align*}
For the two virtual states $z$ and $z'$, the transition rules are fixed:
\begin{align*}
    \bar P^k_{z,a,z'}=1,
    \qquad
    \bar P^k_{z',a,z'}=1,
    \qquad
    \forall a\in\cA .
\end{align*}
All other transition probabilities from $z$ and $z'$ are zero.

For any transition kernel $p$ on $\bar{\cS}$, define $\cut(p)$ by
\begin{align}
\cut(p)_{x,a,y}
:=
\begin{cases}
\displaystyle \frac{p_{x,a,y}}{1-p_{x,a,z}},
    &p_{x,a,z}<1,\ y\neq z,\\[0.8em]
0,  &p_{x,a,z}<1,\ y=z,\\
1,  &p_{x,a,z}=1,\ y=z,\\
0,  &p_{x,a,z}=1,\ y\neq z.
\end{cases} \label{eq:def_cut}
\end{align}

We also define the cut true model $\bar P^{\cut,k} = \cut (\bar P^k)$, which conditions the true transition row on remaining inside the known successor set.
For each $(s,a)$, let the normalizing weight
\begin{align*}
    \rho^k(s,a)
    :=
    \sum_{s' \in\cK^k(s,a)} P_{s,a,s'}.
\end{align*}
If $\rho^k(s,a)>0$, then for every ordinary successor $s'\in\cS$,
\begin{align*}
    \bar P^{\cut,k}_{s,a,s'}
    :=
    \begin{cases}
        \displaystyle
        \frac{P_{s,a,s'}}{\rho^k(s,a)},
        & s'\in\cK^k(s,a),\\[1em]
        0,
        & s'\in\cU^k(s,a),
    \end{cases}
\end{align*}
and $
    \bar P^{\cut,k}_{s,a,z}
    =
    \bar P^{\cut,k}_{s,a,z'}
    :=
    0 .$
    
If $\rho^k(s,a)=0$, then the cut model sends the row to the virtual state:
\begin{align*}
    \bar P^{\cut,k}_{s,a,z}:=1,
    \qquad
    \bar P^{\cut,k}_{s,a,s'}:=0
    \quad \forall s'\in\cS,
    \qquad
    \bar P^{\cut,k}_{s,a,z'}:=0 .
\end{align*}
Finally, the virtual states use the same fixed transition rules under the cut model:
\begin{align*}
    \bar P^{\cut,k}_{z,a,z'}=1,
    \qquad
    \bar P^{\cut,k}_{z',a,z'}=1,
    \qquad
    \forall a\in\cA .
\end{align*}

In words, $\bar P^k$ is the true model in which all unknown outgoing triples are redirected to $z$, while $\bar P^{\cut,k}$ is the true transition kernel conditioned on not leaving the known triple set $\cK^k$.

\paragraph{Reference model.}
The reference model $P^{\tref,k}$ is the empirical analogue of the cut true model.
Let $\widetilde N^k$ is the count which is used to compute $P^{\tref,k}$ before the $k$-th episode following the frozen update rule.
If $\cK^k(s,a)\neq \emptyset$, then for $s'\in\cK^k(s,a)$,
\begin{align*}
    P_{s,a}^{\tref,k}(s')
    :=
    \frac{\widetilde N^k(s,a,s')}
    {\sum_{s' \in\cK^k(s,a)}\widetilde N^k(s,a,s')}.\notag
\end{align*}

For $s'\notin\cK^k(s,a)$, the reference probability is zero.
If $\cK^k(s,a)=\emptyset$, then $P_{s,a}^{\rm ref,k}(z)=1$.
When the episode is clear from context, we omit the superscript $k$ and write $P^{\rm ref}$.

\paragraph{Target pair and trigger variables.}
If episode $k$ reaches $\cO^k$, let $\left(\widetilde s^k,\widetilde a^k\right)$ be the first unlearned state-action pair visited during the first $H_1$ steps.
If no such pair is visited, set $\left(\widetilde s^k,\widetilde a^k\right)=\perp$.
Let $\trigger^k\in\{\true,\false\}$ be the value returned by \Cref{alg:explore};
if \Cref{alg:explore} is not called, set $\trigger^k=\false$.
A call with $\trigger^k=\true$ is counted as an effective exploration call for $\left(\widetilde s^k,\widetilde a^k\right)$.

\paragraph{Suffix lengths and discount factor.}
When \Cref{alg:explore} is called, the remaining suffix has length $\frac{H}{m}$. Without loss of generality, we regard the step where \Cref{alg:explore} is called as the first step, and index the remaining steps in the suffix accordingly.
It is split into
\begin{align*}
    H_3:=\frac{H}{20mS\log S},
    \qquad
    H_2:=\frac{H}{m}-H_3,
    \qquad
    \gamma:=1-\frac1{H_3}.
\end{align*}
The long phase of length $H_2$ is used to reach a candidate state, and, after a successful reach, the short phase of length $H_3$ is used to collect samples of a candidate state-action pair.

\paragraph{Discounted occupancy and reaching objectives.}
For a transition model $p$, stationary policy $\pi$, nonnegative reward $g$, and initial distribution $\mu$, recall that
\begin{align*}
    W_\gamma^\pi(g,p,\mu)
    :=
    \bbE_{p,\pi}\left[
        \sum_{t\ge1}\gamma^{t-1}g(s_t,a_t)
        \;\middle|\;s_1\sim\mu
    \right].
\end{align*}
and
\begin{align*}
    X_\gamma^\pi(\cX,p,\mu)
    :=
    \sum_{t\ge1}\gamma^{t-1}
    \bbP_{p,\pi}
    [
        s_t\in\cX
        \land s_i\notin\cX,\ \forall i<t
        \;\mid\;s_1\sim\mu
    ],
\end{align*}
where $\cX\subset \cS$ is the target set.
We also use the finite-horizon occupancy
\begin{align*}
W_d^\pi(g,p,\mu)
:=\bbE_{p,\pi}\left[
\sum_{t=1}^d g(s_t,a_t)\mid s_1\sim\mu
\right].
\end{align*}
For each $k\in[K]$ and $s\in\cS$, if
$(\widetilde s^k,\widetilde a^k)\neq\perp$, define
\begin{align*}
u^k(s)
&:=\max_{\substack{\pi\in\Pi_{\sta}\\
\pi(\widetilde s^k)=\widetilde a^k}}
X_\gamma^\pi
(\{s\},P^{\tref,k},\II_{\widetilde s^k}),\\
v^k(s,a)
&:=\max_{\pi\in\Pi_{\sta}}
W_\gamma^\pi(\II_{s,a},P^{\tref,k},\II_s).
\end{align*}
   
\subsection{Approximation Property of $P^{\tref}$}\label{sec:exps2}

\begin{lemma}\label{lemma:con1}
For any $\delta\in(0,1)$, with probability at least $1-4S^3A\delta$, for every $k\in[K]$, $a\in\cA$, and $s, s'\in\bar{\cS}$,
\begin{align*}
\ee^{-1/(4S)}P^{\tref,k}_{s,a,s'}
\leq \bar P^{\cut,k}_{s,a,s'}
\leq \ee^{1/(4S)}P^{\tref,k}_{s,a,s'}.
\notag
\end{align*}
\end{lemma}

\begin{proof} [Proof of \Cref{lemma:con1}]
Work on the event in \Cref{lemma:prefcon1}, which happens with probability at least $1 - 4 S^3 A \delta$.

Fix an episode $k$ and an action $a$.
If a pair $(s,a)$ satisfies $\cK^k(s,a)=\emptyset$, then $\rho^k(s,a)=0$ and $P^{\tref,k}_{s,a}=\bar P^{\cut,k}_{s,a}=\II_z$.
For any $a$, we also have $P^{\tref,k}_{z,a}=\bar P^{\cut,k}_{z,a}=\II_{z'}$ and $P^{\tref,k}_{z',a}=\bar P^{\cut,k}_{z',a}=\II_{z'}$.
It remains to consider $(s,a)$ with $\cK^k(s,a)\neq\emptyset$.

For every $s' \in\cK^k(s,a)$,
\begin{align*}
\widetilde N^k(s,a,s') \geq N_{\tref} =1025S^2\log\frac1\delta >0.
\end{align*}
\Cref{eq:ax1} gives
\begin{align*}
\left|\widetilde N^k(s,a,s')-\widetilde N^k(s,a)P_{s,a,s'}\right|
&\leq
4\sqrt{\widetilde N^k(s,a)P_{s,a,s'}\log\frac1\delta}
+4\log\frac1\delta\\
&\leq
\frac{\widetilde N^k(s,a)P_{s,a,s'}}{16S+1}
+(64S+8)\log\frac1\delta.
\end{align*}
Moreover, for $S\geq200$,
\begin{align*}
(16S+1)(64S+8)
=1024S^2+192S+8
\leq1025S^2.
\end{align*}
Together with the frozen-count lower bound, this yields
\begin{align*}
\left|\widetilde N^k(s,a,s')-\widetilde N^k(s,a)P_{s,a,s'}\right|
\leq
\frac{\widetilde N^k(s,a,s')+\widetilde N^k(s,a)P_{s,a,s'}}{16S+1}.
\end{align*}
Rearranging this inequality in both directions gives
\begin{align}
\ee^{-1/(8S)}\widetilde N^k(s,a)P_{s,a,s'} &\le \frac{\widetilde N^k(s,a)P_{s,a,s'}}{1+\frac1{8S}} \notag \\
& \le \widetilde N^k(s,a,s') \notag \\
& \le \left(1+\frac1{8S}\right)\widetilde N^k(s,a)P_{s,a,s'} \notag \\
& \le \ee^{1/(8S)}\widetilde N^k(s,a)P_{s,a,s'}.
\label{eq:con1-count-mult}
\end{align}

Set $M:=\sum_{s' \in\cK^k(s,a)}\widetilde N^k(s,a,s')$ and  $
\rho^k(s,a):=\sum_{s' \in\cK^k(s,a)}P_{s,a,s'}.$ 
Summing \Cref{eq:con1-count-mult} over the current known-successor set gives
\begin{align*}
\ee^{-1/(8S)}\widetilde N^k(s,a)\rho^k(s,a)
\leq M
\leq \ee^{1/(8S)}\widetilde N^k(s,a)\rho^k(s,a).
\end{align*}
Combining with \Cref{eq:con1-count-mult}, for every $s' \in\cK^k(s,a)$,
\begin{align*}
\ee^{-1/(4S)}\underbrace{\frac{\widetilde N^k(s,a,s')}{M}}_{P^{\tref,k}_{s,a,s'}}
\leq \underbrace{\frac{P_{s,a,s'}}{\rho^k(s,a)}}_{\bar P^{\cut,k}_{s,a,s'}}
\leq \ee^{1/(4S)}\underbrace{\frac{\widetilde N^k(s,a,s')}{M}}_{P^{\tref,k}_{s,a,s'}}.
\end{align*}
This completes the proof.
\end{proof}

\begin{lemma}\label{lemma:prefcon1}
With probability at least $1-4S^3A\delta$, for every $k\in[K]$ and every $(s,a,s')\in \cK^k$,
\begin{align}
\left|\frac{\widetilde{N}^k(s,a,s')}{\widetilde{N}^k(s,a)} - P_{s,a,s'} \right|\leq   4\sqrt{\frac{P_{s,a,s'}\log \frac{1}{\delta} }{\widetilde{N}^k(s,a)}} + \frac{4\log \frac{1}{\delta}}{\widetilde{N}^k(s,a)}\label{eq:ax1}
\end{align}
where $\widetilde{N}^k(s,a)=\sum_{y \in\cS}\widetilde{N}^k(s,a,y)$ is the total count in the frozen reference snapshot for $(s,a)$.
\end{lemma}
\begin{proof} [Proof of \Cref{lemma:prefcon1}]
Fix $(s,a)$ and let $Y_1(s,a),Y_2(s,a),\ldots$ be the ordered successors observed on visits to $(s,a)$.
These variables are i.i.d. categorical with parameter $P_{s,a}$.
For each $y\in\cS$, define
\begin{align*}
    N_t(y):=\sum_{i=1}^t\bbI\{Y_i(s,a)=y\},
    \qquad
    \tau(s'):=\inf\{t\geq1:N_t(y) \ge N_{\tref}\}.
\end{align*}

Since $N_{\mathsf{ref}}=1025 S^2\log\frac1\delta$ and $S\ge200$, we have $N_{\mathsf{ref}}\ge10000\log\frac1\delta$.
Thus \Cref{lemma:prefcon2} implies with probability at least $1-4S^2\delta$ that for every ordered pair $y,y'$ with $\tau(y)<\infty$,
\begin{align*}
\left|\frac{N_{\tau(y)}(y')}{\tau(y)}-P_{s,a,y'}\right|
\leq
4\sqrt{\frac{P_{s,a,y'}\log\frac1\delta}{\tau(y)}}+
\frac{4\log\frac1\delta}{\tau(y)}.
\end{align*}

Now fix an episode $k$ and a known triple $(s,a,s')\in\cK^k$.
The frozen reference for $(s,a)$ was last initialized or refreshed at the step at which some successor $y$ entered the known set; because counts are integer, this step is exactly $\tau(y)$.
If the row is not refreshed in later episodes, the same frozen counts are kept.
Consequently $\widetilde N^k(s,a)=\tau(y)$ and $\widetilde N^k(s,a,s')=\sum_{i=1}^{\tau(y)}\mathbb{I}\{Y_i(s,a)=s'\}$ for $y$ that determined the most recent refresh, it then follows \eqref{eq:ax1} for this $(k,s,a,s')$.
Finally, taking a union bound over all $SA$ state-action pairs gives total failure probability at most $4S^3A\delta$.
\end{proof}

\subsection{Common Events and Guarantees}\label{sec:exps-common}

Set $
N_0
:=
\left\lceil N_{\tref}\right\rceil
=
\left\lceil1025S^2\log\frac1\delta\right\rceil
\leq2N_{\tref}.$ 
For every $(s,a)$, list the observed successors in time order, and let
$\cE_{\tref}$ be the event on which \Cref{lemma:prefcon2}, with threshold
$N_0$, holds simultaneously for all pairs of $(s, a)$.  Then
\begin{align}
\bbP(\cE_{\tref})
\geq
1-4S^3A\delta,
\label{eq:exp-reference-event}
\end{align}
and, on $\cE_{\tref}$, \Cref{lemma:prefcon1,lemma:con1} hold
simultaneously.

For every episode $k$ and $(s,a)$, put
\begin{align}
D^k(s,a)
:=
\max\left\{1,
\sum_{s'\in\cK^k(s,a)}\Ntotal^k(s,a,s')
\right\}.
\label{eq:def_D}
\end{align}
Let $\cR_k$ be the event that the comparison in \Cref{lemma:con1} holds
in episode $k$, and let
\begin{align*}
\cM_k
:=
\left\{
\bar P^k_{s,a,z}
\leq
\frac{4SN_{\tref}}{D^k(s,a)}
\quad
\forall(s,a)\in\cS\times\cA
\right\}.
\end{align*}
Both events are measurable before the generated suffix.

We next verify $\cM_k$.  Fix $(s,a)$ and $s'$ with
$P_{s,a,s'}>0$, and let $\tau(s')$ be the index of the $N_0$-th
occurrence of $s'$ among the successors of $(s,a)$.  Since
$N_0\geq N_{\tref}\geq10000\log\frac1\delta$,
\Cref{lemma:prefcon2} gives on $\cE_{\tref}$
\begin{align*}
\left|\frac{N_0}{\tau(s')}-P_{s,a,s'}\right|
&\leq
4\sqrt{\frac{P_{s,a,s'}\log\frac1\delta}{\tau(s')}}
+\frac{4\log\frac1\delta}{\tau(s')}\leq
\frac14P_{s,a,s'}
+\frac{20\log\frac1\delta}{\tau(s')}.
\end{align*}
Using $\log\frac1\delta\leq\frac{N_0}{10000}$ yields
$P_{s,a,s'}\leq\frac{2N_0}{\tau(s')}$.  If $D^k(s,a)=1$, the desired
bound is immediate.  Otherwise $D^k(s,a)\leq\Ntotal^k(s,a)$, and every
unknown successor with positive probability has count at most $N_0-1$.
Thus $\tau(s')>\Ntotal^k(s,a)\geq D^k(s,a)$ and
\begin{align}
\bar P^k_{s,a,z}
=
\sum_{s'\notin\cK^k(s,a)}P_{s,a,s'}
\leq
\frac{2SN_0}{D^k(s,a)}
\leq
\frac{4SN_{\tref}}{D^k(s,a)}.
\label{eq:exp-unknown-mass}
\end{align}
Consequently,
\begin{align}
\cE_{\tref}
\subseteq
\bigcap_{k=1}^K(\cR_k\cap\cM_k).
\label{eq:exp-good-events}
\end{align}

On $\cR_k$, \Cref{lemma:approx} with
$\epsilon=\frac1{4S}$ and $m_\pi\leq S$ gives, for every stationary
policy $\pi$, positive integer $d$, nonnegative reward $r$, target set
$\cX$, initial distribution $\mu$, and $\gamma\in[0,1)$,
\begin{align}
\begin{aligned}
\ee^{-1}W_d^\pi(r,P^{\tref,k},\mu)
&\leq W_d^\pi(r,\bar P^{\cut,k},\mu)
\leq\ee W_d^\pi(r,P^{\tref,k},\mu),\\
\ee^{-1}X_d^\pi(\cX,P^{\tref,k},\mu)
&\leq X_d^\pi(\cX,\bar P^{\cut,k},\mu)
\leq\ee X_d^\pi(\cX,P^{\tref,k},\mu),\\
\ee^{-1}W_\gamma^\pi(r,P^{\tref,k},\mu)
&\leq W_\gamma^\pi(r,\bar P^{\cut,k},\mu)
\leq\ee W_\gamma^\pi(r,P^{\tref,k},\mu),\\
\ee^{-1}X_\gamma^\pi(\cX,P^{\tref,k},\mu)
&\leq X_\gamma^\pi(\cX,\bar P^{\cut,k},\mu)
\leq\ee X_\gamma^\pi(\cX,P^{\tref,k},\mu).
\end{aligned}
\label{eq:exp-model-comparison}
\end{align}

Since the raw counts are integer, for every episode $k$,
\begin{align}
(s,a,s')\notin\cK^k
\quad\Longrightarrow\quad
\Ntotal^k(s,a,s')
\leq
\left\lceil N_{\tref}\right\rceil-1.
\label{eq:exp-unknown-count}
\end{align}
Likewise, \Cref{line:M_upd,line:known_removal} in \Cref{alg:explore} give, for every $(s,a)$,
\begin{align}
\left|
\left\{
k\in[K]:
(\widetilde s^k,\widetilde a^k)=(s,a),\quad
\trigger^k=\true
\right\}
\right|
\leq
\left\lceil N_{\known}\right\rceil.
\label{eq:exp-effective-call-cap}
\end{align}
Equality holds when $(s,a)$ is removed, including the call that removes it.

\begin{lemma}\label{lemma:exp-renewal}
Let $Z$ be the number of visits to $(s,a)$ in $d$ decision positions under
a stationary policy $\pi$ with $\pi(s)=a$, starting from $s$ and including
the initial visit.  Then
\begin{align*}
\bbP\left[Z\geq\frac14\bbE[Z]\right]
\geq
\frac12.
\end{align*}
\end{lemma}

\begin{proof} [Proof of \Cref{lemma:exp-renewal}]
Let $T_1,T_2,\ldots$ be the successive return times to $s$ and put
$R_j:=\min\{T_j,d\}$.  The $R_j$ are i.i.d. positive random variables,
and $Z$ is the stopping index in \Cref{lemma:li4}.  Hence
\begin{align*}
\bbP\left[Z\geq\frac12\bbE[Z]-1\right]
\geq
\frac12.
\end{align*}
If $\bbE[Z]\leq4$, the initial visit gives $Z\geq1\geq\frac14\bbE[Z]$;
otherwise $\frac12\bbE[Z]-1\geq\frac14\bbE[Z]$.
\end{proof}

\begin{lemma}\label{lemma:exp-martingale}
Let $\{\cF_j\}_{j=0}^K$ be a filtration, let $J_j\in[0,1]$ be
$\cF_{j-1}$-measurable, and let $B_j\in[0,1]$ be $\cF_j$-measurable.
If, for some $p>0$,
\begin{align*}
\bbE[B_j\mid\cF_{j-1}]
\geq
pJ_j,
\end{align*}
then, with probability at least $1-\delta$, simultaneously for every
$n\in[K]$,
\begin{align*}
\sum_{j=1}^nJ_j
<
\frac3p\sum_{j=1}^nB_j
+\frac2p\log\frac1\delta.
\end{align*}
\end{lemma}

\begin{proof} [Proof of \Cref{lemma:exp-martingale}]
This follows from the second inequality of
\Cref{lem:martingale_conc_mean} with $l=1$.
\end{proof}

\subsection{Statement and Proof of \Cref{lemma:exp1}}\label{sec:exps3}
\begin{lemma}\label{lemma:exp1}
By running \Cref{alg:new}, with probability at least $1-(4S^3A+SA)\delta$, the
following holds simultaneously for every state-action pair $(s,a)$ that is
removed from $\cO$.  Let $k_0(s,a)$ be the first episode in which
$(s,a)\notin\cO^{k_0(s,a)}$.  Then
\begin{align*}
\Ntotal^{k_0(s,a)}(s,a)
&\geq \frac{\upsilon\log\frac1\delta}{252S^2\log S}\,U(s,a),
\notag\\
N^{k_0(s,a)}(s,a)
&\geq \frac{\upsilon\log\frac1\delta}{504S^2\log S}\,U(s,a).
\notag
\end{align*}
\end{lemma}

\begin{proof} [Proof of \Cref{lemma:exp1}]
Fix a pair $(s,a)$ that is removed after episode $\widetilde k$.  Then
$k_0(s,a)=\widetilde k+1$, and \Cref{lemma:sr} gives that with probability at least $1 - (4S^3A + SA)\delta$,
\begin{align}
\Ntotal^{k_0(s,a)}(s,a)
\geq
2\log\frac1\delta\max_{\pi\in\Pi_{\sta}}
W_{H_3}^{\pi}(\II_{s,a},P,\II_s).
\label{eq:exp1-from-sr}
\end{align}
Let $q:=\lceil \frac{H}{H_3}\rceil \le 21mS\log S$.  Before the first visit to $s$, the reward
$\II_{s,a}$ is zero, and after that visit the continuation has no more time
than a fresh trajectory started from $s$.  Since the reward is nonnegative,
\begin{align*}
U(s,a)
&\leq \max_{\pi\in\Pi}W_H^{\pi}(\II_{s,a},P,\II_s)
\leq \max_{\pi\in\Pi}W_{qH_3}^{\pi}(\II_{s,a},P,\II_s).
\end{align*}
Applying \Cref{lemma:al1} with integer block length $H_3$ and $q$ blocks
therefore yields
\begin{align*}
U(s,a)
&\leq 6q\max_{\pi\in\Pi_{\sta}}
W_{H_3}^{\pi}(\II_{s,a},P,\II_s).
\end{align*}
Combining this with
\Cref{eq:exp1-from-sr} and $m=\frac{4S}{\upsilon}$ gives
\begin{align*}
\Ntotal^{k_0(s,a)}(s,a)
&\geq \frac{\log\frac1\delta}{63mS\log S}\cdot U(s,a)
 =\frac{\upsilon\log\frac1\delta}{252S^2\log S}\cdot U(s,a).
\end{align*}

The pair is learned at episode $k_0(s,a)$, so its raw count is positive and
the doubling rule gives
\begin{align*}
N^{k_0(s,a)}(s,a)
=2^{\lfloor\log_2 \Ntotal^{k_0(s,a)}(s,a)\rfloor} \geq \frac12\Ntotal^{k_0(s,a)}(s,a) \geq \frac{\upsilon\log\frac1\delta}{504S^2\log S}\,U(s,a).
\end{align*}
We complete the proof.
\end{proof}

\begin{lemma}\label{lemma:sr}
With probability at least $1-(4S^3A+SA)\delta$, we have that for every $1\leq\widetilde k\leq K$ and every $(\widetilde s,\widetilde a)\in \cO^{\widetilde k}\setminus\cO^{\widetilde k+1}$ simultaneously,
\begin{align*}
\Ntotal^{\widetilde k+1}(\widetilde s,\widetilde a)
\geq
2\max_{\pi\in\Pi_{\sta}}
W_{H_3}^{\pi}(\II_{\widetilde s,\widetilde a},P,\II_{\widetilde s})
\log\frac1\delta.
\end{align*}
\end{lemma}

\begin{proof} [Proof of \Cref{lemma:sr}]
Fix $\widetilde k$ and
$(\widetilde s,\widetilde a)\in
\cO^{\widetilde k}\setminus\cO^{\widetilde k+1}$, and let
\begin{align*}
\cC
:=
\left\{k\leq\widetilde k:
(\widetilde s^k,\widetilde a^k)=(\widetilde s,\widetilde a),\
\trigger^k=\true
\right\}.
\end{align*}
By \Cref{eq:exp-effective-call-cap},
\begin{align}
|\cC|
=
\left\lceil N_{\known}\right\rceil
\geq
10000\log\frac1\delta.
\label{eq:sr-number-calls}
\end{align}

Fix $k\in\cC$.  On $\cR_k\cap\cM_k$ (see \Cref{sec:exps-common}),
\Cref{lemma:sr-small-z} gives, for every $\pi\in\Pi_{\sta}$ satisfying
$\pi(\widetilde s)=\widetilde a$,
\begin{align}
W_\gamma^\pi(\II_z,\bar P^k,\II_{\widetilde s})
<
\frac1{40}.
\label{eq:sr-discount-z}
\end{align}
Extend $P$ to $z,z'$ using the common rows $z\to z'$ and $z'\to z'$.
Let $\pi^k$ be the reference-model optimizer run by \Cref{alg:explore}.
Then \Cref{lemma:sr-model-transfer} gives
\begin{align*}
W_\gamma^{\pi^k}(
\II_{\widetilde s,\widetilde a},P,\II_{\widetilde s})
\geq
\frac1{16}
\max_{\pi\in\Pi_{\sta}}
W_\gamma^\pi(
\II_{\widetilde s,\widetilde a},P,\II_{\widetilde s}).
\end{align*}
Two applications of \Cref{lemma:eff_single} yield
\begin{align}
W_{H_3}^{\pi^k}(\II_{\widetilde s,\widetilde a},P,\II_{\widetilde s})
&\geq
\frac13
W_\gamma^{\pi^k}(\II_{\widetilde s,\widetilde a},P,\II_{\widetilde s})
\notag\\
&\geq
\frac1{48}
\max_{\pi\in\Pi_{\sta}}
W_\gamma^\pi(\II_{\widetilde s,\widetilde a},P,\II_{\widetilde s})
\notag\\
&\geq
\frac1{144}
\max_{\pi\in\Pi_{\sta}}
W_{H_3}^{\pi}(\II_{\widetilde s,\widetilde a},P,\II_{\widetilde s}).
\label{eq:sr-one-call-mean}
\end{align}

Every reference-model optimizer satisfies
$\pi^k(\widetilde s)=\widetilde a$, since otherwise its target reward is
identically zero.
Let $X^k$ be the number of visits to $(\widetilde s,\widetilde a)$ during
its full $\frac{H}{m}$-step target suffix, including the initial visit, and
let $\cF_k$ be the history immediately before that suffix.  Since the first
$H_3$ steps are contained in the suffix, \Cref{eq:sr-one-call-mean} gives
\begin{align*}
\bbE[X^k\mid\cF_k]
\geq
\frac1{144}
\max_{\pi\in\Pi_{\sta}}
W_{H_3}^{\pi}(
\II_{\widetilde s,\widetilde a},P,\II_{\widetilde s}).
\end{align*}
Conditionally on $\cF_k$, \Cref{lemma:exp-renewal} gives
\begin{align}
\bbP\left[
X^k
\geq
\frac1{576}
\max_{\pi\in\Pi_{\sta}}
W_{H_3}^{\pi}(
\II_{\widetilde s,\widetilde a},P,\II_{\widetilde s})
\,\middle|\,\cF_k
\right]
\geq
\frac12.
\label{eq:sr-one-call-success}
\end{align}

List the calls in $\cC$ in time order, let $k_j$ be the episode of the
$j$-th call, let $X_j:=X^{k_j}$, and let $\cG_{j-1}$ be the history
immediately before its target suffix.  Set $
J_j
:=
\bbI\{\cR_{k_j}\cap\cM_{k_j}\}$ 
and $
B_j
:=
J_j
\bbI\left\{
X_j
\geq
\frac1{576}
\max_{\pi\in\Pi_{\sta}}
W_{H_3}^{\pi}(
\II_{\widetilde s,\widetilde a},P,\II_{\widetilde s})
\right\}$. 
Extend $J_j$ and $B_j$ by zero after the last call.
The variable $J_j$ is $\cG_{j-1}$-measurable, and
\Cref{eq:sr-one-call-success} gives
\begin{align*}
\bbE[B_j\mid\cG_{j-1}]
\geq
\frac12J_j.
\end{align*}
Applying \Cref{lemma:exp-martingale} with $p=\frac12$, and then using
\Cref{eq:exp-good-events,eq:sr-number-calls}, gives on $\cE_{\tref}$
\begin{align*}
\sum_{j=1}^{|\cC|}B_j
&>
\frac{|\cC|-4\log\frac1\delta}{6}
\geq
1166\log\frac1\delta.
\end{align*}
Therefore,
\begin{align*}
\Ntotal^{\widetilde k+1}(\widetilde s,\widetilde a)
\geq
\sum_{k\in\cC}X^k\geq
\sum_{j=1}^{|\cC|}B_jX_j\geq
\frac{\sum_{j=1}^{|\cC|}B_j}{576}
\max_{\pi\in\Pi_{\sta}}
W_{H_3}^{\pi}(\II_{\widetilde s,\widetilde a},P,\II_{\widetilde s})\geq
2\log\frac1\delta
\max_{\pi\in\Pi_{\sta}}
W_{H_3}^{\pi}(\II_{\widetilde s,\widetilde a},P,\II_{\widetilde s}).
\end{align*}
By \Cref{eq:exp-reference-event}, intersecting $\cE_{\tref}$ with the
$SA$ martingale events has failure probability at most
$(4S^3A+SA)\delta$.
\end{proof}

\begin{lemma}[Lemma 3 in~\cite{zhang2022horizon}]\label{lemma:al1}
Let $k$ and $d$ be positive integers.
Recall that $\Pi$ denotes the set of all policies and $\Pi_{\sta}$ denotes the set of stationary policies.
We have that for any $(s,a) \in \SA$,
\begin{align*}
\max_{\pi \in \Pi}W^{\pi}_{kd}(\II_{s,a},P,\II_s) &\leq 6k\max_{\pi\in \Pi_{\sta}}W^{\pi}_{d}(\II_{s,a},P,\II_s).
\end{align*}
\end{lemma}

\subsection{Proof of \Cref{lemma:bdEk}}\label{sec:exps4}

Fix the deterministic loop order and tie-breaking specified in
\Cref{alg:explore}.  Let
\begin{align*}
\cW_{\aux}
:=\{k\in[K]:E^k\text{ occurs and }\trigger^k=\false\}.
\end{align*}
For each $k\in\cW_{\aux}$, let
$\alpha^k=(s_{\aux}^k,a_{\aux}^k)$ be the unique pair
selected at \Cref{line:condition}, and let $\pi_2^k,\pi_3^k$ be the
reaching and sampling policies selected with it.
These objects are fixed before the generated suffix begins.

Let $1=\kappa_0<\kappa_1<\cdots<\kappa_R\leq K$ consist of the initial
boundary and all episodes $k>1$ for which $\cK^k\neq\cK^{k-1}$, and define
$\kappa_{R+1}:=K+1$.
We have $R \le S^2 A$.
For episode $k$, let $I_k$ be the number of strict changes of
the known set before episode $k$; equivalently, $I_k=i$ exactly when
$\kappa_i\leq k<\kappa_{i+1}$.
For every $i\in\{0,\ldots,S^2A\}$, define
\begin{align*}
\cC_i(s,a)
:=\{k\in\cW_{\aux}:
\alpha^k=(s,a),\ I_k=i\}.
\end{align*}
$\{\cC_i (s, a)\}_{i, (s, a)}$ form a partition of $\cW_{\aux}$.

\begin{lemma}\label{lemma:v3}
With probability at least
$1-4S^3A^2\delta$, it holds that
\begin{align}
|\cW_{\aux}|
&\leq216000S\bigg(S^2A\left\lceil N_{\tref}\right\rceil
+SA(S^2A+1)
\left(58320S^2A N_{\tref}+1+\log \frac{1}{\delta} \right)\bigg),
\label{eq:v3-global-bound} \\
&\le O\left(S^8A^3\log\frac1\delta\right).
\end{align}
\end{lemma}

\begin{proof}[Proof of \Cref{lemma:v3}]
Recall the common events in \Cref{sec:exps-common}.  Fix
$(i,s,a)\in\{0,\ldots,S^2A\}\times\SA$ and suppose
$\cC_i(s,a)\neq\emptyset$.  Since $\cK^k$ and $P^{\tref,k}$ are fixed
within an epoch, $v^k(s,a)$ has the same value, denoted by $v$, for all
$k\in\cC_i(s,a)$.  Moreover, $v\geq1$ and $\pi_3^k(s)=a$, since a
stationary policy not choosing $a$ at $s$ has zero $\II_{s,a}$-return.
Set
\begin{align*}
b:=\max\left\{1,\frac{v}{36}\right\}.
\end{align*}
For every $k\in\cC_i(s,a)$, the test at \Cref{line:condition} gives
\begin{align*}
u^k(s)&\geq\frac1{1200S},
&
D^k(s,a)&\leq1620S^2AN_{\tref}u^k(s)v.
\end{align*}

Fix $k\in\cC_i(s,a)$ and suppose $\cR_k$ holds.  Since
$\frac{H_2}{H_3}=20S\log S-1\geq10(S+2)\log(S+2)$, \Cref{eq:exp-model-comparison} and \Cref{lemma:eff} give
\begin{align}
X_{H_2}^{\pi_2^k} (\{s\},\bar P^{\cut,k},\II_{\wt s^k})
\ge \frac{1}{3} X_{H_2}^{\pi_2^k} (\{s\},P^{\tref,k},\II_{\wt s^k})
\ge \frac{1}{30} X_\gamma^{\pi_2^k} (\{s\},P^{\tref,k},\II_{\wt s^k})
= \frac1{30}u^k(s).
\label{eq:v3-reaching-cut}
\end{align}
Let $\widetilde Z_k$ be the number of visits to $(s,a)$ when running $\pi_3^k$ for $H_3$ steps under $\bar P^{\cut,k}$.
By \Cref{eq:exp-model-comparison} and \Cref{lemma:eff_single},
\begin{align*}
\bbE[\widetilde Z_k]
= W_{H_3}^{\pi_3^k} (\II_{s,a},\bar P^{\cut,k},\II_s)
\ge \frac{1}{3} W_{H_3}^{\pi_3^k} (\II_{s,a}, P^{\tref,k},\II_s)
\ge \frac{1}{9} W_\gamma^{\pi_3^k} (\II_{s,a}, P^{\tref,k},\II_s)
= \frac19v.
\end{align*}
If $v\leq36$, then $\widetilde Z_k\geq1=b$.  Otherwise
$\bbE[\widetilde Z_k]>4$ and
$\frac{1}{4}\bbE[\widetilde Z_k]\geq \frac{v}{36}=b$, so
\Cref{lemma:exp-renewal} gives
\begin{align}
\bbP_{\bar P^{\cut,k},\pi_3^k}(\widetilde Z_k\geq b)
\geq\frac12.
\label{eq:v3-one-call-sampling}
\end{align}

Let $\cF_k$ be the history immediately before
\Cref{line:explore_reaching} of \Cref{alg:explore} is executed in episode $k$.
Let $U_k$ be the
event that an unknown triple is observed while executing
\Cref{line:explore_reaching,line:explore_sampling}, and let $B_k$ be the
event that no such triple is observed, $s$ is reached in
\Cref{line:explore_reaching}, and at least $b$ samples of $(s,a)$ are
collected in \Cref{line:explore_sampling}.  Coupling the cut trajectories
with the true trajectory until the first unknown transition, as in the proof
of \Cref{lemma:v2}, and using
\Cref{eq:v3-reaching-cut,eq:v3-one-call-sampling}, gives
\begin{align}
\bbP(U_k\cup B_k\mid\cF_k)
\geq\frac1{60}u^k(s)
\geq\frac1{72000S}
\qquad\text{on }\cR_k.
\label{eq:v3-one-call-useful}
\end{align}

Write $\cC_i(s,a)=\{k_1<\cdots<k_n\}$.  For $j\in[n]$, set
$\cG_{j-1}:=\cF_{k_j}$, and extend the filtration by the terminal history
after the last call.
In \Cref{alg:explore}, since membership in $\cC_i(s,a)$ is determined at
\Cref{line:condition}, before \Cref{line:explore_reaching}, these
sigma-fields form a filtration.
Set
\begin{align*}
J_j&:=\bbI\{\cR_{k_j}\},
&
Y_j&:=J_j\bbI[U_{k_j}\cup B_{k_j}].
\end{align*}
Then $J_j$ is $\cG_{j-1}$-measurable, $Y_j$ is $\cG_j$-measurable, and
\Cref{eq:v3-one-call-useful} gives
\begin{align*}
\bbE[Y_j\mid\cG_{j-1}]
\geq\frac{J_j}{72000S}.
\end{align*}
By \Cref{eq:exp-good-events}, $J_j=1$ on $\cE_{\tref}$.
Thus \Cref{lemma:exp-martingale} with $p=\frac{1}{72000S}$ implies that conditioning on $\cE_{\tref}$, with probability at least $1 - \delta$,
\begin{align}
|\cC_i(s,a)|
&\le \sum_{j=1}^K J_j \notag \\
&\leq216000S\bigl(Q_i(s,a)+G_i(s,a)\bigr)+144000S\log\frac1\delta\notag\\
&\leq216000S\left(
Q_i(s,a)+G_i(s,a)+\log\frac1\delta
\right),
\label{eq:v3-martingale-cell}
\end{align}
where
\begin{align*}
Q_i(s,a)&:=\sum_{k\in\cC_i(s,a)}\bbI[U_k],
&
G_i(s,a)&:=\sum_{k\in\cC_i(s,a)}\bbI[B_k].
\end{align*}

Every $B_{k_j}$ with $j<n$ contributes at least $b$ known-successor
samples to $D^{k_n}(s,a)$.  Hence, using the test at
\Cref{line:condition} and $u^{k_n}(s)\leq1$,
\begin{align*}
b\bigl(G_i(s,a)-1\bigr)_+
&\leq D^{k_n}(s,a)\\
&\leq1620S^2AN_{\tref}u^{k_n}(s)v\\
&\leq1620S^2AN_{\tref}v.
\end{align*}
Since $\frac{v}{b}\leq36$,
\begin{align}
G_i(s,a)\leq1+58320S^2AN_{\tref}.
\label{eq:v3-sample-cap}
\end{align}

Assign every episode counted by some $Q_i(s,a)$ to the first unknown
triple observed in that episode.  By \Cref{eq:exp-unknown-count}, each
triple can be assigned at most $\lceil N_{\tref}\rceil$ times.  Therefore
\begin{align}
\sum_{(s,a)\in\SA}\sum_{i=0}^{S^2A}Q_i(s,a)
\leq S^2A\left\lceil N_{\tref}\right\rceil.
\label{eq:v3-unknown-cap}
\end{align}
Summing
\Cref{eq:v3-martingale-cell} and using
\Cref{eq:v3-sample-cap,eq:v3-unknown-cap} proves
\Cref{eq:v3-global-bound}.

Finally, \Cref{eq:exp-reference-event} fails with probability at most
$2S^3A\delta$, and the martingale event is used once for each of the at
most $SA(S^2A+1)$ potential cells.  A union bound gives total failure
probability at most $4S^3A^2\delta$.
\end{proof}

\restatableBdEk*

\begin{proof}
Define $
\cW := \{k \in [K] : E^k \text{ occurs}\}$, 
and let $\cW_{\aux}$ denote the set defined in \Cref{lemma:v3}.  By the definition of $E^k$ and
\Cref{alg:new}, $\cW$ is exactly the set of episodes in which
\Cref{alg:explore} is called.  Every false return selects exactly one
auxiliary pair, so \Cref{lemma:v3} bounds
$|\cW_{\aux}|$ by the right-hand side of
\Cref{eq:v3-global-bound}.

Summing \Cref{eq:exp-effective-call-cap} over all target pairs gives
\begin{align*}
|\cW\setminus\cW_{\aux}|
\leq SA\left\lceil N_{\known}\right\rceil.
\end{align*}
Combining the two parts gives the explicit bound in the statement.
The only random event used here is the event of \Cref{lemma:v3}.
\end{proof}

\subsubsection{Adapted Lemmas from \cite{zhang2022horizon}}

\begin{lemma}[Adapted from Lemma 18 in \cite{zhang2022horizon}]\label{lemma:eff_single}
Let $d$ be a positive integer, $\gamma = 1-\frac{1}{d}$, $p$ be any transition model, $(s,a)$ be a state-action pair, and $\pi$ be a stationary policy.
Then we have that
\begin{align*}
   \frac{1}{3} W^{\pi}_{d}(\II_{s,a},p,\II_{s}) &\leq  W^{\pi}_{\gamma}(\II_{s,a},p,\II_{s}) \leq 3  W^{\pi}_{d}(\II_{s,a},p,\II_{s}).
\end{align*}
\end{lemma}

\begin{lemma}[Adapted from Lemma 18 in \cite{zhang2022horizon}]\label{lemma:eff}
Let $d_1,d_2$ be positive integers such that $d_1\geq 10S\log(S)d_2$.
Let $\gamma = 1-\frac{1}{d_2}$.
Let $p$ be any transition model, $(s,a)$ a state-action pair, $r$ a non-negative reward, and $\pi$ a stationary policy.
Then we have that
\begin{align*}
   W^{\pi}_{\gamma}(r,p,\II_{s}) \leq 10W^{\pi}_{d_1}(r,p,\II_{s}).
\end{align*}
Moreover, for any target set $\cX$ and any initial distribution $\mu$, the corresponding reaching probabilities satisfy
\begin{align*}
   X^{\pi}_{\gamma}(\cX,p,\mu) \leq 10 X^{\pi}_{d_1}(\cX,p,\mu).
\end{align*}
\end{lemma}

\begin{proof} [Proof of \Cref{lemma:eff}]
The inequality involving $W^\pi$ follows from Lemma~18 of \cite{zhang2022horizon}.
It remains to prove the inequality involving $X^\pi$.

Fix a target set $\cX$ and an initial distribution $\mu$.
Introduce an absorbing state $z_{\cX}$, extend $\pi$ arbitrarily to
$z_{\cX}$, and define the auxiliary reward and transition model by
\begin{align*}
r_{\cX}(x,a)
:=
\bbI[x\in\cX],
\qquad p^{\cX}_{x,a}
:=
\begin{cases}
\II_{z_{\cX}},
    &x\in\cX\cup\{z_{\cX}\},\\
p_{x,a},
    &x\notin\cX.
\end{cases}
\end{align*}
Here $p_{x,a}$ is naturally regarded as assigning zero probability to $z_{\cX}$.

Consequently, for every positive integer $d$,
\begin{align}
X_d^\pi(\cX,p,\mu)
&=
W_d^\pi(r_{\cX},p^{\cX},\mu),
\nonumber
\\X_\gamma^\pi(\cX,p,\mu)
&=
W_\gamma^\pi(r_{\cX},p^{\cX},\mu).\nonumber
\end{align}

The inequality for $W^\pi$ extends to an arbitrary $\mu$ by linearity in the initial distribution.
Hence
\begin{align*}
X_\gamma^\pi(\cX,p,\mu)
&=
W_\gamma^\pi(r_{\cX},p^{\cX},\mu)
\leq
10W_{d_1}^\pi(r_{\cX},p^{\cX},\mu)
=
10X_{d_1}^\pi(\cX,p,\mu),
\end{align*}
completing the proof.
\end{proof}

\begin{lemma}[Adapted from Lemma 5 in~\cite{zhang2022horizon}]\label{lemma:approx}
Let $\cT$ be a finite state space, and let $P',P''$ be transition models on $\cT$ with the same action space such that
\begin{align}
\ee^{-\epsilon}P''_{x,a,y}
\leq
P'_{x,a,y}
\leq
\ee^{\epsilon}P''_{x,a,y}
\qquad
\forall (x,a,y)\in\cT\times\cA\times\cT \label{eq:eps_close}
\end{align}
for some $\epsilon\geq0$.
Fix a stationary deterministic Markov policy $\pi$, and let
\begin{align*}
m_\pi
:=
\left|
\left\{
 x\in\cT:
 P'_{x,\pi(x),\cdot}\neq P''_{x,\pi(x),\cdot}
\right\}
\right|.
\end{align*}
Then, for every initial distribution $\mu$ on $\cT$, $d \in \bbN^+$, non-negative reward $r$, $\cX \subseteq \cT$, and $\gamma \in [0, 1)$,
\begin{align}
\ee^{-4m_\pi\epsilon}W_d^\pi(r,P',\mu)
&\leq
W_d^\pi(r,P'',\mu)
\leq
\ee^{4m_\pi\epsilon}W_d^\pi(r,P',\mu); \label{eq:approx-finite-w} \\
\ee^{-4m_\pi\epsilon}X_d^\pi(\cX,P',\mu)
&\leq
X_d^\pi(\cX,P'',\mu)
\leq
\ee^{4m_\pi\epsilon}X_d^\pi(\cX,P',\mu); \label{eq:approx-finite-x} \\
\ee^{-4m_\pi\epsilon}W_\gamma^\pi(r,P',\mu)
&\leq
W_\gamma^\pi(r,P'',\mu)
\leq
\ee^{4m_\pi\epsilon}W_\gamma^\pi(r,P',\mu); \label{eq:approx-discount-w} \\
\ee^{-4m_\pi\epsilon}X_\gamma^\pi(\cX,P',\mu)
&\leq
X_\gamma^\pi(\cX,P'',\mu)
\leq
\ee^{4m_\pi\epsilon}X_\gamma^\pi(\cX,P',\mu). \label{eq:approx-discount-x}
\end{align}
\end{lemma}

\begin{proof} [Proof of \Cref{lemma:approx}]
Let $\pi$ be fixed, and write
\begin{align*}
\mathcal D_\pi
:=
\left\{
x\in\cT:
P'_{x,\pi(x),\cdot}\neq P''_{x,\pi(x),\cdot}
\right\}
=
\{x_1,\ldots,x_{m_\pi}\}.
\end{align*}
The proof of Lemma~5 in~\cite{zhang2022horizon} shows that, if two transition models $Q$ and $\widetilde Q$ differ only at one pair $(s^*, a^* = \pi(s^*))$, and this pair satisfies \Cref{eq:eps_close}, then for every $d\in\bbN^+$,
\begin{align*}
\ee^{-4\epsilon}W_d^\pi(r,Q,\mu)
\leq
W_d^\pi(r,\widetilde Q,\mu)
\leq
\ee^{4\epsilon}W_d^\pi(r,Q,\mu).
\end{align*}

Set $P^{(0)}=P'$, and, for each $j\in[m_\pi]$, let $P^{(j)}$ be obtained
from $P^{(j-1)}$ by replacing the row selected by $\pi$ at $x_j$ with the
corresponding row of $P''$. Applying the preceding comparison successively
gives
\begin{align*}
\ee^{-4m_\pi\epsilon}W_d^\pi(r,P',\mu)
\leq
W_d^\pi(r,P^{(m_\pi)},\mu)
\leq
\ee^{4m_\pi\epsilon}W_d^\pi(r,P',\mu).
\end{align*}
The models $P^{(m_\pi)}$ and $P''$ agree on every row selected by $\pi$ and
therefore induce the same trajectory distribution under $\pi$. This proves
\Cref{eq:approx-finite-w}.

For any transition model $Q$, nonnegativity of $r$ and a change in the order
of summation give
\begin{align*}
W_\gamma^\pi(r,Q,\mu)
=
(1-\gamma)\sum_{d\geq1}\gamma^{d-1}W_d^\pi(r,Q,\mu).
\end{align*}
Multiplying the finite-horizon comparison by
$(1-\gamma)\gamma^{d-1}$ and summing over $d\geq1$ proves \Cref{eq:approx-discount-w}.

Finally, \Cref{eq:approx-finite-x,eq:approx-discount-x} follow from the reduction in the
proof of \Cref{lemma:eff}, by noting that the reduced transition models satisfy \Cref{eq:eps_close}.
\end{proof}

\begin{lemma}[Lemma 30 in~\cite{zhang2022horizon}]\label{lemma:v1}
Recall $\cut (p)$ in \Cref{eq:def_cut}.
Let $d$ be a positive integer, let $\pi$ be a stationary deterministic Markov policy, extended arbitrarily to $z,z'$, and let $p$ be a transition kernel on $\bar{\cS}$ satisfying $p_{z,a'}=\II_{z'}$ and $p_{z',a'}=\II_{z'}$ for every action $a'$.
For any $(s,a)\in\SA$ such that $\pi(s)=a$, it holds that
\begin{align*}
\left(1-W_d^\pi(\II_z,p,\II_s)\right)
W_d^\pi(\II_{s,a},\cut(p),\II_s)
\leq
W_d^\pi(\II_{s,a},p,\II_s)
\leq
W_d^\pi(\II_{s,a},\cut(p),\II_s).
\notag
\end{align*}
\end{lemma}

\begin{lemma}[Adapted from Lemma 31 in~\cite{zhang2022horizon}]\label{lemma:v2}
Recall $\cut (p)$ in \Cref{eq:def_cut}.
Let $\mu$ be a distribution on $\bar{\cS}$ supported on $\cS$, let $s\in\cS$, let $d$ be a positive integer, and let $\pi=(\pi_h)_{h=1}^d$ be a deterministic Markov policy, extended arbitrarily to $z,z'$.
Suppose $p_{z,a}=\II_{z'}$ and $p_{z',a}=\II_{z'}$ for every action $a$.
Then
\begin{align*}
X_d^\pi(\{s\},\cut(p),\mu)-W_d^\pi(\II_z,p,\mu)
\leq
X_d^\pi(\{s\},p,\mu)
\leq
X_d^\pi(\{s\},\cut(p),\mu).
\notag
\end{align*}
\end{lemma}

\begin{proof} [Proof of \Cref{lemma:v2}]
The left-hand inequality follows from the argument of Lemma~31 in~\cite{zhang2022horizon}.
We prove only the additional right-hand inequality.

Since $\mu$ is supported on $\cS$ and, under both $p$ and $\cut(p)$, a trajectory cannot return to $\cS$ after entering either $z$ or $z'$, every trajectory that reaches $s$ stays in $\cS$ up to its first visit to $s$. Moreover, for every $x,y\in\cS$ and $a\in\cA$, \Cref{eq:def_cut} gives $
p_{x,a,y}
\leq
\cut(p)_{x,a,y}.$

Therefore, for every $t\in[d]$ and every $x_1,\ldots,x_t\in\cS$ such that $x_t=s$ and $x_h\neq s$ for all $h<t$,
\begin{align*}
\mu(x_1)
\prod_{h=1}^{t-1}
p_{x_h,\pi_h(x_h),x_{h+1}}
\leq
\mu(x_1)
\prod_{h=1}^{t-1}
\cut(p)_{x_h,\pi_h(x_h),x_{h+1}}.
\end{align*}
Summing over all such first-hit paths and over $t\in[d]$ yields
\begin{align*}
X_d^\pi(\{s\},p,\mu)
\leq
X_d^\pi(\{s\},\cut(p),\mu),
\end{align*}
as claimed.
\end{proof}

\begin{lemma}[Discounted performance-difference identity, adapted from
Lemma~32 in~\cite{zhang2022horizon}]
\label{lemma:discounted-performance-difference}
Let $p,p'$ be two transition kernels on a finite state space $\cT$,
let $r$ be a bounded reward function, let $\pi$ be a stationary policy,
and let $\mu$ be an initial distribution.  Define the discounted
state-action occupancy under $p$ by
\begin{align*}
d_p^\pi(x,a;\mu)
:=
\sum_{t=1}^{\infty}
\gamma^{t-1}
\bbP_{p,\pi}\bigl((s_t,a_t)=(x,a)\mid s_1\sim\mu\bigr)
=
W_\gamma^\pi(\II_{x,a},p,\mu).
\end{align*}
Then
\begin{align}
W_\gamma^\pi(r,p,\mu)-W_\gamma^\pi(r,p',\mu)=
\gamma
\sum_{x,a}
d_p^\pi(x,a;\mu)
\sum_{y\in\cT}
\bigl(p_{x,a,y}-p'_{x,a,y}\bigr)
W_\gamma^\pi(r,p',\II_y).
\label{eq:discounted-performance-difference}
\end{align}
\end{lemma}

\begin{proof}
Lemma~32 of \cite{zhang2022horizon} states that for every positive integer $d$,
\begin{align*}
&W_d^\pi(r,p,\mu)-W_d^\pi(r,p',\mu)\\
&\qquad=
\sum_{h=1}^{d}
\sum_{x,a}
\bbP_{p,\pi}\bigl((s_h,a_h)=(x,a)\mid s_1\sim\mu\bigr)
\sum_y
\bigl(p_{x,a,y}-p'_{x,a,y}\bigr)
W_{d-h}^\pi(r,p',\II_y).
\end{align*}
Multiply this identity by $(1-\gamma)\gamma^{d-1}$ and sum over
$d\geq1$.  The left-hand side becomes $
W_\gamma^\pi(r,p,\mu)-W_\gamma^\pi(r,p',\mu).$
For every fixed $h$ and $y$,
\begin{align*}
&(1-\gamma)
\sum_{d=h+1}^{\infty}
\gamma^{d-1}
W_{d-h}^\pi(r,p',\II_y)=
\gamma^h
(1-\gamma)
\sum_{\ell=1}^{\infty}
\gamma^{\ell-1}
W_\ell^\pi(r,p',\II_y)
=
\gamma^h
W_\gamma^\pi(r,p',\II_y).
\end{align*}
Interchanging the sums therefore gives
\begin{align*}
\sum_{h=1}^{\infty}\gamma^h
\bbP_{p,\pi}\bigl((s_h,a_h)=(x,a)\mid s_1\sim\mu\bigr) = \gamma d_p^\pi(x,a;\mu),
\end{align*}
which proves \eqref{eq:discounted-performance-difference}.
\end{proof}

\subsubsection{Auxiliary Lemmas for the Proof of \Cref{lemma:sr}}

\begin{lemma}\label{lemma:sr-small-z}
Fix an episode $k$ for which
$(\widetilde s^k,\widetilde a^k)\neq\perp$.  On $\cR_k\cap\cM_k$, if
$\trigger^k=\true$, then every $\pi\in\Pi_{\sta}$ satisfying
$\pi(\widetilde s^k)=\widetilde a^k$ obeys
\begin{align*}
W_\gamma^\pi(
\II_z,\bar P^k,\II_{\widetilde s^k})
<\frac1{40}.
\end{align*}
\end{lemma}

\begin{proof} [Proof of \Cref{lemma:sr-small-z}]
Since $\trigger^k=\true$, no row with an unknown successor satisfies the
condition on \Cref{line:condition} of \Cref{alg:explore}.  Thus, for every $(x,b)$ such that
$\cU^k(x,b)\neq\emptyset$,
\begin{align}
u^k(x)<\frac1{1200S}
\quad\text{or}\quad
D^k(x,b)>
1620S^2AN_{\tref}u^k(x)v^k(x,b). \label{eq:unkown_condition}
\end{align}
Let
\begin{align*}
\cS_0^k
:=
\left\{
x\in\cS:
u^k(x)<\frac1{1200S}
\right\},
\end{align*}
and fix $\pi\in\Pi_{\sta}$ satisfying
$\pi(\widetilde s^k)=\widetilde a^k$.

Multiplying the upper inequalities in \Cref{lemma:v1,lemma:v2} by
$(1-\gamma)\gamma^{d-1}$ and summing over $d\geq1$ gives, whenever
$\pi(x)=b$,
\begin{align*}
X_\gamma^\pi(
\{x\},\bar P^k,\II_{\widetilde s^k})
&\leq
X_\gamma^\pi(
\{x\},\bar P^{\cut,k},\II_{\widetilde s^k}),\\
W_\gamma^\pi(
\II_{x,b},\bar P^k,\II_x)
&\leq
W_\gamma^\pi(
\II_{x,b},\bar P^{\cut,k},\II_x).
\end{align*}
By \Cref{eq:exp-model-comparison} and $\ee<3$,
\begin{align*}
X_\gamma^\pi(
\{x\},\bar P^k,\II_{\widetilde s^k})
&\leq
3u^k(x),\\
W_\gamma^\pi(
\II_{x,b},\bar P^k,\II_x)
&\leq
3v^k(x,b).
\end{align*}

By definition of $W_\gamma$,
\begin{align*}
W_\gamma^\pi(
\II_{x,b},\bar P^k,\II_{\widetilde s^k})
&=
X_\gamma^\pi(
\{x\},\bar P^k,\II_{\widetilde s^k})
W_\gamma^\pi(
\II_{x,b},\bar P^k,\II_x)
\le 9u^k(x)v^k(x,b).
\end{align*}

Now fix $x\notin\cS_0^k$ and put $b=\pi(x)$.  If
$\cU^k(x,b)\neq\emptyset$, the second alternative in \Cref{eq:unkown_condition} applies.
Moreover, $v^k(x,b)\geq1$, since a stationary policy choosing $b$ at
the initial state $x$ receives one unit of reward immediately.  Therefore,
\begin{align*}
&W_\gamma^\pi(
\II_{x,b},\bar P^k,\II_{\widetilde s^k})
\bar P^k_{x,b,z}
\le 9u^k(x)v^k(x,b)
\frac{4SN_{\tref}}{D^k(x,b)}
<\frac1{45SA}.
\end{align*}

Let $\tau_z:=\inf\{t\geq1:s_t=z\}$.  Since $z$ transitions immediately
to $z'$ and $z'$ is absorbing,
\begin{align*}
W_\gamma^\pi(
\II_z,\bar P^k,\II_{\widetilde s^k})
=
\bbE_{\bar P^k,\pi}\left[
\gamma^{\tau_z-1}\II\{\tau_z<\infty\}
\,\middle|\,
s_1=\widetilde s^k
\right].
\end{align*}
Partition the event $\{\tau_z<\infty\}$ according to the ordinary row
used immediately before entering $z$.

If this predecessor state is $x\in\cS_0^k$, then the process must first
reach $x$, and entering $z$ occurs at least one step later.  Hence its
contribution is at most
\begin{align*}
X_\gamma^\pi(
\{x\},\bar P^k,\II_{\widetilde s^k})
<
\frac1{400S}.
\end{align*}
For a predecessor row $(x,b)$ with $x\notin\cS_0^k$, its contribution
is at most
\begin{align*}
W_\gamma^\pi(
\II_{x,b},\bar P^k,\II_{\widetilde s^k})
\bar P^k_{x,b,z}
<
\frac1{45SA}.
\end{align*}
There are at most $S$ low-score states and at most $SA$ remaining rows.
Consequently,
\begin{align*}
W_\gamma^\pi(
\II_z,\bar P^k,\II_{\widetilde s^k})
<
\frac1{400}+\frac1{45}
<
\frac1{40},
\end{align*}
finishing the proof.
\end{proof}

\begin{lemma}\label{lemma:sr-model-transfer}
Fix an episode $k$ and
$(\widetilde s,\widetilde a)\in\cS\times\cA$, and suppose $\cR_k$ holds.
Recall \Cref{line:explore_pi} in \Cref{alg:explore} computes
\begin{align*}
\pi^k
\in
\argmax_{\pi\in\Pi_{\sta}}
W_\gamma^\pi(
\II_{\widetilde s,\widetilde a},
P^{\tref,k},
\II_{\widetilde s}).
\end{align*}
If, for every $\pi\in\Pi_{\sta}$ satisfying
$\pi(\widetilde s)=\widetilde a$,
\begin{align*}
W_\gamma^\pi(
\II_z,\bar P^k,\II_{\widetilde s})
<\frac1{40},
\end{align*}
then
\begin{align*}
W_\gamma^{\pi^k}(
\II_{\widetilde s,\widetilde a},
P,\II_{\widetilde s})
\geq
\frac1{16}
\max_{\pi\in\Pi_{\sta}}
W_\gamma^\pi(
\II_{\widetilde s,\widetilde a},
P,\II_{\widetilde s}).
\end{align*}
\end{lemma}

\begin{proof} [Proof of \Cref{lemma:sr-model-transfer}]
Extend $P$ to $\bar{\cS}$ using the common rows
$z\to z'$ and $z'\to z'$.  Fix
$\pi\in\Pi_{\sta}$ satisfying
$\pi(\widetilde s)=\widetilde a$, and write $
V_p(x)
:=
W_\gamma^\pi(
\II_{\widetilde s,\widetilde a},p,\II_x)$ and $
p\in
\{P,\bar P^k,\bar P^{\cut,k},P^{\tref,k}\}.$
Let $
\tau_{\widetilde s}
:=
\inf\{t\geq1:s_t=\widetilde s\}.$ 
No reward is collected before $\tau_{\widetilde s}$.  By the Markov
property, for every ordinary state $x$,
\begin{align*}
V_P(x)
&=
V_P(\widetilde s)
\bbE_{P,\pi}\left[
\gamma^{\tau_{\widetilde s}-1}
\II\{\tau_{\widetilde s}<\infty\}
\,\middle|\,s_1=x
\right]
\leq
V_P(\widetilde s).
\end{align*}
The same argument gives $
V_{\bar P^{\cut,k}}(x)
\leq
V_{\bar P^{\cut,k}}(\widetilde s).$ 
Moreover, the values at $z$ and $z'$ are zero because neither virtual
state can reach $\widetilde s$.

For $(x,b)\in\cS\times\cA$, define
\begin{align*}
d_{\bar P^k}^\pi(x,b)
:=
\sum_{t=1}^{\infty}
\gamma^{t-1}
\bbP_{\bar P^k,\pi}\left[
(s_t,a_t)=(x,b)
\,\middle|\,s_1=\widetilde s
\right].
\end{align*}
Since every visit to $z$ occurs one step after an ordinary row is used,
\begin{align}
q_k^\pi
:=
W_\gamma^\pi(
\II_z,\bar P^k,\II_{\widetilde s})
=
\gamma
\sum_{x,b}
d_{\bar P^k}^\pi(x,b)
\bar P^k_{x,b,z}.
\label{eq:sr-model-transfer-q}
\end{align}

Applying
\Cref{lemma:discounted-performance-difference} with occupancy under
$\bar P^k$ gives
\begin{align*}
V_P(\widetilde s)-V_{\bar P^k}(\widetilde s)=
\gamma
\sum_{x,b}
d_{\bar P^k}^\pi(x,b)
\sum_{y\in\bar{\cS}}
\left(
P_{x,b,y}-\bar P^k_{x,b,y}
\right)V_P(y).
\end{align*}
The two models agree on known successors, while $\bar P^k$ redirects
the unknown-successor mass to $z$.  Since $V_P(z)=0$,
\begin{align*}
\sum_{y\in\bar{\cS}}
\left(
P_{x,b,y}-\bar P^k_{x,b,y}
\right)V_P(y)
=
\sum_{y\in\cU^k(x,b)}
P_{x,b,y}V_P(y)
\leq
\bar P^k_{x,b,z}V_P(\widetilde s).
\end{align*}
All terms are nonnegative.  Therefore,
\begin{align}
0
\leq
V_P(\widetilde s)-V_{\bar P^k}(\widetilde s)
\leq
q_k^\pi V_P(\widetilde s).
\label{eq:sr-model-transfer-true}
\end{align}

Similarly,
\begin{align*}
&V_{\bar P^{\cut,k}}(\widetilde s)
-
V_{\bar P^k}(\widetilde s)
=
\gamma
\sum_{x,b}
d_{\bar P^k}^\pi(x,b)
\sum_{y\in\bar{\cS}}
\left(
\bar P^{\cut,k}_{x,b,y}
-
\bar P^k_{x,b,y}
\right)
V_{\bar P^{\cut,k}}(y).
\end{align*}
Fix $(x,b)$.
If $\bar P^k_{x,b,z}=1$,
the corresponding rows of $\bar P^k$ and $\bar P^{\cut,k}$ are both
concentrated on $z$.
If $\bar P^k_{x,b,z}<1$, then
\begin{align*}
\bar P^k_{x,b,y}
=
(1-\bar P^k_{x,b,z})\bar P^{\cut,k}_{x,b,y},
\qquad y\neq z.
\end{align*}
Since $V_{\bar P^{\cut,k}}(z)=0$,
\begin{align*}
\sum_{y\in\bar{\cS}}
\left(
\bar P^{\cut,k}_{x,b,y}
-
\bar P^k_{x,b,y}
\right)
V_{\bar P^{\cut,k}}(y)
=
\bar P^k_{x,b,z}
\sum_{y\neq z}
\bar P^{\cut,k}_{x,b,y}
V_{\bar P^{\cut,k}}(y)
\le \bar P^k_{x,b,z} V_{\bar P^{\cut,k}}(\widetilde s).
\end{align*}
Consequently,
\begin{align}
0
\leq
V_{\bar P^{\cut,k}}(\widetilde s)
-
V_{\bar P^k}(\widetilde s)
\leq
q_k^\pi
V_{\bar P^{\cut,k}}(\widetilde s).
\label{eq:sr-model-transfer-cut}
\end{align}

By \Cref{eq:exp-model-comparison},
\begin{align}
\frac1\ee
V_{\bar P^{\cut,k}}(\widetilde s)
\leq
V_{P^{\tref,k}}(\widetilde s)
\leq
\ee
V_{\bar P^{\cut,k}}(\widetilde s).
\label{eq:sr-model-transfer-reference}
\end{align}
Since $q_k^\pi<\frac{1}{40}$, \Cref{eq:sr-model-transfer-cut} gives
\begin{align*}
V_{\bar P^k}(\widetilde s)
>
\frac{39}{40}
V_{\bar P^{\cut,k}}(\widetilde s)
>
\frac14
V_{P^{\tref,k}}(\widetilde s).
\end{align*}
Likewise,
\Cref{eq:sr-model-transfer-true,eq:sr-model-transfer-cut} gives
\begin{align*}
V_P(\widetilde s)
<
\frac{40}{39}
V_{\bar P^k}(\widetilde s)
\leq
\frac{40}{39}
V_{\bar P^{\cut,k}}(\widetilde s)
<
4V_{P^{\tref,k}}(\widetilde s).
\end{align*}

Every optimizer $\pi$ under $P^{\tref,k}$ satisfies
$\pi (\widetilde s)=\widetilde a$.  Indeed, a stationary policy choosing
another action at $\widetilde s$ never receives the target reward, whereas
choosing $\widetilde a$ receives one unit immediately.  The same argument
allows the true-model maximum to be restricted to policies satisfying
$\pi(\widetilde s)=\widetilde a$.

Applying the preceding lower comparison to $\pi^k$, the upper comparison
to a true-model optimizer, and then using the optimality of $\pi^k$ under
$P^{\tref,k}$ gives
\begin{align*}
W_\gamma^{\pi^k}(
\II_{\widetilde s,\widetilde a},
P,\II_{\widetilde s})
&\geq
W_\gamma^{\pi^k}(
\II_{\widetilde s,\widetilde a},
\bar P^k,\II_{\widetilde s})\\
&>
\frac14
W_\gamma^{\pi^k}(
\II_{\widetilde s,\widetilde a},
P^{\tref,k},\II_{\widetilde s})\\
&=
\frac14
\max_{\pi\in\Pi_{\sta}}
W_\gamma^\pi(
\II_{\widetilde s,\widetilde a},
P^{\tref,k},\II_{\widetilde s})\\
&>
\frac1{16}
\max_{\pi\in\Pi_{\sta}}
W_\gamma^\pi(
\II_{\widetilde s,\widetilde a},
P,\II_{\widetilde s}).
\end{align*}
This completes the proof.
\end{proof}

\end{document}